\def\year{2022}\relax
\documentclass[letterpaper]{article} 
\usepackage{aaai22}  
\usepackage{times}  
\usepackage{helvet}  
\usepackage{courier}  
\usepackage[hyphens]{url}  
\usepackage{graphicx} 
\usepackage{natbib}  
\usepackage[para,online,flushleft]{threeparttable} 
\usepackage{multirow}
\usepackage{amsfonts}       
\usepackage{amssymb}
\usepackage{amsthm}
\usepackage{amsmath}
\usepackage{mathtools}
\usepackage{nicefrac}       
\usepackage{xcolor}
\usepackage{microtype}      
\usepackage{gensymb}
\usepackage{pifont}
\usepackage{algorithmic}
\usepackage{algorithm}
\usepackage[algo2e,ruled,vlined]{algorithm2e}

\usepackage{booktabs}       
\newtheorem{theorem}{Theorem}
\newtheorem{lemma}{Lemma}
\newtheorem{corollary}{Corollary}
\newtheorem{assumption}{Assumption}

\newtheorem{remark}{Remark}
\newtheorem{fact}{Fact}
\DeclarePairedDelimiter{\ceil}{\lceil}{\rceil}
\usepackage{caption} 
\DeclareCaptionStyle{ruled}{labelfont=normalfont,labelsep=colon,strut=off} 
\usepackage{newfloat}
\usepackage{listings}
\usepackage{enumitem}
\floatstyle{ruled}
\newfloat{listing}{tb}{lst}{}
\floatname{listing}{Listing}
\title{\texttt{MDPGT}: Momentum-based Decentralized Policy Gradient Tracking}
\author{Zhanhong Jiang\textsuperscript{\rm 1}, Xian Yeow Lee\textsuperscript{\rm 2}, Sin Yong Tan\textsuperscript{\rm 2}, Kai Liang Tan\textsuperscript{\rm 2}, \\Aditya Balu\textsuperscript{\rm 2}, Young M. Lee\textsuperscript{\rm 1}, Chinmay Hegde\textsuperscript{\rm 3}, Soumik Sarkar\textsuperscript{\rm 2}}
\affiliations{\textsuperscript{\rm 1}Johnson Controls Inc., 507 East Michigan St, Milwaukee, WI 53202,\\ \textsuperscript{\rm 2}Iowa State University, Ames, IA 50010,\\ \textsuperscript{\rm 3}New York University, 6 MetroTech Center, Brooklyn, NY 11201,\\
\{zhanhong.jiang, young.m.lee\}@jci.com, \{xylee, tsyong98, kailiang, baditya, soumiks\}@iastate.edu, \{chinmay.h\}@nyu.edu
}

\begin{document}

\maketitle

\begin{abstract}
We propose a novel policy gradient method for multi-agent reinforcement learning, which leverages two different variance-reduction techniques and does not require  large batches over iterations. Specifically, we propose a momentum-based decentralized policy gradient tracking (\texttt{MDPGT}) where a new momentum-based variance reduction technique is used to approximate the local policy gradient surrogate with importance sampling, and an intermediate parameter is adopted to track two consecutive policy gradient surrogates. Moreover, \texttt{MDPGT} provably achieves best available sample complexity of $\mathcal{O}(N^{-1}\epsilon^{-3})$ for converging to an $\epsilon$-stationary point of the global average of $N$ local performance functions (possibly nonconcave). This outperforms the state-of-the-art sample complexity in decentralized model-free reinforcement learning and when initialized with a single trajectory, the sample complexity matches those obtained by the existing decentralized policy gradient methods. We further validate the theoretical claim for the Gaussian policy function. When the required error tolerance $\epsilon$ is small enough, \texttt{MDPGT} leads to a linear speed up, which has been previously established in decentralized stochastic optimization, but not for reinforcement leaning. Lastly, we provide empirical results on a multi-agent reinforcement learning benchmark environment to support our theoretical findings.
\end{abstract}

\section{Introduction}
Multi-agent reinforcement learning (MARL) is an emerging topic which has been explored both in theoretical~\cite{nguyen2014decentralized,zhang2018fully,qu2019value,zhang2021finite} and empirical settings~\cite{helou2020fully,mukherjee2020model,zhou2020drle}. Several appealing applications of MARL can be seen in \cite{zhang2019decentralized,nguyen2020deep} and relevant references therein. 

While MARL can primarily be cast into two different categories, i.e., cooperative~\cite{li2020adaptive,wang2020federated,li2020f2a2} and competitive~\cite{chen2020delay}, our focus is in the cooperative setting; see~\cite{wei2021last} for details on the competitive setting. Cooperative MARL is typically modeled as a networked multi-agent Markov decision process (MDP)~\cite{zhang2018networked,chu2020multi,zhang2018fully} in which the agents share a centralized reward function~\cite{simoes2020multi,ackermann2019reducing}. However, in practice, this is not necessarily the case, and instead a more general yet challenging scenario is that agents have \textit{heterogeneous} reward functions. Inherently, the ultimate goal in such a cooperative MARL setting is for agents to maximize the \textit{global average} of local long-term returns. To address this problem, various algorithms have been proposed, including distributed-learning~\cite{arslan2016decentralized,nguyen2017selectively} and distributed actor-critic~\cite{li2020f2a2,ryu2020multi}. More recent works have successfully showed finite-sample analysis for decentralized batch MARL~\cite{zhang2021finite} and leveraged advances in analysis of descent-ascent algorithms~\cite{lu2021decentralized}. 

These preliminary attempts have facilitated the theoretical understanding of  cooperative MARL by showing explicit sample complexity bounds, which match that of standard (vanilla) stochastic gradient descent (\texttt{SGD}). Additionally, recent works~\cite{huang2020momentum,xu2019sample} in centralized RL have revealed that with simple \emph{variance reduction} techniques, this sample complexity can be reduced to $\mathcal{O}(\epsilon^{-3})$ to reach an $\epsilon$-stationary point (i.e., $\mathbb{E}[\|\nabla J(\mathbf{x})\|]\leq \epsilon$, where $J$ is the return function and $\mathbf{x}\in\mathbb{R}^d$ is the decision variable to be optimized), which has been admitted as the best complexity in decentralized optimization~\cite{das2020faster,karimireddy2020mime}. However, no similar matching bounds have yet been reported in the decentralized (cooperative MARL) setting. Hence, this motivates the question:

\begin{quote}
\emph{Can we achieve a sample complexity of $\mathcal{O}(\epsilon^{-3})$ in decentralized MARL via variance reduction?}
\end{quote}

In this paper, we answer this question affirmatively by proposing a variance-reduced policy gradient tracking approach, \texttt{MDPGT} (Algorithm~\ref{mdpgt}), and analyzing it in Theorem~\ref{theorem_1}. Additionally, we propose a variation (based on a different initialization) that enables state-of-the-art (SOTA) sample complexity for decentralized MARL~\cite{zhang2021finite,lu2021decentralized}. See Table~\ref{table:findings} for SOTA comparisons. Specifically:
\begin{enumerate}[leftmargin=*,nosep]
    \item We propose \texttt{MDPGT}, in which we use a stochastic policy gradient surrogate, a convex combination of {the vanilla stochastic policy gradient and an importance sampling-based stochastic recursive algorithm} (\texttt{SARAH})~\cite{nguyen2017sarah} for the local gradient update. Instead of directly applying the stochastic policy gradient surrogate in the parameter update, an intermediate parameter is adopted to track the difference between two consecutive stochastic policy gradient surrogates. 
    For smooth nonconcave performance functions, we show that \texttt{MDPGT} with the mini-batch initialization can converge to an $\epsilon$-stationary point in $\mathcal{O}(N^{-1}\epsilon^{-3})$ gradient-based updates which matches the best available known upper bounds~\cite{huang2020momentum}. 
    \item We modify the initialization of the proposed algorithm \texttt{MDPGT} by using a single trajectory instead of a mini-batch of trajectories. Surprisingly, we find that \textit{only one trajectory} results in a larger sampling complexity $\mathcal{O}(N^{-1}\epsilon^{-4})$, which, however, is the same as obtained by the SOTA~\cite{zhang2021finite,lu2021decentralized} with a linear speed up when $\epsilon$ is sufficiently small.
    Additionally, our algorithm shows that when updating the policy parameter in \texttt{MDPGT}, the mini-batch size is $\mathcal{O}(1)$ instead of being $\epsilon$-related~\cite{xu2019sample,qu2019value}, which can significantly improve practical efficiency.
    \item To facilitate the theoretical understanding of \texttt{MDPGT}, we leverage a benchmark gridworld environment for numerical simulation and compare our proposed algorithm to a baseline decentralized policy gradient (\texttt{DPG}) and the momentum-based decentralized policy gradient (\texttt{MDPG}, described in Algorithm~\ref{mdpg} in the supplementary materials), which is a new variant created in this work for the purpose of empirical comparison. We show that our theoretical claims are valid based on the experiments.
\end{enumerate}

\begin{table*}[h]
\caption{\textit{Comparisons between existing and proposed approaches.}}
\begin{center}
\begin{threeparttable}
\begin{tabular}{c c c c c c}
    \toprule
    \textbf{Method} & \textbf{ Complexity} & \textbf{Decentralized} & \textbf{Variance Reduction} & \textbf{Linear Speed Up} & \textbf{I.S.}\\ \midrule
      \texttt{MBPG}~\cite{huang2020momentum}&$\mathcal{O}(\epsilon^{-3})$&\color{red}\ding{55}&\color{green}\ding{51}&\color{red}\ding{55}&\color{green}\ding{51}\\
      \texttt{Value Prop}~\cite{qu2019value}&$\mathcal{O}(\epsilon^{-4})$&\color{green}\ding{51}&\color{red}\ding{55}&\color{red}\ding{55}&\color{red}\ding{55}\\
      \texttt{DCPG}~\cite{zeng2020decentralized}&$\mathcal{O}(\epsilon^{-4})$&\color{green}\ding{51}&\color{red}\ding{55}&\color{red}\ding{55}&\color{red}\ding{55}\\
      \texttt{Safe-Dec-PG}~\cite{lu2021decentralized}&$\mathcal{O}(\epsilon^{-4})$&\color{green}\ding{51}&\color{green}\ding{51}&\color{red}\ding{55}&\color{red}\ding{55}\\
      \texttt{DFQI}~\cite{zhang2021finite}&$\mathcal{O}(\epsilon^{-4})$&\color{green}\ding{51}&\color{green}\ding{51}&\color{red}\ding{55}&\color{red}\ding{55}\\
      \texttt{Dec-TD(0)+GT}~\cite{lin2021decentralized}&N/A&\color{green}\ding{51}&\color{green}\ding{51}&\color{red}\ding{55}&\color{red}\ding{55}\\
        \textbf{\texttt{MDPGT}}~(ours)&$\mathcal{O}(\epsilon^{-4})$&\color{green}\ding{51}&\color{green}\ding{51}&\color{green}\ding{51}&\color{green}\ding{51}\\
      \textbf{\texttt{MDPGT-MI}}~(ours)&$\mathcal{O}(\epsilon^{-3})$&\color{green}\ding{51}&\color{green}\ding{51}&\color{green}\ding{51}&\color{green}\ding{51}\\
      \bottomrule
\end{tabular}
\begin{tablenotes}
\item[1] \textbf{Complexity}: Sampling complexity for achieving $\mathbb{E}[\|\nabla J(\mathbf{x})\|]\leq \epsilon$. \item [2] \textbf{Linear Speed Up}: If an algorithm has $\mathcal{O}(1/\sqrt{K})$ convergence, then its sampling complexity of attaining an $\mathcal{O}(\epsilon)$ accurate solution is $\epsilon^{-2}$. Similarly, $\mathcal{O}(1/\sqrt{NK})$ corresponds to $N^{-1}\epsilon^{-2}$, which is $N$ times faster than the former. Typically, $K$ has to satisfy a certain condition. \item[3] \texttt{MDPGT-MI} is \texttt{MDPGT} with mini-batch initialization. We use this notation for conveniently classifying two different initialization approaches. In the rest of paper, we still adopt \texttt{MDPGT} to unify these two approaches. \item[4] \textbf{I.S.} denotes the utilization of importance sampling.
\end{tablenotes}
\end{threeparttable}
\end{center}
\label{table:findings}
\end{table*}

\noindent\textbf{Related Works.}
Most previous decentralized MARL papers~\cite{zhang2018fully,suttle2020multi,li2020f2a2,chen2020delay,bono2018cooperative} tend to focus on convergence to the optimal return. Exceptions include ~\cite{qu2019value}, where they proved non-asymptotic convergence rates with nonlinear function approximation using value propagation. This enables us to approximately derive the number of stochastic gradient evaluations. However, the algorithm involves the complex inner-outer structure and requires the size of the mini-batch to be $\sqrt{K}$, with $K$ being the number of iterations, which may not be practically implementable. ~\citet{zhang2021finite} obtain $\mathcal{O}(\epsilon^{-4})$ for the cooperative setting by using gradient tracking (\texttt{GT}), which is a bias correction technique dedicated to decentralized optimization, but with several specifically imposed assumptions, such as stationary sample paths, which may not be realistic. \citet{lu2021decentralized} also utilize \texttt{GT} but require dual parameter updates to achieve $\mathcal{O}(\epsilon^{-4})$; our approach is different and simpler. 
In this context, we mention that centralized counterparts of MARL~\cite{huang2020momentum,xu2019sample,papini2018stochastic} have achieved sample complexity of $\mathcal{O}(\epsilon^{-3})$. However, in both~\citet{xu2019sample} and~\citet{papini2018stochastic}, the size of mini-batch is $\epsilon$-related, which is more computationally sophisticated than those in both~\cite{huang2020momentum} and our proposed method. We provide additional discussion of related work in the supplementary materials.

\section{Preliminaries}\label{prelim}
We first formulate MARL, followed by an overview of variance reduction techniques and decentralized policy gradients.

\subsection{MARL Formulation}
In this context, we consider a networked system involving multiple agents (say $N$) that \textit{collaboratively} solve \textit{dynamic} optimization problems. Specifically, the system can be quantified as a graph, i.e., $\mathcal{G}=(\mathcal{V}, \mathcal{E})$, where $\mathcal{V}=\{1,2,...,N\}$ is the vertex set and $\mathcal{E}\subseteq\mathcal{V}\times\mathcal{V}$ is the edge set. Throughout the paper, we assume that $\mathcal{G}$ is \textit{static and undirected}, though in a few previous works~\cite{lin2019communication,suttle2020multi,zhang2018fully} $\mathcal{G}$ could be directed. The goal of this work is to provide the rigorous theoretical analysis for our decentralized MARL algorithm, with the property of the graph not being the main focus. When a pair of agents $i$ and $j$ can communicate with each other, we have $(i,j)\in\mathcal{E}$. We also define the neighborhood of a specific agent $i$, $Nb(i)$, such that $Nb(i) \triangleq \{j|j\in\mathcal{V}, (i,j)\in\mathcal{E} \textnormal{or}\;j=i\}$. Only agents in $Nb(i)$ are able to communicate with the agent $i$. We next present the definition of networked MARL on top of $\mathcal{G}$.

With multiple agents, the networked Markov decision process is thus characterized by a tuple $(\mathcal{S}, \{\mathcal{A}^i\}_{i\in\mathcal{V}}, \mathcal{P}, \{r^i\}_{i\in\mathcal{V}}, \mathcal{G}, \gamma)$, where $\mathcal{S}$ indicates a global state space shared by all agents in $\mathcal{V}$ with $|\mathcal{S}|<\infty$, $\mathcal{A}^i$ signifies the action space specified for agent $i$, and $\gamma\in(0,1]$ is the discount factor. Moreover, in the cooperative MARL setting, the environment is driven by the \textit{joint} action space instead of individual action spaces. Thus, $\mathcal{A}\triangleq\prod_{i\in\mathcal{V}}\mathcal{A}^i$ is defined as the joint action space over all agents in $\mathcal{V}$. $\mathcal{P}:\mathcal{S}\times\mathcal{A}\times\mathcal{S}\to [0,1]$ represents the probability function to transition the current state to the next state. $\{r^i\}_{i\in\mathcal{V}}:\mathcal{S}\times\mathcal{A}\to\mathbb{R}$ is the local reward function of agent $i$ and $r^i\in[-R,R](R>0)$. Additionally, states and actions are assumed to be \textit{globally} observable, while the rewards are only locally observable. Such an assumption corresponds to the definition of the cooperative MARL setting and has been generic in previous works~\cite{zhang2018fully,zhang2018networked,zhang2021finite}. 

We next describe how agents behave in such an environment. Suppose that the current state of the environment is $s_k\in\mathcal{S}$, where $k$ is the time step. Each agent $i$ chooses its own action $a^i_k\in\mathcal{A}^i$, based on the local policy, $\pi^i:\mathcal{S}\times\mathcal{A}^i\to[0,1]$. For a parameterized policy, we denote by $\pi^i_{\mathbf{x}^i}(s, a^i)$, which indicates the probability of agent $i$ choosing action $a^i$ given the current state $s$ and $\mathbf{x}^i\in\mathbb{R}^{d_{i}}$ here is the policy parameter. Stacking all local policy parameter together yields $\mathbf{x}=[(\mathbf{x}^1)^\top,(\mathbf{x}^2)^\top,...,(\mathbf{x}^N)^\top]^\top\in\mathbb{R}^{\sum_{i\in\mathcal{V}}d_i}$. Hence, the joint policy can be denoted as $\pi_{\mathbf{x}}:\mathcal{S}\times\mathcal{A}\to[0,1]$, where $\pi_{\mathbf{x}}(s,a)\triangleq\prod_{i\in\mathcal{V}}\pi^i_{\mathbf{x}^i}(s, a^i)$ and $a\in\mathcal{A}$. In this context, the decisions are decentralized due to locally observable rewards, locally evaluated policies and locally executed actions. To simplify the notations, we drop the $\mathbf{x}^i$ for $\pi^i_{\mathbf{x}^i}$ and $\mathbf{x}$ for $\pi_{\mathbf{x}}$ respectively for local and joint policies throughout the rest of the paper. With the joint policy $\pi$ and the state transition function $\mathcal{P}$, the environment evolves from $s$ to $s'$ with the probability $\mathcal{P}(s'|s,a)$. Another assumption imposed in this paper for the policy function is that for all $i\in\mathcal{V}, s\in\mathcal{S}, a^i\in\mathcal{A}^i$, $\pi^i(s, a^i)$ is {continuously differentiable} w.r.t. all $\mathbf{x}^i\in\mathbb{R}^{d_i}$. Such an assumption will assist in characterizing the smoothness of the objective function.

The goal for each agent is to learn a local policy $\pi^i_*$ such that the joint policy $\pi_*$ is able to maximize the \textit{global average} of expected cumulative discounted rewards, i.e.,
\begin{equation}\label{objective}
    \pi_{*}=\textnormal{argmax}_{\mathbf{x}\in\mathbb{R}^d}J(\mathbf{x})\triangleq\frac{1}{N}\sum_{i\in\mathcal{V}}\mathbb{E}\bigg[\sum^H_{h=0}\gamma^hr^i_h\bigg],
\end{equation}
where $H$ is the horizon and $d=\sum_{i\in\mathcal{V}}d_i$. Several works~\cite{zhang2018fully,qu2019value,zhang2021finite,lin2019communication,suttle2020multi} have made their attempts to resolve this optimization problem, leading to different algorithms. 
Since each agent only has access to local information, a communication protocol needs to be introduced in the system, as done in decentralized optimization. With that, well-known centralized policy-based algorithms can be extended as MARL algorithms. Nevertheless, one issue that has not been sufficiently explored is \textit{the inherent policy gradient variance, which could even be more significant in the MARL algorithms.} Consequently, this work propose novel MARL algorithms to investigate how to reduce the policy gradient variance during the optimization.
\subsection{Variance Reduction and Bias Correction}
In stochastic optimization, variance reduction techniques have been well studied and applied to either centralized or decentralized gradient descent type of algorithms, such as \texttt{SVRG}~\cite{johnson2013accelerating}, \texttt{SARAH}~\cite{nguyen2017sarah}, \texttt{SPIDER}~\cite{fang2018spider}, \texttt{Hybrid-SARAH}~\cite{tran2019hybrid} and \texttt{STORM}~\cite{cutkosky2019momentum}. In another line of work, the \texttt{GT} technique~\cite{pu2020distributed,sun2019convergence} was proposed specifically for consensus-based decentralized optimization techniques to improve the convergence rate by tracking and correcting each agent’s locally aggregated gradients. In our work, we leverage both \texttt{Hybrid-SARAH} and \texttt{GT} to reduce the policy gradient variance and correct the policy gradient bias respectively in the MARL and achieve the best convergence rate. 
\texttt{Hybrid-SARAH} performs with a trade-off parameter to balance the effect between vanilla stochastic gradient and \texttt{SARAH}. More detail on these techniques are elaborated in the supplementary materials.


So far, we are not aware of existing results that have successfully shown \texttt{SARAH} or \texttt{Hybrid-SARAH} type of variance reduction techniques well suited for \textit{decentralized non-oblivious} learning problems, e.g., MARL. Consequently, the regular \texttt{Hybrid-SARAH} technique cannot be directly applied to MARL; we address this challenge in sections below.

\subsection{Decentralized Policy Gradient}
Given a time horizon $H$, we define a trajectory specifically for agent $i$, as $\tau^i\triangleq\{s_0, a^i_0, ..., s_{H-1}, a^i_{H-1}\}$ under any stationary policy. By following the trajectory $\tau^i$, a cumulative discounted reward is given as $ \mathcal{R}_i(\tau^i)\triangleq\sum_{h=0}^H\gamma^hr^i_h$ such that an individual return can be obtained as:
\begin{equation}
    J_i(\mathbf{x}^i)\triangleq\mathbb{E}_{\tau^i\sim p_i(\tau^i|\mathbf{x}^i)}[\mathcal{R}_i(\tau^i)]=\int \mathcal{R}_i(\tau^i)p_i(\tau^i|\mathbf{x}^i)d\tau^i,
\end{equation}
where $p_i(\tau^i|\mathbf{x}^i)$ is the probability distribution over $\tau^i$ that is equivalent to the following expression given the initial distribution $\rho_0^i=\rho^i(s_0)$. Without loss of generalization, we can assume that the initial distribution is identical for all agents, namely $\rho(s_0)$. Then, we have,
\begin{equation}\label{distribution}
    p_i(\tau^i|\mathbf{x}^i) = \rho_0(s_0)\prod_{h=0}^{H-1}\mathcal{P}(s_{h+1}|s_h, a^i_h)\pi^i(a^i_h|s_h).
\end{equation}
For each agent $i$, the goal is to find an optimal policy $\pi^i_*$ to maximize the return $J_i(\mathbf{x}^i)$. As discussed above, the underlying dynamic distribution results in a non-oblivious learning problem, which is more significant in MARL. To resolve this issue, the decentralized policy gradient is a decent choice. As background knowledge of MARL, we next present how to arrive at the local stochastic policy gradient, which will help characterize the analysis for the proposed algorithms.

Computing the gradient of $J_i(\mathbf{x}^i)$ w.r.t $\mathbf{x}^i$ yields the following formula:
\begin{equation}
\begin{split}
    \nabla J_i(\mathbf{x}^i)&=\int \mathcal{R}_i(\tau^i)\frac{\nabla p_i(\tau^i|\mathbf{x}^i)}{p_i(\tau^i|\mathbf{x}^i)}p_i(\tau^i|\mathbf{x}^i)d\tau^i\\&=\mathbb{E}_{\tau^i\sim p_i(\tau^i|\mathbf{x}^i)}[\nabla\textnormal{log}p_i(\tau^i|\mathbf{x}^i)\mathcal{R}_i(\tau^i)]
\end{split}
\end{equation}
In practice, $p_i(\tau^i|\mathbf{x}^i)$ is typically unknown such that the accurate full policy gradient for agent $i$ is difficult to obtain. Thus, similar to decentralized stochastic gradient descent~\cite{jiang2017collaborative}, we calculate the policy gradient by sampling a mini-batch of trajectories $\mathcal{B} = \{\tau^i_m\}^{|\mathcal{B}|}_{m=1}$ from the distribution $p_i(\tau^i|\mathbf{x}^i)$ such that
\begin{equation}\label{stochastic_pg}
    \hat{\nabla} J_i(\mathbf{x}^i)=\frac{1}{|\mathcal{B}|}\sum_{m\in\mathcal{B}}\nabla\textnormal{log}p_i(\tau^i_m|\mathbf{x}^i)\mathcal{R}_i(\tau^i_m).
\end{equation}
In addition, combining Eq.~\ref{distribution}, we can observe that $\nabla\textnormal{log}p_i(\tau^i_m|\mathbf{x}^i)$ is independent of the probability transition $\mathcal{P}$. Hence, Eq.~\ref{stochastic_pg} is written as
\begin{equation}\label{gradient_estimator}
\begin{split}
    \hat{\nabla}J_i(\mathbf{x}^i)&=\frac{1}{|\mathcal{B}|}\sum_{m\in\mathcal{B}}\mathbf{g}_i(\tau^i_m|\mathbf{x}^i) \\&= \frac{1}{|\mathcal{B}|}\sum_{m\in\mathcal{B}}\bigg(\sum_{h=0}^{H-1}\nabla_{\mathbf{x}^i}\textnormal{log}\pi^i(a^{i,m}_h,s^m_h)\bigg)\cdot\\&\bigg(\sum_{h=0}^{H-1}\gamma^hr^i_h(a^{i,m}_h,s^m_h)\bigg)
\end{split}
\end{equation}
In the above equation, $\mathbf{g}_i(\tau^i|\mathbf{x}^i)$ is the unbiased estimate of $\nabla J_i(\mathbf{x}^i)$, i.e., $\mathbb{E}[\mathbf{g}^i(\tau^i|\mathbf{x}^i)] = \nabla J_i(\mathbf{x}^i)$. Some well-known policy gradient estimators can be obtained through Eq.~\ref{gradient_estimator}, such as decentralized REINFORCE, which is the direct extension of its centralized version. We refer interested readers to~\cite{huang2020momentum} for more details.


\section{Our Proposed Approach: MDPGT}\label{proposed_algo}

\subsection{Hybrid Importance Sampling \texttt{SARAH}}
In this subsection, we propose a hybrid importance sampling version of \texttt{SARAH}, termed \texttt{HIS-SARAH}, for decentralized policy gradient updates. First, we define the importance sampling weight~\cite{metelli2020importance} as follows:
\begin{equation}\label{is}
    \upsilon(\tau|\mathbf{x}',\mathbf{x})=\frac{p(\tau|\mathbf{x}')}{p(\tau|\mathbf{x})}=\prod_{h=0}^{H-1}\frac{\pi_{\mathbf{x}'}(a_h|s_h)}{\pi_{\mathbf{x}}(a_h|s_h)}.
\end{equation}
As mentioned in the last section, due to the non-oblivious learning problem, $\mathbb{E}_{\tau\sim p(\tau|\mathbf{x})}[\mathbf{g}(\tau|\mathbf{x})-\mathbf{g}(\tau|\mathbf{x}')]\neq\nabla J(\mathbf{x}) - \nabla J(\mathbf{x}')$. With Eq.~\ref{is} we have $\mathbb{E}_{\tau\sim p(\tau|\mathbf{x})}[\mathbf{g}(\tau|\mathbf{x})- \upsilon(\tau|\mathbf{x}',\mathbf{x})\mathbf{g}(\tau|\mathbf{x}')]=\nabla J(\mathbf{x}) - \nabla J(\mathbf{x}')$, which has been analyzed in~\cite{huang2020momentum} for centralized policy optimization methods and will be a key relationship in our proof. We denote by $\mathbf{u}^i$ the stochastic policy gradient surrogate for agent $i$. 
Thus, applying Eq.~\ref{is} in a decentralized manner for \texttt{Hybrid-SARAH} (See Supplementary materials for definition) gives the following update law at a time step $k$:
\begin{equation}\label{surrogate}
\begin{split}
    \mathbf{u}^i_k &= \beta\mathbf{g}_i(\tau^i_k|\mathbf{x}^i_k)+(1-\beta)[\mathbf{u}_{k-1}^i+\mathbf{g}_i(\tau^i_k|\mathbf{x}^i_k)\\&-\upsilon_i(\tau^i_k|\mathbf{x}^i_{k-1},\mathbf{x}^i_k)\mathbf{g}_i(\tau^i_k|\mathbf{x}^i_{k-1})].
\end{split}
\end{equation}
The second term on the right hand side of Eq.~\ref{surrogate} differ in the extra importance sampling weight compared to Eq.~\ref{storm} in the supplementary materials. Intuitively, $\upsilon_i(\tau^i_k|\mathbf{x}^i_{k-1},\mathbf{x}^i_k)$ resolves the non-stationarity in the MARL and retains the regular variance reduction property of \texttt{HIS-SARAH} as applied in supervised learning problems. Clearly, each $\mathbf{u}_k^i$ is a \textit{conditionally} biased estimator $\nabla J_i(\mathbf{x}^i_k)$, i.e., $\mathbb{E}[\mathbf{u}_k^i]\neq\nabla J_i(\mathbf{x}^i_k)$ typically. Nevertheless, it can be shown that $\mathbb{E}[\mathbf{u}_k^i] = \mathbb{E}[\nabla J_i(\mathbf{x}^i_k)]$, which implies that $\mathbf{u}_k^i$ acts as a \textit{surrogate} for the underlying exact full policy gradient. Therefore, $\mathbf{u}_k^i$ will be called directly the stochastic policy gradient surrogate for the rest of analysis. With Eq.~\ref{surrogate} in hand, we now are ready to present the algorithmic framework in the following.

\subsection{Algorithmic Framework}
We first present \texttt{MDPGT} (in Algorithm~\ref{mdpgt}), which only takes a trajectory to initialize the policy gradient surrogate, leading to significant randomness due to the conditionally biased estimator property at the starting point of optimization, but still retaining the same sampling complexity as compared to the SOTA of MARL. To have a better initialization, we also present another way of initialization by sampling a mini-batch of trajectories from the distribution (in blue in Algorithm~\ref{mdpgt}). Surprisingly, we will see that with a proper size of mini-batch initialization, the sampling complexity of our proposed approach complies with the best result in centralized RL, which improves the SOTA of MARL.

\begin{algorithm}
\SetAlgoLined
\KwResult{$\tilde{\mathbf{x}}_K$ chosen uniformly random from $\{\mathbf{x}^i_k,i\in\mathcal{V}\}^K_{k=1}$}
 \textbf{Input:} $\mathbf{x}^i_0=\bar{\mathbf{x}}_0\in\mathbb{R}^d,\eta\in\mathbb{R}^+,\beta\in(0,1),\mathbf{W}\in\mathbb{R}^{N\times N},\mathbf{v}^i_0=\mathbf{0}_d,\mathbf{u}_{-1}^i=\mathbf{0}_d,K, \mathcal{B}\in\mathbb{Z}^+,k=1$\;
 
 Initialize the local policy gradient surrogate by sampling a trajectory $\tau^i_0$ from $p_i(\tau^i|\mathbf{x}^i_0):\mathbf{u}^i_0=\mathbf{g}_i(\tau^i_0|\mathbf{x}^i_0)$, or \textcolor{blue}{by sampling a \textit{mini-batch} of trajectories $\{\tau_0^{i,m}\}_{m=1}^{|\mathcal{B}|}$ from $p_i(\tau^i|\mathbf{x}^i_0):\mathbf{u}^i_0=\frac{1}{|\mathcal{B}|}\sum_{m=1}^{|\mathcal{B}|}\mathbf{g}_i(\tau^{i,m}_0|\mathbf{x}^i_0)$}\;
 
Initialize the local policy gradient tracker: $\mathbf{v}^i_1=\sum_{j\in Nb(i)}\omega_{ij}\mathbf{v}^j_0+\mathbf{u}^i_0-\mathbf{u}^i_{-1}$\;

Initialize the local estimate of the policy network parameter: $\mathbf{x}^i_1=\sum_{j\in Nb(i)}\omega^{ij}(\mathbf{x}^j_0+\eta\mathbf{v}^j_1)$\;

 \While{$k<K$}{
  \For{each agent}{
    Sample a trajectory $\tau^i_k$ from $p_i(\tau^i|\mathbf{x}^i_k)$ and compute the local policy gradient surrogate using Eq.~\ref{surrogate}\;
    
    Update the local policy gradient tracker $\mathbf{v}^i_{k+1}=\sum_{j\in Nb(i)}\omega_{ij}\mathbf{v}^j_k+\mathbf{u}^i_k-\mathbf{u}^i_{k-1}$\;
    
    Update the local estimate of the policy network parameters $\mathbf{x}^i_{k+1}=\sum_{j\in Nb(i)}\omega_{ij}(\mathbf{x}^j_k+\eta\mathbf{v}^j_{k+1})$\;
    }
$k=k+1$\;
 }
 \caption{\texttt{MDPGT}}
 \label{mdpgt}
\end{algorithm}

A brief outline of Algorithm~\ref{mdpgt} is as follows. The initialization of the policy gradient surrogate $\mathbf{u}^i_0$ can either be based on only a trajectory sampled from $p_i(\tau^i|\mathbf{x}^i_0)$ or a mini-batch. Subsequently, the policy gradient tracker and network parameters are initialized based on $\mathbf{u}^i_0$. The core part of the algorithm consists of each individual update for $\mathbf{u}^i_k, \mathbf{v}^i_k$, and $\mathbf{x}^i_k$. By controlling the value of $\beta$ in~Eq. \ref{surrogate}, \texttt{MDPGT} can degenerate to either vanilla decentralized policy gradient (with $\beta=1$) or decentralized version of \texttt{SRVR-PG}~\cite{xu2019sample} (with $\beta=0$), both with the gradient tracking step. In our work, to emphasize the impact of the trade-off on the policy gradient surrogate, we keep $\beta\in(0,1)$, which makes $\beta$ act more closely as the momentum coefficient in accelerated \texttt{SGD} algorithms~\cite{singh2020squarm}. 

We emphasize that we are unaware of theoretical results for decentralized \texttt{SRVR-PG}. Hence, the proof techniques presented in this paper can also apply to this case. Another implication from Algorithm~\ref{mdpgt} is that at the beginning of each time step $k$, only one trajectory is required for computing the policy gradient, allowing for the batch size to be \textit{independent} of $\epsilon$, i.e., $\mathcal{O}(1)$, where we omit the number of agents $N$ when considering the whole networked system. 
\section{Theoretical Convergence}\label{theory}
In this section, we present an analysis of MDPGT. Most of the assumptions below are mild, and standard in the decentralized optimization and RL literature. Due to space limitations, we defer auxiliary lemmas and proofs to the supplementary materials. 

\begin{assumption}\label{assum_1}
Gradient and Hessian matrix of function $\textnormal{log}\pi^i(a^i|s)$ are bounded, i.e., there exist constants $C_g, C_h>0$ such that $\|\nabla\textnormal{log}\pi^i(a^i|s)\|\leq C_g$ and $\|\nabla^2 \textnormal{log}\pi^i(a^i|s)\|\leq C_h$, for all $i\in\mathcal{V}$.
\end{assumption}
Note that we skip the subscript $\mathbf{x}^i$ at $\pi^i$ for the notation simplicity. In this context, we did not impose the bounded policy gradient assumption, though it can be derived based on the above assumption, which has been adopted in centralized RL algorithms~\cite{zhang2021convergence,huang2020momentum,xu2019sample}. Additionally, it also helps derive the smoothness of $J_i(\mathbf{x}^i)$ that has typically been exposed as an assumption in decentralized learning/optimization literature. 
\begin{assumption}\label{assum_3}
The mixing matrix $\mathbf{W}\in\mathbb{R}^{N\times N}$ is doubly stochastic such that $\lambda\triangleq\|\mathbf{W}-\mathbf{P}\|\in[0,1)$, where $\lambda$ signifies the second largest eigenvalue to measure the algebraic connectivity of the graph, and $\mathbf{P}=\frac{1}{N}\mathbf{1}^\top\mathbf{1}$ and $\mathbf{1}$ is a column vector with each entry being 1.
\end{assumption}
\begin{assumption}\label{assum_4}
Variance of importance sampling weight $\upsilon_i(\tau^i|\mathbf{x}_1, \mathbf{x}_2)$ is bounded, i.e., there exists a constant $\mathcal{M}>0$ such that $\mathbb{V}(\upsilon_i(\tau^i|\mathbf{x}_1, \mathbf{x}_2))\leq \mathcal{M}$, for any $\mathbf{x}_1, \mathbf{x}_2\in\mathbb{R}^{d_i}$ and $\tau^i\sim p_i(\tau^i|\mathbf{x}_1)$, for all $i\in\mathcal{V}$.
\end{assumption}
Assumption~\ref{assum_3} is generic in most previous works on decentralized optimization, though such a property has been relaxed in some works~\cite{nedic2014distributed}. However, we have not been aware of any existing works in MARL doing such a relaxation and its investigation can be of independent interest. Assumption~\ref{assum_4} is specifically introduced for importance sampling-based methods. Such an assumption is critical to construct the relationship between $\mathbb{V}(\upsilon_i(\tau^i|\mathbf{x}_1, \mathbf{x}_2))$ and $\|\mathbf{x}_1-\mathbf{x}_2\|^2$, through which the impact of the variance of importance sampling on the convergence can be explicitly quantified. Another typical assumption is for the bounded variance of stochastic gradient such that $\mathbb{E}[\|\mathbf{g}_i(\tau^i|\mathbf{x}^i)-\nabla J_i(\mathbf{x}^i)\|^2]\leq \sigma^2_i$. However, under MARL setting, such a result can be derived from Assumption~\ref{assum_1} and we present the formal result in Lemma~\ref{lemma_1}. In this context, we also have $\bar{\sigma}^2=\frac{1}{N}\sum^N_{i=1}\sigma^2_i$, for all $i\in\mathcal{V}$. The explicit expression of $\bar{\sigma}^2$ is given in the supplementary materials.
\subsection{Main Results}
We present the main results to show specifically the convergence rates for \texttt{MDPGT} when it is initialized by a mini-batch of trajectories. We denote by $L>0$ the \textit{smoothness constant} and $G>0$ the \textit{upper bound} of $\|\mathbf{g}_i(\tau^i|\mathbf{x}^i)\|$ for all $i\in\mathcal{V}$. We further define a constant $C_\upsilon>0$ such that $\mathbb{V}(\upsilon_i(\tau^i|\mathbf{x}_1, \mathbf{x}_2))\leq C^2_\upsilon\|\mathbf{x}_1-\mathbf{x}_2\|^2$. The explicit expressions of these constants are derived in lemmas in the supplementary materials. Note that in our work, $G$ is not directly assumed, but instead derived based on Assumption~\ref{assum_1}. 
\begin{theorem}\label{theorem_1}
Let Assumptions~\ref{assum_1},\ref{assum_3} and~\ref{assum_4} hold. Let the momentum coefficient $\beta=\frac{96L^2+96G^2C_\upsilon}{N}\eta^2$. If \texttt{MDPGT} is initialized by a mini-batch of trajectories with the size being $\mathcal{B}$ and the step size satisfies the following condition
\begin{equation*}\begin{split}0<\eta\leq\textnormal{min}&\bigg\{\frac{(1-\lambda^2)^2}{\lambda\sqrt{12844L^2+9792G^2C_\upsilon}},\frac{\sqrt{N(1-\lambda^2)}\lambda}{31\sqrt{L^2+G^2C_\upsilon}},\\&\frac{1}{6\sqrt{6(L^2+G^2C_\upsilon)}}\bigg\},\end{split}\end{equation*} then the output $\tilde{\mathbf{x}}_K$ satisfies: for all $K\geq 2$:
\begin{equation}\label{eq.14}
\begin{split}
    &\mathbb{E}[\|\nabla J(\tilde{\mathbf{x}}_K)\|^2]\leq \frac{4(J^*-J(\bar{\mathbf{x}}_0))}{\eta K}+\frac{4\bar{\sigma}^2}{N|\mathcal{B}|\beta K}+\frac{8\beta\bar{\sigma}^2}{N}\\&+\frac{34\lambda^2}{KN(1-\lambda^2)^3}\|\nabla \mathbf{J}(\bar{\mathbf{x}}_0)\|^2+\frac{68\lambda^2\bar{\sigma}^2}{(1-\lambda^2)^3|\mathcal{B}|K}+\frac{204\lambda^2\beta^2\bar{\sigma}^2}{(1-\lambda^2)^3},
\end{split}
\end{equation}
where $J^*$ is the upper bound of $J(\mathbf{x})$ and $\|\nabla \mathbf{J}(\bar{\mathbf{x}}_0)\|^2\triangleq\sum^N_{i=1}\|\nabla J_i(\bar{\mathbf{x}}_0)\|^2$.
\end{theorem}
Theorem~\ref{theorem_1} depicts that when $K\to\infty$, \texttt{MDPGT} enables convergence to a \textit{steady-state error} in a sublinear rate $\mathcal{O}(1/K)$ if $\eta$ and $\beta$ are selected properly, i.e.,
\begin{equation}\label{sss}
    \mathbb{E}[\|\nabla J(\tilde{\mathbf{x}}_K)\|^2]\leq\frac{8\beta\bar{\sigma}^2}{N}+\frac{204\lambda^2\beta^2\bar{\sigma}^2}{(1-\lambda^2)^3}.
\end{equation}
By observing Eq.~\ref{sss}, the steady-state error is determined by the number of agents, the variance of stochastic policy gradient, and the spectral gap of the graph $1-\lambda$. Increasing the number of agents leads to a small error bound. Though different network topologies imply different error bounds, the higher order term of $\beta$ can reduce the impact of the spectral gap on the error bound. Another suggestion from Eq.~\ref{sss} is that $\eta$ and $\beta$ can be reduced to make the steady-state error arbitrarily small, while in return this can affect the speed of convergence. Surprisingly, even though we have to adopt the bounded stochastic policy gradient derived from Assumption~\ref{assum_1} for analysis, the error bound in Eq.~\ref{sss} only depends heavily on the variance, which is inherently consistent with most conclusions from decentralized optimization in literature without the bounded stochastic gradient assumption. While $J^*$ is essentially correlated to the upper bound of reward $R$, in this context, we still adopt the implicit $J^*$ for convenience. We next provide the analysis for the non-asymptotic behavior, defining appropriately $\eta, \beta$, and $|\mathcal{B}|$.
\begin{corollary}\label{coro_1}
Let $\eta=\frac{N^{2/3}}{8LK^{1/3}}, \beta=\frac{DN^{1/3}}{64L^2K^{2/3}}, |\mathcal{B}|=\ceil[bigg]{\frac{K^{1/3}}{N^{2/3}}}$ in Theorem~\ref{theorem_1}. We have,
\begin{equation}\label{eq_coro1}
\begin{split}
   &\mathbb{E}[\|\nabla J(\tilde{\mathbf{x}}_K)\|^2]\leq\underbrace{\frac{256L^3D(J^*-J(\bar{\mathbf{x}}_0))+2048L^4\bar{\sigma}^2+D^2\bar{\sigma}^2}{8L^2D(NK)^{2/3}}}_{T_1}\\&+\underbrace{\frac{34\lambda^2}{KN(1-\lambda^2)^3}\|\nabla \mathbf{J}(\bar{\mathbf{x}}_0)\|^2+\frac{\lambda^2\bar{\sigma}^2(51D^2+69632N^{2/3}L^4)}{1024(1-\lambda^2)^3K^{4/3}L^4}}_{T_2},
\end{split}
\end{equation}
for all
\begin{equation*}
    \begin{split}
       &K\geq\textnormal{max}\bigg\{\frac{N^2D^{1.5}}{512L^3}, \frac{29791\sqrt{N}(L^2+G^2C_\upsilon)^{1.5}}{512L^3\lambda^3(1-\lambda^2)^{1.5}},\\& \frac{(12844L^2+9792G^2C_\upsilon)^{1.5}N^2\lambda^3}{512L^3(1-\lambda^2)^6}\bigg\},
    \end{split}
\end{equation*}where
$D=96L^2+96G^2C_\upsilon$.
\end{corollary}
\begin{remark}\label{remark_1}
An implication from Corollary~\ref{coro_1} is that at the early stage of optimization, before $T_1$ in Eq.~\ref{eq_coro1} dominates, the complexity is tightly related to the algebraic connectivity of the network topology in $T_2$, which is measured by the spectral gap $1-\lambda$. However, $T_2$ is in a large order of $1/K$. As the optimization moves towards the latter stage where $T_1$ dominates, the overall complexity is independent of the network topology.
\end{remark}
For the ease of exposition, with Corollary~\ref{coro_1}, when $K$ is sufficiently large, it is an immediate consequence as $\mathbb{E}[\|\nabla J(\tilde{\mathbf{x}}_K)\|^2]\leq\mathcal{O}((NK)^{-2/3})$. Thus, for achieving $\mathbb{E}[\|\nabla J(\tilde{\mathbf{x}}_K)\|]\leq\epsilon$, the following relationship is obtained:
\[
\begin{aligned}
\mathbb{E}[\|\nabla J(\tilde{\mathbf{x}}_K)\|] &=\sqrt{(\mathbb{E}[\|\nabla J(\tilde{\mathbf{x}}_K)\|])^2} \leq\sqrt{\mathbb{E}[\|\nabla J(\tilde{\mathbf{x}}_K)\|^2]}\leq\epsilon.
\end{aligned}
\]Combining the last two inequalities results in the ultimate sampling complexity, i.e., $\mathcal{O}(N^{-1}\epsilon^{-3})$, which exhibits \textit{linear speed up}. More importantly, this is $N$ times smaller than the sampling complexity of the centralized momentum-based policy gradient methods~\cite{huang2020momentum} that performs on a single node. However, we have known from Corollary~\ref{coro_1} that typically $K$ has to be large, which will in the linear speed up regime reduce $\eta$ and $\beta$.

We also investigate a worse initialization with only a single trajectory sampled from $p_i(\tau^i|\mathbf{x}^i_0)$. However, without a mini-batch initialization, the eventual sampling complexity is $\mathcal{O}(N^{-1}\epsilon^{-4})$ (see Theorem~\ref{theorem_2} and Corollary~\ref{coro_2}). Though variance reduction techniques have not reduced the order of $\epsilon^{-1}$, compared to the SOTA approaches, the linear speed up still enables the complexity to be $N$ times smaller than that in~\cite{xu2019sample,huang2020momentum}. Additionally, different from traditional decentralized learning problems, MARL has more significant variances in the optimization procedure due to the non-oblivious characteristic. Using just a single trajectory for each agent to initialize is can be a poor scheme, but the adopted variance reduction techniques can successfully maintain the SOTA sampling complexity in a decentralized setting. Please refer to the supplementary materials for formal results and proof.

\textbf{Implication for Gaussian Policy.}
We study the sample complexity when the policy function $\pi^i(a^i|s)$ of each agent is explicitly a Gaussian distribution. For a bounded action space $\mathcal{A}^i\subset\mathbb{R}$, a Gaussian policy parameterized by $\mathbf{x}_i$ is defined as
\begin{equation}\label{gaussion_dis_1}
    \pi^i(a^i|s)=\frac{1}{\sqrt{2\pi}}\textnormal{exp}\bigg(-\frac{((\mathbf{x}^i)^\top\phi_i(s)-a^i)^2}{2\xi^2}\bigg),
\end{equation}
where $\xi^2$ is a constant standard deviation parameter and $\phi_i(s):\mathcal{S}\to\mathbb{R}^{d_i}$ is a mapping from the state space to the feature space. Thus, the following formal result can be obtained. A more formal analysis and proof can be seen in the supplementary materials.
\begin{corollary}\label{coro_3}
Let $\pi^i(a^i|s)$ be defined as a Gaussian distribution in Eq.~\ref{gaussion_dis_1} with $|a^i|\leq C_a$, where $C_a, C_f>0$, and $\|\phi_i(s)\|\leq C_f$, and $\eta, \beta, |\mathcal{B}|$ be defined as in Corollary~\ref{coro_1}. The sampling complexity of attaining the accuracy $\mathbb{E}[\|\nabla J(\tilde{\mathbf{x}}_K)\|]\leq \epsilon$ is $\mathcal{O}((1-\gamma)^{-4.5}N^{-1}\epsilon^{-3})$.
\end{corollary}
\section{Numerical Experiments and Results}\label{experiment}

To validate our proposed algorithm, we performed experiments on a multi-agent benchmark environment with a cooperative navigation task that has been commonly used as a benchmark in several previous works~\cite{qu2019value,zhang2018fully,lu2021decentralized}. Our platform for cooperative navigation is derived off the particle environment introduced by~\cite{lowe2017multi}. In our modification, all agents are initialized at random locations with a specific goal in a 2-dimensional grid world. Each agent observes the combined position and velocity of itself and all other agents. The agent is capable of moving up, down, left or right with the objective of navigating to their respective goals. The reward function of each agent is defined as the negative euclidean distance of the agent to the goal. Additionally, a penalty of -1 is imposed whenever the agent collides with any other agents. All agent's policy is represented by a 3-layer dense neural network with 64 hidden units with $tanh$ activation functions. The agents were trained for 50,000 episodes with a horizon of 50 steps and discount factor of 0.99. For the sake of brevity, we present numerical results in only one environment setting with five agents. Additional results with different number of agents and a simplified environment and computing infrastructure details are available in the supplementary materials\footnote{Codes to reproduce these results are available at the following repository: https://github.com/xylee95/MD-PGT}.
 
\subsection{Efficacy of MDPGT}

\begin{figure}
    \centering
    \includegraphics[width=0.8\columnwidth]{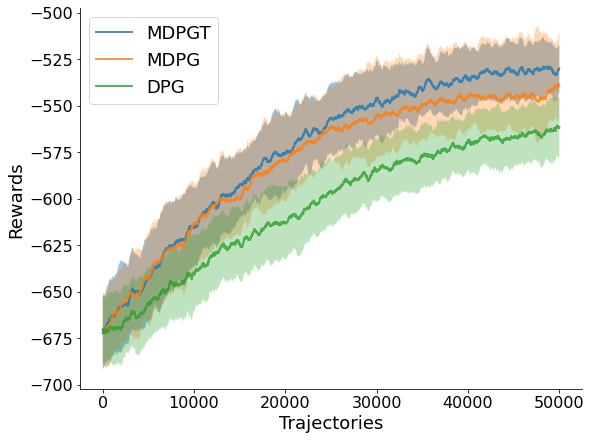}
    \caption{Average rewards obtain by \texttt{MDPGT}, \texttt{MDPG} and \texttt{DPG} in a cooperative navigation task for five agents. For \texttt{MDPGT}, $\beta$ = 0.5 and results shown are averaged across five random seeds. The line plots denote the mean value and shaded regions denote the standard deviation of rewards.}
    \label{fig:five_agents}
\end{figure}

Figure~\ref{fig:five_agents} illustrates the average training rewards obtained by the five agents in the cooperative navigation gridworld environment. As observed, both \texttt{MDPG} and \texttt{MDPGT} significantly outperforms the baseline, denoted as \texttt{DPG}. Comparing \texttt{MDPG} with \texttt{MDPGT}, we observe that while both algorithms initially have similar performance, \texttt{MDPGT} begins to outperform \texttt{MDPG} around the 15,000 iteration. Additionally, when we compare the standard deviation of the rewards, shown as shaded regions, we observe that standard deviation of \texttt{MDPGT} is also smaller than the standard deviation of \texttt{MDPG}. In summary, these results validate our theoretical findings that utilizing gradient tracking as bias correction technique does improve the performance of the algorithm. Additionally, the improvement in terms of sampling complexity over \texttt{DPG} is empirically evident through the result.

\subsection{Effect of Momentum Coefficient}

\begin{figure}
    \centering
    \includegraphics[width=0.8\columnwidth]{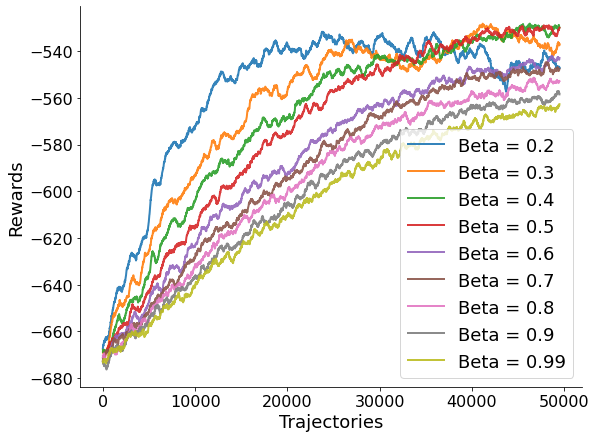}
    \caption{Ablation study illustrating the effect of various momentum coefficients, $\beta$ on the performance of \texttt{MDPGT} for five agents in the cooperative navigation environment.}
    \label{fig:momentum_coefficient}
\end{figure}

\begin{figure}
    \centering
    \includegraphics[width=0.8\columnwidth]{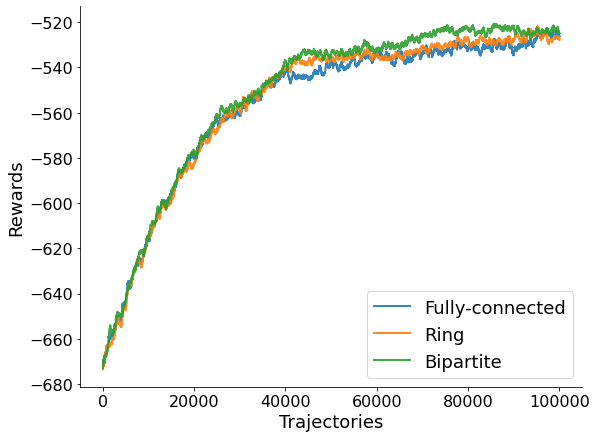}
    \caption{Experiment results for five agents in the cooperative navigation environment to compare the effects of different network topologies. $\beta$ = 0.5 for all experiments shown.}
    \label{fig:topologies}
\end{figure}

Next, we perform an additional ablation study to investigate the effect of the momentum coefficient $\beta$ on the performance of \texttt{MDPGT}. As shown in Figure~\ref{fig:momentum_coefficient}, we see that the choice of momentum coefficient does indeed have an effect on the performance. A $\beta$ that is low can induce a faster convergence rate, but at the cost of a higher fluctuations in rewards, as seen by $\beta$ = 0.2 and 0.3. Conversely, a $\beta$ value that is too high will cause the surrogate to degenerate into vanilla policy gradients and reflects a similar performance as $\texttt{DPG}$, which matches the implication by Eq.~\ref{sss}. Ultimately, we see that for this environment, $\beta$ = 0.4 and 0.5 offers the perfect balance between convergence rate and stability/variance of the training. Hence, $\beta$ can be viewed as hyper-parameter which can be tuned to trade off between optimizing for convergence versus training stability. 

\subsection{Effect of Different Topologies}

Finally, we provide evidence which confirms the fact that our proposed method holds under various networks topologies. To test our hypothesis, we train five agents in the same cooperative navigation environment using three different network topologies: fully-connected, ring and bi-partite topology. As seen in Figure~\ref{fig:topologies}, the five agents achieves similar rewards despite communicating via different network topologies. This validates our claim in Remark~\ref{remark_1}.

\section{Conclusions}\label{conclusions}
This paper proposes a novel MARL algorithm that involves variance reduction techniques to reduce the sampling complexity of decentralized policy-based methods. Specifically we have developed the algorithmic framework and analyzed it in a principled manner. An importance sampling-based stochastic recursive momentum is presented as the policy gradient surrogate, which is taken as input to a policy gradient tracker. Through theoretical analysis, we have found that the proposed method can improve the sampling efficiency in the decentralized RL settings compared to the SOTA methods. To the best of our knowledge, this is the first time to achieve the best available rate, $\mathcal{O}(\epsilon^{-3})$, for generic  (possibly non-concave) performance functions. Empirical results have shown the superiority of the proposed \texttt{MDPGT} over the baseline decentralized policy gradient methods. Future research directions include: 1) incorporating more complex decentralized environments in our experiments to reveal potentially novel and interesting results; 2) extending the proposed method to model-based decentralized RL settings to improve further the sampling efficiency; 3) theoretically analyzing the robustness of the proposed method under adversarial attacks.




\bibliography{references.bib}
\onecolumn
\section*{Supplementary Materials}
We present the additional analysis and detailed proof for the auxiliary lemmas and main theorems as well as additional empirical results.
\subsection*{Additional Related Works}
The most recent work~\cite{zhang2021convergence} in centralized RL introduced \texttt{SIVR-PG}, which employed the hidden convex nature of general utility function and leveraged techniques from composition optimization to attain the sampling complexity of $\tilde{\mathcal{O}}(\epsilon^{-2})$. Nevertheless, \texttt{SIVR-PG} requires an additional assumption that the unnormalized
state-action occupancy measure in the utility function is convex and its corresponding inverse mapping exists, which may not rigorously hold for some RL settings.
A concurrent work~\cite{zhang2021marl} recently that has been published on arXiv proposes a Decentralized Shadow Reward Actor Critic (\texttt{DSAC}) for general utilities, which can be treated as the extension to the decentralized setting of \texttt{SIVR-PG}. In their work, the authors develop the shadow reward that estimates the derivative of the local utility with
respect to their occupancy measure and show that \texttt{DSAC} converges to $\epsilon$-stationarity in $\mathcal{O}(\epsilon^{-2.5})$ or even faster $\mathcal{O}(\epsilon^{-2})$ with high probability. They further establish the global optimality by adopting the diminishing step size. Though \texttt{DSAC} reveals faster convergence rate, additional occupancy measure needs to be incorporated and updated in the algorithmic framework. Also, the shadow reward they have defined is essentially a derivative, which may vary significantly among different agents. This could be another source of variance during updates. In the analysis, we have observed that more assumptions are required to arrive at the theoretical claims, including strongly convex critic objective function. Different from most previous works, the constructed error bounds in their work are not based on expectation. The parameters, including step size, batch size, trajectory length, and even final iteration need to be carefully defined. In the empirical evaluation, the authors present a cooperative navigation with safety constraint and do not show any comparison with the decentralized baseline methods. We therefore are not going to compare our proposed method with theirs due to different emphasis.
\subsection*{\texttt{Hybrid-SARAH} and \texttt{GT}}
\textbf{\texttt{Hybrid-SARAH}}. In this context, we denote by $\mathbf{g}(\zeta,\mathbf{x})$ the stochastic gradient, where $\zeta$ is the random seed. As discussed above, \texttt{Hybrid-SARAH} is the combination between vanilla stochastic gradient and \texttt{SARAH} such that the following expression can be obtained for a time step $k$:
\begin{equation}\label{storm}
    \mathbf{u}_{k} = \beta \mathbf{g}(\zeta_k, \mathbf{x}_k) + (1-\beta)(\mathbf{u}_{k-1} + \mathbf{g}(\zeta_k,\mathbf{x}_k) - \mathbf{g}(\zeta_k,\mathbf{x}_{k-1})),
\end{equation}
where $\beta\in[0,1]$ is the momentum parameter. One can observe from Eq.~\ref{storm} that in the part of \texttt{SARAH}, the term $\mathbf{g}(\zeta_k,\mathbf{x}_{k-1})$ requires the computation of $\mathbf{g}(\zeta_k,\mathbf{x}_{k-1})$ and the copy of last iteration of parameters $\mathbf{x}_{k-1}$, which could increase the computational power and memory significantly if the dimension of $\mathbf{x}$ is quite high, though this may not be an issue for a modern computer. Additionally, this form of update in \texttt{SARAH} enables to improve the convergence rates in different problems such as decentralized learning~\cite{xin2020near} and reinforcement learning~\cite{huang2020momentum,xu2019sample}. Therefore we will still adapt \texttt{Hybrid-SARAH} in our MARL algorithms.

\textbf{\texttt{GT}.} Another technique that has specifically been developed for decentralized optimization is \texttt{GT}, which is close to \texttt{SARAH}, but focusing on tracking and correcting each agent's locally aggregated gradients. However, in terms of formulation, \texttt{GT} typically appears in a form with a consensus or mixing step that can expressed as follows:
\begin{equation}\label{gt}
    \mathbf{v}^i_{k+1} = \sum_{j\in Nb(i)}\omega_{ij}\mathbf{v}^j_k+\mathbf{g}_i(\zeta_k^i,\mathbf{x}^i_k) - \mathbf{g}_i(\zeta_{k-1}^i,\mathbf{x}^i_{k-1}),
\end{equation}
where $\omega_{ij}\in\mathbf{W}\in\mathbb{R}^{N\times N}$ is the probability of an edge in the communication network $\mathcal{G}$. $\mathbf{W}$ is the mixing matrix to define the topology of the communication network and one assumption will be imposed for the matrix. An immediate observation from Eq.~\ref{gt} is that the update requires the copy of $\mathbf{g}_i(\zeta_{k-1}^i,\mathbf{x}^i_{k-1})$ from last iteration without the extra computation as in \texttt{SARAH}. The difference between \texttt{GT} and \texttt{SARAH} is not the focus of this work and instead we will combine both \texttt{Hybrid-SARAH} and \texttt{GT} together to reduce the stochastic gradient variance and correct the stochastic gradient bias in the MARL.
\subsection*{Additional Analysis for Algorithm~\ref{mdpgt}}
We discuss briefly how \texttt{MDPGT} enables to reduce the policy gradient variance. Let $\Delta^i_k\triangleq\mathbf{u}^i_k-\nabla J_i(\mathbf{x}^i_k)$. With simple mathematical manipulations, we can obtain
\begin{equation}
\begin{split}
    &\mathbb{E}[\Delta^i_k]=\mathbb{E}[(1-\beta)\Delta^i_{k-1}+\beta(\underbrace{\mathbf{g}_i(\tau^i_k|\mathbf{x}^i_k)-\nabla J_i(\mathbf{x}^i_k)}_{T_3})+\\&(1-\beta)(\underbrace{\mathbf{g}_i(\tau^i_k|\mathbf{x}^i_k)-\upsilon_i(\tau^i_k|\mathbf{x}^i_{k-1},\mathbf{x}^i_k)\mathbf{g}_i(\tau^i_k|\mathbf{x}^i_{k-1})-\nabla J_i(\mathbf{x}^i_k)+\nabla J_i(\mathbf{x}^i_{k-1})}_{T_4})].
\end{split}
\end{equation}
We can observe from the last equation that $\mathbb{E}[\Delta^i_k]=(1-\beta)\mathbb{E}[\Delta^i_{k-1}]$ due to $\mathbb{E}_{\tau^i_k\sim p_i(\tau^i|\mathbf{x}^i_k)}[T_3]=0$ and $\mathbb{E}_{\tau^i_k\sim p_i(\tau^i|\mathbf{x}^i_k)}[T_4]=0$. Applying Cauthy-Schwartz inequality, we further attain the upper bound for $\mathbb{E}[\|\Delta^i_k\|^2]$ as follows,
\begin{equation}\label{variance_bound}
\begin{split}
\mathbb{E}[\|\Delta^i_k\|^2]&\leq (1-\beta)^2\mathbb{E}[\|\Delta^i_{k-1}\|^2]+2\beta^2\mathbb{E}[\|T_3\|^2]\\&+2(1-\beta)^2\mathbb{E}[\|T_4\|^2].
\end{split}
\end{equation}
In the sequel, we will establish the relationship between the importance sampling weight $\upsilon_i(\tau^i_k|\mathbf{x}^i_{k-1},\mathbf{x}^i_k)$ and $\|\mathbf{x}^i_k-\mathbf{x}^i_{k-1}\|^2$ such that the following equality can be obtained
\begin{equation}\label{t2}
\begin{split}
    \mathcal{O}(\|T_4\|^2)&=\mathcal{O}(\|\mathbf{x}^i_k-\mathbf{x}^i_{k-1}\|^2)\\&=\mathcal{O}(\|\mathbf{x}^i_k+\bar{\mathbf{x}}_k-\bar{\mathbf{x}}_k+\bar{\mathbf{x}}_{k-1}-\bar{\mathbf{x}}_{k-1}-\mathbf{x}^i_{k-1}\|^2).
\end{split}
\end{equation}
Eq.~\ref{t2} is bounded above by the following
\[\mathcal{O}(3\|\mathbf{x}^i_k-\bar{\mathbf{x}}_k\|^2+3\|\mathbf{x}^i_{k-1}-\bar{\mathbf{x}}_{k-1}\|^2+3\eta^2\|\bar{\mathbf{u}}_{k-1}\|^2),\]
where $\bar{*}$ is the ensemble average of all agents $i\in\mathcal{V}$ and the first two terms signify the consensus estimate, which can be controlled by $\beta$ and $\eta$. The bound also follows from the fact that $\bar{\mathbf{u}}_{k-1} = \bar{\mathbf{v}}_k$ (see Lemma~\ref{lemma_4}). Therefore, in light of Eq.~\ref{variance_bound}, setting $\beta$ and $\eta$ properly is able to reduce the variance. This also motivates the adoption of policy gradient surrogate in the work.
\subsection*{Auxiliary Lemmas and Their Proof}
We start with the results for the smoothness constant $L$ and the upper bound of $\|\mathbf{g}_i(\tau^i|\mathbf{x}^i)\|$, $G$.
\begin{lemma}\label{lemma_1}
Let $\mathbf{g}_i(\tau^i|\mathbf{x}^i)$ be defined in Eq.~\ref{gradient_estimator} for all $i\in\mathcal{V}$. Assumption~\ref{assum_1} implies the following conclusions:
\begin{itemize}
    \item $\|\mathbf{g}_i(\tau^i|\mathbf{x}_1)-\mathbf{g}_i(\tau^i|\mathbf{x}_2)\|\leq L\|\mathbf{x}_1-\mathbf{x}_2\|, \forall \mathbf{x}_1, \mathbf{x}_2\in\mathbb{R}^{d_{i}}$, with $L=\frac{C_hR}{(1-\gamma)^2}$;
    \item $J({\mathbf{x}})$ is $L$-smooth, i.e., $\|\nabla^2J(\mathbf{x})\|\leq L$;
    \item $\|\mathbf{g}_i(\tau^i|\mathbf{x}^i)\|\leq G, \forall \mathbf{x}^i\in\mathbb{R}^{d_i}$, with $G=\frac{C_gR}{(1-\gamma)^2}$;
    \item Variance of stochastic policy gradient $\mathbf{g}_i(\tau^i|\mathbf{x}^i)$ is bounded, i.e., $\mathbb{E}[\|\mathbf{g}_i(\tau^i|\mathbf{x}^i)-\nabla J_i(\mathbf{x}^i)\|^2]\leq \sigma^2_i, \bar{\sigma}^2=\frac{1}{N}\sum^N_{i=1}\sigma^2_i$, for all $i\in\mathcal{V}$, where $\bar{\sigma}^2=\frac{C_g^2R^2}{(1-\gamma)^4}$.
\end{itemize}
\end{lemma}
\begin{proof}
Recall the definition of $\mathbf{g}_i(\tau^i|\mathbf{x}^i)$ in Eq.~\ref{gradient_estimator} such that the PGT estimator~\cite{sutton1999policy} is as follows:
\[\mathbf{g}_i(\tau^i|\mathbf{x}^i)=\sum^{H-1}_{h=0}\sum^{h}_{q=0}(\gamma^qr^i_q(a^i_q,s_q)-b_q)\nabla_{\mathbf{x}^i}\textnormal{log}\pi^i(a^i_h, s_h),\]where $b_q$ is a constant baseline and we specify the subscript $\mathbf{x}^i$ to indicate the gradient w.r.t it. Later this is omitted unless specified appropriately. Further, we have
\begin{equation}
\begin{split}
    \|\nabla \mathbf{g}_i(\tau^i|\mathbf{x}^i)\|&=\bigg\|\sum^{H-1}_{h=0}\nabla^2\textnormal{log}\pi^i(a^i_h, s_h)\bigg(\sum^{h}_{q=0}\gamma^qr^i_q(a^i_q,s_q)\bigg)\bigg\|\\&\leq\bigg(\sum^{H-1}_{q=0}\|\nabla^2\textnormal{log}\pi^i(a^i_q, s_q)\|\frac{R}{1-\gamma}\bigg)\\&\leq\frac{C_hR}{(1-\gamma)^2}
\end{split}
\end{equation}
where $b_q=0$. The above inequalities follow from that $\gamma\in(0,1)$. When $b_q\neq 0$, we can easily scale it with $\gamma^h$ and the above result still holds but with different constant coefficient. Thus, the first part of Lemma~\ref{lemma_1} can be proved.

As the PGT estimator is an unbiased estimator of the policy gradient $\nabla J_i(\mathbf{x}^i)$, we then have \[\nabla J_i(\mathbf{x}^i)=\mathbb{E}_{\tau^i}[\mathbf{g}_i(\tau^i|\mathbf{x}^i)], \nabla^2 J_i(\mathbf{x}^i)=\mathbb{E}_{\tau^i}[\nabla\mathbf{g}_i(\tau^i|\mathbf{x}^i)].\] Hence, the smoothness of $J_i(\mathbf{x}^i)$ can be directly implied from the Lipchitzness of $\mathbf{g}_i(\tau^i|\mathbf{x}^i)$,
\[\|\nabla^2 J_i(\mathbf{x}^i)\|=\|\mathbb{E}_{\tau^i}[\nabla\mathbf{g}_i(\tau^i|\mathbf{x}^i)]\|\leq\|\nabla\mathbf{g}_i(\tau^i|\mathbf{x}^i)\|\leq\frac{C_hR}{(1-\gamma)^2}.\] As the above relationship holds for each agent $i\in\mathcal{V}$, the second part of Lemma~\ref{lemma_1} can be obtained.

Similarly, taking the norm of $\mathbf{g}_i(\tau^i|\mathbf{x}^i)$ leads to
\[\|\mathbf{g}_i(\tau^i|\mathbf{x}^i)\|\leq\bigg\|\sum^{H-1}_{h=0}\nabla\textnormal{log}\pi^i(a^i_h, s_h)\frac{\gamma^hR(1-\gamma^{H-h})}{1-\gamma}\bigg\|\leq\frac{C_gR}{(1-\gamma)^2}.\] Based on the above inequality, it is immediately obtained that \[\|\mathbf{g}_i(\tau^i|\mathbf{x}^i)\|^2\leq\frac{C_g^2R^2}{(1-\gamma)^4}.\] For a random variable $X$, we have that $\mathbb{E}[\|X-\mathbb{E}[X]\|^2]\leq\mathbb{E}[\|X\|^2]$ such that \[\mathbb{E}[\|\mathbf{g}_i(\tau^i|\mathbf{x}^i)-\nabla J_i(\mathbf{x}^i)\|^2]\leq\frac{C_g^2R^2}{(1-\gamma)^4}.\] Also, $\bar{\sigma}^2=\frac{C_g^2R^2}{(1-\gamma)^4}$.
The proof is complete.
\end{proof}
Lemma~\ref{lemma_1} completes the definitions of $L$ and $G$. One immediate observation for both $L$ and $G$ is that they are independent of the horizon $H$, which hence are tighter. In the following, we introduce a relation for the importance sampling weight. We first introduce the R\'enyi divergence between two distributions $Q$ and $Z$ as follows:
\[\mathcal{D}_\alpha(Q||Z)=\frac{1}{\alpha-1}\textnormal{log}_2\int_xQ(x)\bigg(\frac{Q(x)}{Z(x)}\bigg)^{\alpha-1}dx,\] which is non-negative for all $\alpha>0$. The exponentiated R\'enyi divergence is $\mathcal{D}_{\alpha}(Q||Z)=2^{D_\alpha(Q||Z)}$.
\begin{lemma}({Lemma 1 in~\cite{cortes2010learning}})\label{lemma_2}
Let $\upsilon_i(\tau^i|\mathbf{x}_1, \mathbf{x}_2) = \frac{p_i(\tau^i|\mathbf{x}_1)}{p_i(\tau^i|\mathbf{x}_2)}$ be the importance sampling weight for distributions $p_i(\tau^i|\mathbf{x}_1)$ and $p_i(\tau^i|\mathbf{x}_2)$. Then the expectation, second moment and variance of $\upsilon_i(\tau^i|\mathbf{x}_1, \mathbf{x}_2)$ satisfy the following results: \[\mathbb{E}[\upsilon_i(\tau^i|\mathbf{x}_1, \mathbf{x}_2)] = 1, \mathbb{E}[\upsilon_i^2(\tau^i|\mathbf{x}_1, \mathbf{x}_2)]=\mathcal{D}_2(p_i(\tau^i|\mathbf{x}_1)||p_i(\tau^i|\mathbf{x}_2)),\] and \[\mathbb{V}(\upsilon_i(\tau^i|\mathbf{x}_1, \mathbf{x}_2))=\mathcal{D}_2(p_i(\tau^i|\mathbf{x}_1)||p_i(\tau^i|\mathbf{x}_2))-1.\]
\end{lemma}
With Lemma~\ref{lemma_2} in hand, we now are ready to state the relation between $\mathbb{V}(\upsilon_i(\tau^i|\mathbf{x}_1, \mathbf{x}_2))$ and $\|\mathbf{x}_1-\mathbf{x}_2\|^2$.
\begin{lemma}\label{lemma_3}
For any $\mathbf{x}_1,\mathbf{x}_2\in\mathbb{R}^{d_i}, \forall i\in\mathcal{V}$, let $\upsilon_i(\tau^i|\mathbf{x}_1, \mathbf{x}_2) = \frac{p_i(\tau^i|\mathbf{x}_1)}{p_i(\tau^i|\mathbf{x}_2)}$. With Assumptions~\ref{assum_1} and~\ref{assum_4}, We have 
\begin{equation}
    \mathbb{V}[\upsilon_i(\tau^i|\mathbf{x}_1, \mathbf{x}_2)]\leq C_\upsilon\|\mathbf{x}_1-\mathbf{x}_2\|^2,
\end{equation}
where $C_\upsilon=H(2HC_g^2+C_h)(\mathcal{M}+1)$.
\end{lemma}
\begin{proof}
According to Lemma~\ref{lemma_2}, the variance of the importance sampling weight is
\[\mathbb{V}[\upsilon_i(\tau^i|\mathbf{x}_1,\mathbf{x}_2)]=\mathcal{D}_2(p_i(\tau^i|\mathbf{x}_1)||p_i(\tau^i|\mathbf{x}_2))-1.\]By definition, the following relationship can be obtained
\[\mathcal{D}_2(p_i(\tau^i|\mathbf{x}_1)||p_i(\tau^i|\mathbf{x}_2))=\int_{\tau^i}p_i(\tau^i|\mathbf{x}_1)\frac{p_i(\tau^i|\mathbf{x}_1)}{p_i(\tau^i|\mathbf{x}_2)}d\tau^i=\int_{\tau^i}p_i(\tau^i|\mathbf{x}_1)^2p_i(\tau^i|\mathbf{x}_2)^{-1}d\tau^i.\] Taking the gradient of $\mathcal{D}_2(p_i(\tau^i|\mathbf{x}_1)||p_i(\tau^i|\mathbf{x}_2))$ w.r.t. $\mathbf{x}_1$, we then have
\[\nabla_{\mathbf{x}_1}\mathcal{D}_2(p_i(\tau^i|\mathbf{x}_1)||p_i(\tau^i|\mathbf{x}_2))=2\int_{\tau^i}p_i(\tau^i|\mathbf{x}_1)\nabla_{\mathbf{x}_1} p_i(\tau^i|\mathbf{x}_1) p_i(\tau^i|\mathbf{x}_2)^{-1}d\tau^i.\] Letting $\mathbf{x}_1=\mathbf{x}_2$ in the last equation yields \[\nabla_{\mathbf{x}_1}\mathcal{D}_2(p_i(\tau^i|\mathbf{x}_1)||p_i(\tau^i|\mathbf{x}_2))\bigg|_{\mathbf{x}_1=\mathbf{x}_2}=2\int_{\tau^i}\nabla_{\mathbf{x}_1} p_i(\tau^i|\mathbf{x}_1)d\tau^i\bigg|_{\mathbf{x}_1=\mathbf{x}_2}=0.\] Applying the mean value theorem w.r.t. $\mathbf{x}_1$ results in\begin{equation}\label{mean_value_theorem}\mathcal{D}_2(p_i(\tau^i|\mathbf{x}_1)||p_i(\tau^i|\mathbf{x}_2))=1+\frac{1}{2}(\mathbf{x}_1-\mathbf{x}_2)^\top\nabla^2_{\mathbf{x}_3}\mathcal{D}_2(p_i(\tau^i|\mathbf{x}_3)||p_i(\tau^i|\mathbf{x}_2))(\mathbf{x}_1-\mathbf{x}_2),\end{equation}where $\mathbf{x}_3=t\mathbf{x}_1+(1-t)\mathbf{x}_2, t\in[0,1]$. The last inequality follows from the fact that $\mathcal{D}_2(p_i(\tau^i|\mathbf{x}_2)||p_i(\tau^i|\mathbf{x}_2))=1$. We have now obtained another expression for $\mathcal{D}_2(p_i(\tau^i|\mathbf{x}_1)||p_i(\tau^i|\mathbf{x}_2))$ and to bound it, we shall compute the Hessian matrix $\nabla^2_{\mathbf{x}_3}\mathcal{D}_2(p_i(\tau^i|\mathbf{x}_3)||p_i(\tau^i|\mathbf{x}_2))$. Taking the gradient $\nabla_{\mathbf{x}_3}\mathcal{D}_2(p_i(\tau^i|\mathbf{x}_3)||p_i(\tau^i|\mathbf{x}_2))$ of w.r.t. $\mathbf{x}_3$ leads to \[\begin{split}\nabla^2_{\mathbf{x}_3}\mathcal{D}_2(p_i(\tau^i|\mathbf{x}_3)||p_i(\tau^i|\mathbf{x}_2))&=2\int_{\tau^i}\nabla_{\mathbf{x}_3}\textnormal{log} p_i(\tau^i|\mathbf{x}_3)\nabla_{\mathbf{x}_3}\textnormal{log} p_i(\tau^i|\mathbf{x}_3)^\top\frac{p_i(\tau^i|\mathbf{x}_3)^2}{p_i(\tau^i|\mathbf{x}_2)} d\tau^i\\&+2\int_{\tau^i}\nabla_{\mathbf{x}_3}^2p_i(\tau^i|\mathbf{x}_3)p_i(\tau^i|\mathbf{x}_3)p_i(\tau^i|\mathbf{x}_2)^{-1}d\tau^i.\end{split}\]The above equation implies that to obtain the Hessian matrix, we need to attain $\nabla_{\mathbf{x}_3}^2p_i(\tau^i|\mathbf{x}_3)$, which signifies the Hessian matrix of the trajectory distribution function. We next derive the the Hessian matrix of log-density function. \[\nabla_{\mathbf{x}_3}^2\textnormal{log}p_i(\tau^i|\mathbf{x}_3)=-p_i(\tau^i|\mathbf{x}_3)^2\nabla_{\mathbf{x}_3}p_i(\tau^i|\mathbf{x}_3)\nabla_{\mathbf{x}_3}p_i(\tau^i|\mathbf{x}_3)^\top+p_i(\tau^i|\mathbf{x}_3)^{-1}\nabla_{\mathbf{x}_3}^2p_i(\tau^i|\mathbf{x}_3).\] Combining the last two equations yields \[\begin{split}\|\nabla^2_{\mathbf{x}_3}\mathcal{D}_2(p_i(\tau^i|\mathbf{x}_3)||p_i(\tau^i|\mathbf{x}_2))\|&=\bigg\|4\int_{\tau^i}\nabla_{\mathbf{x}_3}\textnormal{log} p_i(\tau^i|\mathbf{x}_3)\nabla_{\mathbf{x}_3} p_i(\tau^i|\mathbf{x}_3)^\top\frac{p_i(\tau^i|\mathbf{x}_3)^2}{p_i(\tau^i|\mathbf{x}_2)} d\tau^i\\&+2\int_{\tau^i}\nabla^2_{\mathbf{x}_3}\textnormal{log} p_i(\tau^i|\mathbf{x}_3)\frac{p_i(\tau^i|\mathbf{x}_3)^2}{p_i(\tau^i|\mathbf{x}_2)} d\tau^i\bigg\|\\&\leq\int_{\tau^i}\frac{p_i(\tau^i|\mathbf{x}_3)^2}{p_i(\tau^i|\mathbf{x}_2)}(4\|\nabla_{\mathbf{x}_3}\textnormal{log} p_i(\tau^i|\mathbf{x}_3)\|^2+2\|\nabla^2_{\mathbf{x}_3}\textnormal{log} p_i(\tau^i|\mathbf{x}_3)\|)d\tau^i\\&\leq(4H^2C^2_g+2HC_h)\mathbb{E}[\upsilon_i(\tau^i|\mathbf{x}_1,\mathbf{x}_2)^2]\\&\leq 2H(2HC^2_g+C_h)(\mathcal{M}+1),\end{split}\]where the second inequality follows from Assumption~\ref{assum_1} and Lemma~\ref{lemma_2} and the last inequality follows from Assumption~\ref{assum_4}. Therefore, substituting the above result into Eq.~\ref{mean_value_theorem} attain the following desirable result \[\mathbb{V}[\upsilon_i(\tau^i|\mathbf{x}_1, \mathbf{x}_2)]=\mathcal{D}_2(p_i(\tau^i|\mathbf{x}_1)||p_i(\tau^i|\mathbf{x}_2))-1\leq C_\upsilon\|\mathbf{x}_1-\mathbf{x}_2\|^2.\]
\end{proof}
Lemma~\ref{lemma_3} is critical in the proof as it will translate \texttt{HIS-SARAH} into a squared norm that can be bounded above. We next present the outline of how to prove Theorem~\ref{theorem_1} and apply the same proof techniques to show Theorem~\ref{theorem_2}. We first define some notations for the convenience of proof. In a generalized case, the dimension for each $\mathbf{x}^i$ is not necessarily the same, while in this work we assume that $d_1=d_2=...=d_N=\frac{d}{N}$ for the ease of exposition and in this context, $d/N$ is assumed to be an integer. Recalling the update laws for $\mathbf{v}^i_k$ and $\mathbf{x}^i_k$ in Algorithm~\ref{mdpgt} in a matrix form, we have:
\begin{subequations}
\begin{align}
    \mathbf{v}_{k+1} &= \underline{\mathbf{W}}\mathbf{v}_k+\mathbf{u}_k-\mathbf{u}_{k-1}\\
    \mathbf{x}_{k+1} &= \underline{\mathbf{W}}(\mathbf{x}_k+\eta\mathbf{v}_{k+1})
\end{align}
\end{subequations}
where $\underline{\mathbf{W}}\triangleq\mathbf{W}\otimes \mathbf{I}_{d/N}$ and $\mathbf{x}_k, \mathbf{v}_k, \mathbf{u}_k$ are square-integrable random vectors in $\mathbb{R}^d$ that concatenate the local estimates of the solution $\{\mathbf{x}^i_k\}_{i=1}^N$, gradient trackers $\{\mathbf{v}^i_k\}_{i=1}^N$, and the stochastic policy gradient surrogates $\{\mathbf{u}^i_k\}^N_{i=1}$. Additionally, we denote $\nabla \mathbf{J}(\mathbf{x}_k)=[\nabla J_1(\mathbf{x}_k^1)^\top,..., \nabla J_N(\mathbf{x}_k^N)^\top]^\top$ and define the ensemble averages as follows.
\[\bar{\mathbf{x}}_k\triangleq\frac{1}{N}(\mathbf{1}^\top_N\otimes \mathbf{I}_{d/N})\mathbf{x}_k, \bar{\mathbf{v}}_k\triangleq\frac{1}{N}(\mathbf{1}^\top_N\otimes \mathbf{I}_{d/N})\mathbf{v}_k,\]\[\bar{\mathbf{u}}_k\triangleq\frac{1}{N}(\mathbf{1}^\top_N\otimes \mathbf{I}_{d/N})\mathbf{u}_k, \overline{\nabla\mathbf{J}}(\mathbf{x}_k)\triangleq\frac{1}{N}(\mathbf{1}^\top_N\otimes \mathbf{I}_{d/N})\nabla\mathbf{f}(\mathbf{x}_k).\] With these definitions, we present some known results for gradient tracking type of algorithms in decentralized optimization and refer interested readers to~\cite{xin2021hybrid} for the detail of proof.
\begin{lemma}\label{lemma_4}
The following relationships hold for \texttt{MDPGT}.
\begin{itemize}
    \item Define $\mathbf{\Lambda} = \frac{1}{N}(\mathbf{1}^\top_N\otimes \mathbf{I}_{d/N})$. Thus, $\|\underline{\mathbf{W}}\mathbf{x}-\mathbf{\Lambda}\mathbf{x}\|\leq\lambda\|\mathbf{x}-\mathbf{\Lambda}\mathbf{x}\|$.
    \item $\bar{\mathbf{v}}_{k+1}=\bar{\mathbf{u}}_k$, for all $k\geq 0$.
    \item $\|\overline{\nabla\mathbf{J}}(\mathbf{x}_k)-\nabla J(\bar{\mathbf{x}}_k)\|^2\leq\frac{L^2}{N}\|\mathbf{x}_k-\mathbf{\Lambda}\mathbf{x}_k\|^2$, for all $k\geq 0$.
\end{itemize}
\end{lemma}
The first relationship in Lemma~\ref{lemma_4} has been well-known due to a doubly stochastic mixing matrix $\underline{\mathbf{W}}$ and the second one is an immediate consequence of the update law for the gradient tracker. The third relationship is due to the second conclusion from Lemma~\ref{lemma_1}. Multiplying the update for $\mathbf{x}^i_{k+1}$ by $\Lambda$ yields the following equality
\begin{equation}\label{average_state}
    \bar{\mathbf{x}}_{k+1} = \bar{\mathbf{x}}_{k} + \eta\bar{\mathbf{v}}_{k+1} = \bar{\mathbf{x}}_{k} + \eta\bar{\mathbf{u}}_{k}, \;\forall k\geq 0.
\end{equation}
We next establish a key technical lemma that sheds light on the convergence in terms of the second moment of $\nabla J(\bar{\mathbf{x}}_k)$.
\begin{lemma}\label{lemma_5}
Let $\bar{\mathbf{x}}_k$ be generated by Eq.~\ref{average_state}. If the step size $\eta\in(0,\frac{1}{2L}]$, then for all $K\geq 0$, we have:
\begin{equation}\label{descent}
\begin{split}
    &\sum_{k=0}^{K}\|\nabla J(\bar{\mathbf{x}}_k)\|^2\leq\frac{2(J^*-J(\bar{\mathbf{x}}_0))}{\eta}-\frac{1}{2}\sum_{k=0}^K\|\bar{\mathbf{u}}_k\|^2\\&+2\sum_{k=0}^K\|\bar{\mathbf{u}}_k-\overline{\nabla\mathbf{J}}(\mathbf{x}_k)\|^2+\frac{2L^2}{N}\sum_{k=0}^K\|\mathbf{x}_k-\Lambda\mathbf{x}_k\|^2.
\end{split}
\end{equation}
\end{lemma}
\begin{proof}
Based on the second conclusion of Lemma~\ref{lemma_1}, we have the following relationship:
\begin{equation}
\begin{split}
    J(\bar{\mathbf{x}}_{k+1})&\geq J(\bar{\mathbf{x}}_k) + \langle\nabla J(\bar{\mathbf{x}}_k),(\bar{\mathbf{x}}_{k+1}-\bar{\mathbf{x}}_k)\rangle-\frac{L}{2}\|\bar{\mathbf{x}}_{k+1}-\bar{\mathbf{x}}_k\|^2\\&\geq J(\bar{\mathbf{x}}_k)+\eta\langle\nabla J(\bar{\mathbf{x}}_k),\bar{\mathbf{u}}_k\rangle-\frac{\eta^2L}{2}\|\bar{\mathbf{u}}_k\|^2.
\end{split}
\end{equation}
According to the basic inequality $\langle\mathbf{a},\mathbf{b}\rangle=\frac{1}{2}(\|\mathbf{a}\|^2+\|\mathbf{b}\|^2-\|\mathbf{a}-\mathbf{b}\|^2)$, the last inequality, the following can be obtained
\begin{equation}\label{eq38}
\begin{split}
    J(\bar{\mathbf{x}}_{k+1})&\geq J(\bar{\mathbf{x}}_k)+\frac{\eta}{2}\|\nabla J(\bar{\mathbf{x}}_k)\|^2+\bigg(\frac{\eta}{2}-\frac{\eta^2L}{2}\bigg)\|\bar{\mathbf{u}}_k\|^2-\frac{\eta}{2}\|\bar{\mathbf{u}}_k-\nabla J(\bar{\mathbf{x}}_k)\|^2\\&\geq J(\bar{\mathbf{x}}_k)+\frac{\eta}{2}\|\nabla J(\bar{\mathbf{x}}_k)\|^2+\bigg(\frac{\eta}{2}-\frac{\eta^2L}{2}\bigg)\|\bar{\mathbf{u}}_k\|^2-\frac{\eta}{2}\|\bar{\mathbf{u}}_k-\nabla J(\bar{\mathbf{x}}_k)+\overline{\nabla\mathbf{J}}(\mathbf{x}_k)-\overline{\nabla\mathbf{J}}(\mathbf{x}_k)\|^2\\&\geq J(\bar{\mathbf{x}}_k)+\frac{\eta}{2}\|\nabla J(\bar{\mathbf{x}}_k)\|^2+\frac{\eta}{4}\|\bar{\mathbf{u}}_k\|^2-\eta\|\bar{\mathbf{u}}_k-\overline{\nabla\mathbf{J}}(\mathbf{x}_k)\|^2-\frac{\eta L^2}{N}\|\mathbf{x}_k-\Lambda\mathbf{x}_k\|^2,
\end{split}
\end{equation}
where the third inequality is due to $\|\mathbf{a}+\mathbf{b}\|^2\leq 2(\|\mathbf{a}\|^2+\|\mathbf{b}\|^2)$, the smoothness of $J$ and $\eta\in(0,\frac{1}{2L}]$. Rewriting Eq.~\ref{eq38} yields the following relationship
\begin{equation}
    \|\nabla J(\bar{\mathbf{x}}_k)\|^2\leq\frac{2(J(\bar{\mathbf{x}}_{k+1})-J(\bar{\mathbf{x}}_k))}{\eta}-\frac{1}{2}\|\bar{\mathbf{u}}_k\|^2+2\|\bar{\mathbf{u}}_k-\overline{\nabla\mathbf{J}}(\mathbf{x}_k)\|^2+\frac{2 L^2}{N}\|\mathbf{x}_k-\Lambda\mathbf{x}_k\|^2.
\end{equation}
Summing up the last equation from $0$ to $K$ completes the proof.
\end{proof}
According to Lemma~\ref{lemma_5}, it is clearly observed that in order to arrive at the main results presented in the previous section, the upper bound of the right hand side in the last inequality needs to be explicitly quantified. One may argue that in the main theorems it is $\tilde{\mathbf{x}}_K$ instead of $\bar{\mathbf{x}}_K$. Dividing Eq.~\ref{descent} by $\frac{1}{K+1}$ allows us to approximate $\tilde{\mathbf{x}}_K$ using the ensemble average. Additionally, a sufficiently large $K$ implies all agents have converged to the (near-) optimal solution, suggesting that the gap between $\tilde{\mathbf{x}}_K$ and $\bar{\mathbf{x}}_K$ is quite close to 0. Though in the final result, the convergence rate is specifically for $\tilde{\mathbf{x}}_K$, Lemma~\ref{lemma_5} helps facilitate the understanding of convergence for \texttt{MDPGT} and will be critical to derive the final error bounds in the main theorems. Thus, in the sequel, the analysis is dedicated to finding the upper bounds for $\sum_{k=0}^K\|\bar{\mathbf{u}}_k-\overline{\nabla\mathbf{J}}(\mathbf{x}_k)\|^2$ and $\sum_{k=0}^K\|\mathbf{x}_k-\Lambda\mathbf{x}_k\|^2$. To start the analysis, we present two relationships that will help significantly characterize the recursions of gradient variances, consensus, and gradient tracking errors.
\begin{lemma}\label{lemma_6}
Let $\{e_k\}_{k\geq 0}, \{b_k\}_{k\geq 0}$, and $\{c_k\}_{k\geq 0}$ be nonnegative sequences such that $e_{k}\leq qe_{k-1}+qb_{k-1}+c_k+C, \forall k\geq 1$, where $q\in(0,1)$ and $C\geq 0$. Then, we have for all $K\geq 1$,
\begin{equation}\label{eq.23}
    \sum_{k=1}^Ke_k\leq \frac{e_0}{1-q}+\frac{1}{1-q}\sum^{K-1}_{k=0}b_k+\frac{1}{1-q}\sum^{K}_{k=0}c_k+\frac{CK}{1-q}.
\end{equation}
Similarly, if $e_{k+1}\leq qe_{k}+b_{k-1}+C, \forall k\geq 1$, then, we have for all $K\geq 2$,
\begin{equation}\label{eq.24}
    \sum_{k=1}^Ke_k\leq \frac{e_1}{1-q}+\frac{1}{1-q}\sum^{K-2}_{k=0}b_k+\frac{CK}{1-q}.
\end{equation}
\end{lemma}
Please follow the proof of Lemma 6 in~\cite{xin2021hybrid} for more detail. We next derive the relevant recursions for the gradient variances $\mathbb{E}[\|\bar{\mathbf{u}}_k-\overline{\nabla\mathbf{J}}(\mathbf{x}_k)\|^2]$ and $\mathbb{E}[\|\mathbf{u}_k-\nabla\mathbf{J}(\mathbf{x}_k)\|^2]$. Essentially the former can be treated intuitively as the average of the latter up to a factor of $\frac{1}{N^2}$, but they are proved separately.
\begin{lemma}\label{lemma_7}
Let $\bar{\mathbf{u}}_k$ and $\mathbf{x}_k$ be generated by \texttt{MDPGT}. Then for all $k\geq 1$, we have
\begin{equation}\label{eq.25}
\begin{split}
    &\mathbb{E}[\|\bar{\mathbf{u}}_k-\overline{\nabla\mathbf{J}}(\mathbf{x}_k)\|^2]\leq (1-\beta)^2\mathbb{E}[\|\bar{\mathbf{u}}_{k-1}-\overline{\nabla\mathbf{J}}(\mathbf{x}_{k-1})\|^2]\\&+\frac{12(L^2+G^2C_\upsilon)\eta^2(1-\beta)^2}{N}\mathbb{E}[\|\bar{\mathbf{u}}_{k-1}\|^2]+\frac{2\beta^2\bar{\sigma}^2}{N}\\&+\frac{12(L^2+G^2C_\upsilon)(1-\beta)^2}{N^2}\mathbb{E}[\|\mathbf{x}_k-\Lambda\mathbf{x}_k\|^2\\&+\|\mathbf{x}_{k-1}-\Lambda\mathbf{x}_{k-1}\|^2],
\end{split}
\end{equation}
and
\begin{equation}\label{eq.26}
\begin{split}
    &\mathbb{E}[\|\mathbf{u}_k-\nabla\mathbf{J}(\mathbf{x}_k)\|^2]\leq (1-\beta)^2\mathbb{E}[\|\mathbf{u}_{k-1}-\nabla\mathbf{J}(\mathbf{x}_{k-1})\|^2]\\&+12N(L^2+G^2C_\upsilon)\eta^2(1-\beta)^2\mathbb{E}[\|\bar{\mathbf{u}}_{k-1}\|^2]+2\beta^2\bar{\sigma}^2N\\&+12(L^2+G^2C_\upsilon)(1-\beta)^2\mathbb{E}[\|\mathbf{x}_k-\Lambda\mathbf{x}_k\|^2\\&+\|\mathbf{x}_{k-1}-\Lambda\mathbf{x}_{k-1}\|^2].
\end{split}
\end{equation}
\end{lemma}
\begin{proof}
Recalling
\begin{equation}
    \mathbf{u}^i_k = \beta\mathbf{g}_i(\tau^i_k|\mathbf{x}^i_k)+(1-\beta)[\mathbf{u}_{k-1}^i+\mathbf{g}_i(\tau^i_k|\mathbf{x}^i_k)-\upsilon_i(\tau^i_k|\mathbf{x}^i_{k-1},\mathbf{x}^i_k)\mathbf{g}_i(\tau^i_k|\mathbf{x}^i_{k-1})],
\end{equation}
we then have
\begin{equation}\label{eq.41}
\begin{split}
    \mathbf{u}^i_k-\nabla J_i(\mathbf{x}^i_k)&=\beta\mathbf{g}_i(\tau^i_k|\mathbf{x}^i_k)+(1-\beta)[\mathbf{u}_{k-1}^i+\mathbf{g}_i(\tau^i_k|\mathbf{x}^i_k)-\upsilon_i(\tau^i_k|\mathbf{x}^i_{k-1},\mathbf{x}^i_k)\mathbf{g}_i(\tau^i_k|\mathbf{x}^i_{k-1})]\\&-\beta \nabla J_i(\mathbf{x}^i_k)-(1-\beta)\nabla J_i(\mathbf{x}^i_k)\\&=\beta(\mathbf{g}_i(\tau^i_k|\mathbf{x}^i_k)-\nabla J_i(\mathbf{x}^i_k))+(1-\beta)[\nabla J_i(\mathbf{x}^i_{k-1})+\mathbf{g}_i(\tau^i_k|\mathbf{x}^i_k)\\&-\upsilon_i(\tau^i_k|\mathbf{x}^i_{k-1},\mathbf{x}^i_k)\mathbf{g}_i(\tau^i_k|\mathbf{x}^i_{k-1})-\nabla J_i(\mathbf{x}^i_k)]+(1-\beta)(\mathbf{u}^i_k-\nabla J_i(\mathbf{x}^i_{k-1})).
\end{split}
\end{equation}
As $\mathbb{E}[\mathbf{g}_i(\tau^i_k|\mathbf{x}^i_k)-\nabla J_i(\mathbf{x}^i_k)]=\mathbf{0}_{d/N}$ and $\mathbb{E}[\mathbf{g}_i(\tau^i_k|\mathbf{x}^i_k)-\upsilon_i(\tau^i_k|\mathbf{x}^i_{k-1},\mathbf{x}^i_k)\mathbf{g}_i(\tau^i_k|\mathbf{x}^i_{k-1})+\nabla J_i(\mathbf{x}^i_{k-1})-\nabla J_i(\mathbf{x}^i_k)]=\mathbf{0}_{d/N}$. Applying Eq.~\ref{eq.41} from 1 to $N$ and taking the ensemble average results in 
\begin{equation}
\begin{split}
    \bar{\mathbf{u}}_k-\overline{\nabla \mathbf{J}}(\mathbf{x}_k)&=(1-\beta)(\bar{\mathbf{u}}_{k-1}-\overline{\nabla \mathbf{J}}(\mathbf{x}_{k-1}))+\beta\frac{1}{N}\sum^N_{i=1}(\mathbf{g}_i(\tau^i_k|\mathbf{x}^i_k)-\nabla J_i(\mathbf{x}^i_k))\\&+(1-\beta)\frac{1}{N}\sum^N_{i=1}(\mathbf{g}_i(\tau^i_k|\mathbf{x}^i_k)-\upsilon_i(\tau^i_k|\mathbf{x}^i_{k-1},\mathbf{x}^i_k)\mathbf{g}_i(\tau^i_k|\mathbf{x}^i_{k-1})+\nabla J_i(\mathbf{x}^i_{k-1})-\nabla J_i(\mathbf{x}^i_k))
\end{split}
\end{equation}
Let \[\mathbf{n}_k=\frac{1}{N}\sum^N_{i=1}(\mathbf{g}_i(\tau^i_k|\mathbf{x}^i_k)-\nabla J_i(\mathbf{x}^i_k))\] and \[\mathbf{z}_k=\frac{1}{N}\sum^N_{i=1}(\mathbf{g}_i(\tau^i_k|\mathbf{x}^i_k)-\upsilon_i(\tau^i_k|\mathbf{x}^i_{k-1},\mathbf{x}^i_k)\mathbf{g}_i(\tau^i_k|\mathbf{x}^i_{k-1})+\nabla J_i(\mathbf{x}^i_{k-1})-\nabla J_i(\mathbf{x}^i_k)).\] We have
\begin{equation}\label{eq.43}
\begin{split}
    \mathbb{E}[\|\bar{\mathbf{u}}_k-\overline{\nabla \mathbf{J}}(\mathbf{x}_k)\|^2] &= (1-\beta)^2\|\bar{\mathbf{u}}_{k-1}-\overline{\nabla \mathbf{J}}(\mathbf{x}_{k-1})\|^2+\mathbb{E}[\|\beta\mathbf{n}_k+(1-\beta)\mathbf{z}_k\|^2]\\&+2\mathbb{E}[\langle(1-\beta)(\bar{\mathbf{u}}_{k-1}-\overline{\nabla \mathbf{J}}(\mathbf{x}_{k-1})),\beta\mathbf{n}_k+(1-\beta)\mathbf{z}_k\rangle]\\&=(1-\beta)^2\|\bar{\mathbf{u}}_{k-1}-\overline{\nabla \mathbf{J}}(\mathbf{x}_{k-1})\|^2+\mathbb{E}[\|\beta\mathbf{n}_k+(1-\beta)\mathbf{z}_k\|^2]\\&\leq (1-\beta)^2\|\bar{\mathbf{u}}_{k-1}-\overline{\nabla \mathbf{J}}(\mathbf{x}_{k-1})\|^2+2\beta^2\mathbb{E}[\|\mathbf{n}_k\|^2]+2(1-\beta)^2\mathbb{E}[\|\mathbf{z}_k\|^2].
\end{split}
\end{equation}
The second equality follows from that $\mathbb{E}[\langle(1-\beta)(\bar{\mathbf{u}}_{k-1}-\overline{\nabla \mathbf{J}}(\mathbf{x}_{k-1})),\beta\mathbf{n}_k+(1-\beta)\mathbf{z}_k\rangle]=0$ as $\mathbb{E}[\mathbf{n}_k]=\mathbf{0}_{d/N}$ and $\mathbb{E}[\mathbf{z}_k]=\mathbf{0}_{d/N}$. We next bound the terms $\mathbb{E}[\|\mathbf{n}_k\|^2]$ and $\mathbb{E}[\|\mathbf{z}_k\|^2]$.
As for $k\geq 1$,
\begin{equation}
\begin{split}
    \mathbb{E}[\|\mathbf{n}_k\|^2]&=\frac{1}{N^2}\sum^N_{i=1}\mathbb{E}[\|\mathbf{g}_i(\tau^i_k|\mathbf{x}^i_k)-\nabla J_i(\mathbf{x}^i_k)\|^2]+\frac{1}{N^2}\sum_{i\neq j}\mathbb{E}[\langle\mathbf{g}_i(\tau^i_k|\mathbf{x}^i_k)-\nabla J_i(\mathbf{x}^i_k), \mathbf{g}_j(\tau^j_k|\mathbf{x}^j_k)-\nabla J_j(\mathbf{x}^j_k)\rangle]\\&=\frac{1}{N^2}\sum^N_{i=1}\mathbb{E}[\|\mathbf{g}_i(\tau^i_k|\mathbf{x}^i_k)-\nabla J_i(\mathbf{x}^i_k)\|^2]\leq\frac{\bar{\sigma}^2}{N}.
\end{split}
\end{equation}
The second equality follows from that
\begin{equation}
\begin{split}
&\mathbb{E}[\langle\mathbf{g}_i(\tau^i_k|\mathbf{x}^i_k)-\nabla J_i(\mathbf{x}^i_k), \mathbf{g}_j(\tau^j_k|\mathbf{x}^j_k)-\nabla J_j(\mathbf{x}^j_k)\rangle]\\&=\mathbb{E}[\langle\mathbb{E}[\mathbf{g}_i(\tau^i_k|\mathbf{x}^i_k)]-\nabla J_i(\mathbf{x}^i_k),\mathbf{g}_j(\tau^j_k|\mathbf{x}^j_k)-\nabla J_j(\mathbf{x}^j_k)\rangle]\\&=0.
\end{split}
\end{equation}
For the term $\mathbb{E}[\|\mathbf{z}_k\|^2]$, we have for all $k\geq 1$,
\begin{equation}
\begin{split}
    \mathbb{E}[\|\mathbf{z}_k\|^2]&=\mathbb{E}[\|\frac{1}{N}\sum^N_{i=1}(\mathbf{g}_i(\tau^i_k|\mathbf{x}^i_k)-\upsilon_i(\tau^i_k|\mathbf{x}^i_{k-1},\mathbf{x}^i_k)\mathbf{g}_i(\tau^i_k|\mathbf{x}^i_{k-1})+\nabla J_i(\mathbf{x}^i_{k-1})-\nabla J_i(\mathbf{x}^i_k))\|^2]\\&=\frac{1}{N^2}\sum^N_{i=1}\mathbb{E}[\|\mathbf{g}_i(\tau^i_k|\mathbf{x}^i_k)-\upsilon_i(\tau^i_k|\mathbf{x}^i_{k-1},\mathbf{x}^i_k)\mathbf{g}_i(\tau^i_k|\mathbf{x}^i_{k-1})+\nabla J_i(\mathbf{x}^i_{k-1})-\nabla J_i(\mathbf{x}^i_k)\|^2]\\&+\frac{1}{N^2}\sum_{i\neq j}\mathbb{E}[\langle\mathbf{g}_i(\tau^i_k|\mathbf{x}^i_k)-\upsilon_i(\tau^i_k|\mathbf{x}^i_{k-1},\mathbf{x}^i_k)\mathbf{g}_i(\tau^i_k|\mathbf{x}^i_{k-1})+\nabla J_i(\mathbf{x}^i_{k-1})-\nabla J_i(\mathbf{x}^i_k),\\&\mathbf{g}_j(\tau^j_k|\mathbf{x}^j_k)-\upsilon_j(\tau^j_k|\mathbf{x}^j_{k-1},\mathbf{x}^j_k)\mathbf{g}_j(\tau^j_k|\mathbf{x}^j_{k-1})+\nabla J_j(\mathbf{x}^j_{k-1})-\nabla J_j(\mathbf{x}^j_k)\rangle]\\&=\frac{1}{N^2}\sum^N_{i=1}\mathbb{E}[\|\mathbf{g}_i(\tau^i_k|\mathbf{x}^i_k)-\upsilon_i(\tau^i_k|\mathbf{x}^i_{k-1},\mathbf{x}^i_k)\mathbf{g}_i(\tau^i_k|\mathbf{x}^i_{k-1})+\nabla J_i(\mathbf{x}^i_{k-1})-\nabla J_i(\mathbf{x}^i_k)\|^2]\\&\leq \frac{1}{N^2}\sum^N_{i=1}\mathbb{E}[\|\mathbf{g}_i(\tau^i_k|\mathbf{x}^i_k)-\upsilon_i(\tau^i_k|\mathbf{x}^i_{k-1},\mathbf{x}^i_k)\mathbf{g}_i(\tau^i_k|\mathbf{x}^i_{k-1})\|^2].
\end{split}
\end{equation}
The third equality follows from the same argument we have for the $\mathbb{E}[\langle\mathbf{g}_i(\tau^i_k|\mathbf{x}^i_k)-\nabla J_i(\mathbf{x}^i_k), \mathbf{g}_j(\tau^j_k|\mathbf{x}^j_k)-\nabla J_j(\mathbf{x}^j_k)\rangle]$. The last inequality follows from the fact that\[\mathbb{E}[\mathbf{g}_i(\tau^i_k|\mathbf{x}^i_k)-\upsilon_i(\tau^i_k|\mathbf{x}^i_{k-1},\mathbf{x}^i_k)\mathbf{g}_i(\tau^i_k|\mathbf{x}^i_{k-1})]=\nabla J_i(\mathbf{x}^i_{k})-\nabla J_i(\mathbf{x}^i_{k-1})\] and the variance decomposition, i.e., for any vector $\mathbf{a}\in\mathbb{R}^p$, \[\mathbb{E}[\|\mathbf{a}-\mathbb{E}[\mathbf{a}]\|^2]=\mathbb{E}[\|\mathbf{a}\|^2]-\|\mathbb{E}[\mathbf{a}]\|^2.\] Hence, we have now
\begin{equation}
    \begin{split}
        \mathbb{E}[\|\mathbf{z}_k\|^2]&\leq\frac{1}{N^2}\sum^N_{i=1}\mathbb{E}[\|\mathbf{g}_i(\tau^i_k|\mathbf{x}^i_k)-\upsilon_i(\tau^i_k|\mathbf{x}^i_{k-1},\mathbf{x}^i_k)\mathbf{g}_i(\tau^i_k|\mathbf{x}^i_{k-1})\|^2]\\&=\frac{1}{N^2}\sum^N_{i=1}\mathbb{E}[\|\mathbf{g}_i(\tau^i_k|\mathbf{x}^i_k)-\mathbf{g}_i(\tau^i_k|\mathbf{x}^i_{k-1})+\mathbf{g}_i(\tau^i_k|\mathbf{x}^i_{k-1})-\upsilon_i(\tau^i_k|\mathbf{x}^i_{k-1},\mathbf{x}^i_k)\mathbf{g}_i(\tau^i_k|\mathbf{x}^i_{k-1})\|^2]\\&=\frac{1}{N^2}\sum^N_{i=1}\mathbb{E}[2\|\mathbf{g}_i(\tau^i_k|\mathbf{x}^i_k)-\mathbf{g}_i(\tau^i_k|\mathbf{x}^i_{k-1})\|^2+2\|\mathbf{g}_i(\tau^i_k|\mathbf{x}^i_{k-1})-\upsilon_i(\tau^i_k|\mathbf{x}^i_{k-1},\mathbf{x}^i_k)\mathbf{g}_i(\tau^i_k|\mathbf{x}^i_{k-1})\|^2]\\&\leq \frac{2L^2}{N^2}\sum_{i=1}^N\mathbb{E}[\|\mathbf{x}^i_k-\mathbf{x}^i_{k-1}\|^2]+\frac{2}{N^2}\sum^N_{i=1}\mathbb{E}[\|1-\upsilon_i(\tau^i_k|\mathbf{x}^i_{k-1},\mathbf{x}^i_k)\|^2\|\mathbf{g}_i(\tau^i_k|\mathbf{x}^i_{k-1})\|^2]\\&\leq \frac{2L^2}{N^2}\sum_{i=1}^N\mathbb{E}[\|\mathbf{x}^i_k-\mathbf{x}^i_{k-1}\|^2]+\frac{2G^2}{N^2}\mathbb{E}[\|1-\upsilon_i(\tau^i_k|\mathbf{x}^i_{k-1},\mathbf{x}^i_k)\|^2]\\&=\frac{2L^2}{N^2}\sum_{i=1}^N\mathbb{E}[\|\mathbf{x}^i_k-\mathbf{x}^i_{k-1}\|^2]+\frac{2G^2}{N^2}\sum^N_{i=1}\mathbb{V}(\upsilon_i(\tau^i_k|\mathbf{x}^i_{k-1},\mathbf{x}^i_k))\\&\leq \frac{2L^2}{N^2}\sum_{i=1}^N\mathbb{E}[\|\mathbf{x}^i_k-\mathbf{x}^i_{k-1}\|^2]+\frac{2G^2C_\upsilon}{N^2}\sum^N_{i=1}\mathbb{E}[\|\mathbf{x}^i_k-\mathbf{x}^i_{k-1}\|^2]\\&=\frac{2L^2+2G^2C_\upsilon}{N^2}\mathbb{E}[\|\mathbf{x}_k-\mathbf{x}_{k-1}\|^2]\\&=\frac{2L^2+2G^2C_\upsilon}{N^2}\mathbb{E}[\|\mathbf{x}_k-\Lambda\mathbf{x}_k+\Lambda\mathbf{x}_k-\Lambda\mathbf{x}_{k-1}+\Lambda\mathbf{x}_{k-1}-\mathbf{x}_{k-1}\|^2]\\&\leq\frac{6L^2+6G^2C_\upsilon}{N}\mathbb{E}[\|\bar{\mathbf{u}}_{k-1}\|^2]+\frac{6L^2+6G^2C_\upsilon}{N^2}\mathbb{E}[\|\mathbf{x}_k-\Lambda\mathbf{x}_k\|^2+\|\Lambda\mathbf{x}_{k-1}-\mathbf{x}_{k-1}\|^2].
    \end{split}
\end{equation}
The second inequality follows from the Cauthy-Schwartz inequality. The third inequality follows from Lemma~\ref{lemma_1} while the fourth inequality follows from Lemma~\ref{lemma_3}. Hence, the following relationship can be obtained
\begin{equation}
    \begin{split}
        \mathbb{E}[\|\bar{\mathbf{u}}_k-\overline{\nabla \mathbf{f}}(\mathbf{x}_k)\|^2]&\leq (1-\beta)^2\mathbb{E}[\|\bar{\mathbf{u}}_{k-1}-\overline{\nabla \mathbf{f}}(\mathbf{x}_{k-1})\|^2]+\frac{12L^2+12G^2C_\upsilon}{N}\eta^2(1-\beta)^2\mathbb{E}[\|\bar{\mathbf{u}}_{k-1}\|^2]\\&+\frac{2\beta^2\bar{\sigma}^2}{N}+\frac{12L^2+12G^2C_\upsilon}{N^2}(1-\beta)^2\mathbb{E}[\|\mathbf{x}_k-\Lambda\mathbf{x}_k\|^2+\|\Lambda\mathbf{x}_{k-1}-\mathbf{x}_{k-1}\|^2],
    \end{split}
\end{equation}
which completes the first part of proof for Lemma~\ref{lemma_7}.

Recalling Eq.~\ref{eq.41}, we have
\begin{equation}
    \begin{split}
        \mathbf{u}^i_k-\nabla J_i(\mathbf{x}^i_k)&=\beta(\mathbf{g}_i(\tau^i_k|\mathbf{x}^i_k)-\nabla J_i(\mathbf{x}^i_k))+(1-\beta)[\nabla J_i(\mathbf{x}^i_{k-1})+\mathbf{g}_i(\tau^i_k|\mathbf{x}^i_k)\\&-\upsilon_i(\tau^i_k|\mathbf{x}^i_{k-1},\mathbf{x}^i_k)\mathbf{g}_i(\tau^i_k|\mathbf{x}^i_{k-1})-\nabla J_i(\mathbf{x}^i_k)]+(1-\beta)(\mathbf{u}^i_k-\nabla J_i(\mathbf{x}^i_{k-1})).
    \end{split}
\end{equation}
Based on the analysis above, it is immediately obtained that the expectations of the first and second terms are all zero. Following the similar proof in Eq.~\ref{eq.43}, we have for all $k\geq 1$,
\begin{equation}
    \begin{split}
        &\mathbb{E}[\|\mathbf{u}^i_k-\nabla J_i(\mathbf{x}^i_k)\|^2]\leq (1-\beta)^2\|\mathbf{u}^i_{k-1}-\nabla J_i(\mathbf{x}^i_{k-1})\|^2+2\beta^2\mathbb{E}[\|\mathbf{g}_i(\tau^i_k|\mathbf{x}^i_k)-\nabla J_i(\mathbf{x}^i_k)\|^2]\\&+2(1-\beta)^2\mathbb{E}[\|\nabla J_i(\mathbf{x}^i_{k-1})+\mathbf{g}_i(\tau^i_k|\mathbf{x}^i_k)-\upsilon_i(\tau^i_k|\mathbf{x}^i_{k-1},\mathbf{x}^i_k)\mathbf{g}_i(\tau^i_k|\mathbf{x}^i_{k-1})-\nabla J_i(\mathbf{x}^i_k)\|^2]\\&\leq(1-\beta)^2\|\mathbf{u}^i_{k-1}-\nabla J_i(\mathbf{x}^i_{k-1})\|^2+2\beta^2\mathbb{E}[\|\mathbf{g}_i(\tau^i_k|\mathbf{x}^i_k)-\nabla J_i(\mathbf{x}^i_k)\|^2]\\&+2(1-\beta)^2\mathbb{E}[\|\mathbf{g}_i(\tau^i_k|\mathbf{x}^i_k)-\upsilon_i(\tau^i_k|\mathbf{x}^i_{k-1},\mathbf{x}^i_k)\mathbf{g}_i(\tau^i_k|\mathbf{x}^i_{k-1})\|^2]
    \end{split}
\end{equation}
According to Lemma~\ref{lemma_1} and Lemma~\ref{lemma_3}, the following relationship can be attained
\begin{equation}\label{eq.51}
\begin{split}
    \mathbb{E}[\|\mathbf{u}^i_k-\nabla J_i(\mathbf{x}^i_k)\|^2]&\leq(1-\beta)^2\|\mathbf{u}^i_{k-1}-\nabla J_i(\mathbf{x}^i_{k-1})\|^2+2\beta^2\sigma^2_i+4(1-\beta)^2(L^2+C_\upsilon G^2)\mathbb{E}[\|\mathbf{x}^i_k-\mathbf{x}^i_{k-1}\|^2]\\&\leq(1-\beta)^2\|\mathbf{u}^i_{k-1}-\nabla J_i(\mathbf{x}^i_{k-1})\|^2+2\beta^2\sigma^2_i\\&+12(1-\beta)^2(L^2+C_\upsilon G^2)(\mathbb{E}[\|\mathbf{x}^i_k-\bar{\mathbf{x}}_k\|^2]+\|\bar{\mathbf{x}}_k-\bar{\mathbf{x}}_{k-1}\|^2+\|\bar{\mathbf{x}}_{k-1}-\mathbf{x}^i_{k-1}\|^2)\\&= (1-\beta)^2\|\mathbf{u}^i_{k-1}-\nabla J_i(\mathbf{x}^i_{k-1})\|^2+2\beta^2\sigma^2_i+12(1-\beta)^2(L^2+C_\upsilon G^2)\eta^2\mathbb{E}[\|\bar{\mathbf{u}}_{k-1}\|^2]\\&+12(1-\beta)^2(L^2+C_\upsilon G^2)(\mathbb{E}[\|\mathbf{x}^i_k-\bar{\mathbf{x}}_k\|^2+\|\mathbf{x}^i_{k-1}-\bar{\mathbf{x}}_{k-1}\|^2]).
\end{split}
\end{equation}
Applying Eq.~\ref{eq.51} over $i$ from 1 to $N$ completes the second part of the proof for Lemma~\ref{lemma_7}.
\end{proof}
With Lemma~\ref{lemma_6} and Lemma~\ref{lemma_7} in hand, we now are ready to present the upper bounds for both $\sum_{k=1}^K\mathbb{E}[\|\bar{\mathbf{u}}_k-\overline{\nabla\mathbf{J}}(\mathbf{x}_k)\|^2]$ and $\sum_{k=1}^K\mathbb{E}[\|\mathbf{u}_k-\nabla\mathbf{J}(\mathbf{x}_k)\|^2]$ in the following lemma.
\begin{lemma}\label{lemma_8}
Let $\bar{\mathbf{u}}_k$ and $\mathbf{x}_k$ be generated by \texttt{MDPGT} initialized with a mini-batch of trajectories $\mathcal{B}$. Then for any $\beta\in(0,1)$, $\forall K\geq 1$, we have
\begin{equation}\label{ave_gradient_var}
    \begin{split}
        &\sum_{k=0}^K\mathbb{E}[\|\bar{\mathbf{u}}_k-\overline{\nabla\mathbf{J}}(\mathbf{x}_k)\|^2]\leq\frac{\bar{\sigma}^2}{|\mathcal{B}|N\beta}+\frac{12(L^2+G^2C_\upsilon)\eta^2}{N\beta}\times\\&\sum_{k=0}^{K-1}\mathbb{E}[\|\bar{\mathbf{u}}_k\|^2]+\frac{24(L^2+G^2C_\upsilon)}{N^2\beta}\sum_{k=0}^K\mathbb{E}[\|\mathbf{x}_k-\Lambda\mathbf{x}_k\|^2]\\&+\frac{2\beta\bar{\sigma}^2K}{N},
    \end{split}
\end{equation}
and
\begin{equation}
    \begin{split}
        &\sum_{k=0}^K\mathbb{E}[\|\mathbf{u}_k-\nabla\mathbf{J}(\mathbf{x}_k)\|^2]\leq\frac{N\bar{\sigma}^2}{|\mathcal{B}|\beta}+\frac{12N(L^2+G^2C_\upsilon)\eta^2}{\beta}\times\\&\sum_{k=0}^{K-1}\mathbb{E}[\|\bar{\mathbf{u}}_k\|^2]+\frac{24(L^2+G^2C_\upsilon)}{\beta}\sum_{k=0}^K\mathbb{E}[\|\mathbf{x}_k-\Lambda\mathbf{x}_k\|^2]\\&+2N\beta\bar{\sigma}^2K.
    \end{split}
\end{equation}
\end{lemma}
\begin{proof}
We apply the conclusions from Lemma~\ref{lemma_6} to Lemma~\ref{lemma_7}. We first show the upper error bound for $\sum_{k=0}^K\mathbb{E}[\|\bar{\mathbf{u}}_k-\overline{\nabla \mathbf{J}}(\mathbf{x}_k)\|^2]$.
Substituting Eq.~\ref{eq.25} into Eq.~\ref{eq.23} leads to the following inequality
\begin{equation}
    \begin{split}
        &\sum_{k=0}^K\mathbb{E}[\|\bar{\mathbf{u}}_k-\overline{\nabla \mathbf{J}}(\mathbf{x}_k)\|^2]\leq \frac{1}{1-(1-\beta)^2}\mathbb{E}[\|\bar{\mathbf{u}}_0-\overline{\nabla \mathbf{J}}(\mathbf{x}_0)\|^2]+\frac{12(L^2+G^2C_\upsilon)\eta^2(1-\beta)^2}{N(1-(1-\beta)^2)}\sum^{K-1}_{k=0}\mathbb{E}[\|\bar{\mathbf{u}}_{k-1}\|^2]\\&+\frac{2\beta^2\bar{\sigma}^2K}{N(1-(1-\beta)^2)}+\frac{12(L^2+G^2C_\upsilon)(1-\beta)^2}{N^2(1-(1-\beta)^2)}\mathbb{E}\sum^{K-1}_{k=0}[\|\mathbf{x}_k-\Lambda\mathbf{x}_k\|^2+\|\mathbf{x}_{k-1}-\Lambda\mathbf{x}_{k-1}\|^2]\\&\leq\frac{1}{1-(1-\beta)^2}\mathbb{E}[\|\bar{\mathbf{u}}_0-\overline{\nabla \mathbf{J}}(\mathbf{x}_0)\|^2]+\frac{12(L^2+G^2C_\upsilon)\eta^2(1-\beta)^2}{N(1-(1-\beta)^2)}\sum^{K-1}_{k=0}\mathbb{E}[\|\bar{\mathbf{u}}_{k-1}\|^2]\\&+\frac{2\beta^2\bar{\sigma}^2K}{N(1-(1-\beta)^2)}+\frac{24(L^2+G^2C_\upsilon)(1-\beta)^2}{N^2(1-(1-\beta)^2)}\sum^{K}_{k=0}\mathbb{E}[\|\mathbf{x}_k-\Lambda\mathbf{x}_k\|^2].
    \end{split}
\end{equation}
We now process the first term on the right hand side of the last inequality. 
\[\mathbb{E}[\|\bar{\mathbf{u}}_0-\overline{\nabla \mathbf{J}}(\mathbf{x}_0)\|^2]=\mathbb{E}[\|\frac{1}{N}\sum_{i=1}^N\frac{1}{|\mathcal{B}|}\sum^{|\mathcal{B}|}_{m=1}(\mathbf{g}_i(\tau^{i,m}_{0}|\mathbf{x}_0^i)-\nabla J_i(\mathbf{x}_0^i))\|^2].\] It is immediately obtained that
\[\mathbb{E}[\|\bar{\mathbf{u}}_0-\overline{\nabla \mathbf{J}}(\mathbf{x}_0)\|^2]=\frac{1}{N^2|\mathcal{B}|^2}\sum^N_{i=1}\sum^{|\mathcal{B}|}_{m=1}\mathbb{E}[\|\mathbf{g}_i(\tau^{i,m}_{0}|\mathbf{x}_0^i)-\nabla J_i(\mathbf{x}_0^i)\|^2]\leq\frac{\bar{\sigma}^2}{N|\mathcal{B}|}.\]Observing that $\frac{1}{1-(1-\beta)^2}\leq\frac{1}{\beta}$, we can obtain the first conclusion in Lemma~\ref{lemma_8}. Likewise, if we apply Eq.~\ref{eq.23} to Eq.~\ref{eq.26}, the following relationship can be attained
\begin{equation}\label{eq.53}
    \begin{split}
        \sum^K_{k=1}\mathbb{E}[\|\mathbf{u}_k-\nabla\mathbf{J}(\mathbf{x}_k)\|^2]&\leq\frac{\mathbb{E}[\|\mathbf{u}_0-\nabla\mathbf{J}(\mathbf{x}_0)\|^2]}{\beta}+\frac{12N(L^2+G^2C_\upsilon)\eta^2}{\beta}\sum^{K-1}_{k=0}\mathbb{E}[\|\bar{\mathbf{u}}_k\|^2]\\&+\frac{24(L^2+G^2C_\upsilon)}{\beta}\sum_{k=0}^K\mathbb{E}[\|\mathbf{x}_k-\Lambda\mathbf{x}_k\|^2]+2N\beta\bar{\sigma}^2K.
    \end{split}
\end{equation}
As\begin{equation}\label{eq.54}\begin{split}\mathbb{E}[\|\mathbf{u}_0-\nabla\mathbf{J}(\mathbf{x}_0)\|^2]&=\sum_{i=1}^N\mathbb{E}[\|\frac{1}{|\mathcal{B}|}\sum^{|\mathcal{B}|}_{m=1}(\mathbf{g}_i(\tau^{i,m}_{0}|\mathbf{x}_0^i)-\nabla J_i(\mathbf{x}_0^i))\|^2]\\&=\frac{1}{|\mathcal{B}|^2}\sum^N_{i=1}\sum^{|\mathcal{B}|}_{m=1}\mathbb{E}[\|\mathbf{g}_i(\tau^{i,m}_{0}|\mathbf{x}_0^i)-\nabla J_i(\mathbf{x}_0^i)\|^2]\leq\frac{N\bar{\sigma}^2}{|\mathcal{B}|}.\end{split}\end{equation}Substituting Eq.~\ref{eq.54} into Eq.~\ref{eq.53} completes the proof for the second conclusion in Lemma~\ref{lemma_8}.
\end{proof}
It is clearly observed from Eq.~\ref{ave_gradient_var} that the error bound is dependent of the consensus error. To bound the consensus, we first present a lemma to bound the gradient tracking errors. As the initialization approach can be to use either a single trajectory or a mini-batch of trajectories, the lemma also consists of the explicit bound for the initial gradient tracking error.
Before that, we present a fact that establishes the relationship between $\|\mathbf{v}_{k+1}-\Lambda\mathbf{v}_{k+1}\|^2$ and $\|\mathbf{x}_{k+1}-\Lambda\mathbf{x}_{k+1}\|^2$.
\begin{fact}\label{fact_1}
For all $k\geq 0$, we have
\begin{equation}
\begin{split}
    \|\mathbf{x}_{k+1}-\Lambda\mathbf{x}_{k+1}\|^2&\leq\frac{1+\lambda^2}{2}\|\mathbf{x}_{k}-\Lambda\mathbf{x}_{k}\|^2\\&+\frac{2\eta^2\lambda^2}{1-\lambda^2}\|\mathbf{v}_{k+1}-\Lambda\mathbf{v}_{k+1}\|^2,
\end{split}
\end{equation}
and
\begin{equation}
\begin{split}
    \|\mathbf{x}_{k+1}-\Lambda\mathbf{x}_{k+1}\|^2&\leq2\lambda^2\|\mathbf{x}_{k}-\Lambda\mathbf{x}_{k}\|^2\\&+2\eta^2\lambda^2\|\mathbf{v}_{k+1}-\Lambda\mathbf{v}_{k+1}\|^2.
\end{split}
\end{equation}
\end{fact}
This fact can be easily shown true by using the update laws of \texttt{MDPGT} and the Young's inequality.
\begin{lemma}\label{lemma_9}
Let $\mathbf{v}_k$ be generated by \texttt{MDPGT} initialized with a mini-batch of trajectories $\mathcal{B}$. We have the following relationships:
\begin{equation}
    \mathbb{E}[\|\mathbf{v}_1-\Lambda\mathbf{v}_1\|^2]\leq \frac{N\bar{\sigma}^2}{|\mathcal{B}|}+\|\nabla \mathbf{J}(\mathbf{x}_0)\|^2,
\end{equation}
and if $0<\eta\leq\frac{1-\lambda^2}{8\lambda\sqrt{9L^2+3G^2C_\upsilon}}$, then for all $k\leq 1$,
\begin{equation}\label{eq.32}
\begin{split}
&\mathbb{E}[\|\mathbf{v}_{k+1}-\Lambda\mathbf{v}_{k+1}\|^2]\leq\frac{3+\lambda^2}{4}\mathbb{E}[\|\mathbf{v}_{k}-\Lambda\mathbf{v}_{k}\|^2]\\&+\frac{48L^2+24(L^2+G^2C_\upsilon)}{1-\lambda^2}N\eta^2\mathbb{E}[\|\bar{\mathbf{u}}_{k-1}\|^2]+4N\beta^2\bar{\sigma}^2\\&+\frac{114L^2+72(L^2+G^2C_\upsilon)}{1-\lambda^2}\mathbb{E}[\|\mathbf{x}_{k-1}-\Lambda\mathbf{x}_{k-1}\|^2]\\&+\frac{10\beta^2}{1-\lambda^2}\mathbb{E}[\|\mathbf{u}_{k-1}-\nabla\mathbf{J}(\mathbf{x}_{k-1})\|^2].
\end{split}
\end{equation}
\end{lemma}
\begin{proof}
According to the update law, we have
\begin{equation}
    \mathbf{v}_1 - \Lambda\mathbf{v}_1 = \underline{\mathbf{W}}\mathbf{v}_0+\mathbf{u}_0-\mathbf{u}_{-1} - \Lambda\underline{\mathbf{W}}\mathbf{v}_0-\Lambda\mathbf{u}_0+\Lambda\mathbf{u}_{-1}=(\mathbf{I}-\Lambda)\mathbf{u}_0.
\end{equation}
It follows from $\mathbf{v}_0=\mathbf{0}$ and $\mathbf{u}_{-1}=\mathbf{0}$. Hence, we have
\begin{equation}
    \mathbb{E}[\|\mathbf{v}_1 - \Lambda\mathbf{v}_1\|^2]=\mathbb{E}[\|(\mathbf{I}-\Lambda)\mathbf{u}_0\|^2]\leq\mathbb{E}[\|\mathbf{u}_0-\nabla\mathbf{J}(\mathbf{x}_0)+\nabla\mathbf{J}(\mathbf{x}_0)\|^2]
\end{equation}
Since $\mathbf{u}^i_0=\frac{1}{|\mathcal{B}|}\sum^{|\mathcal{B}|}_{m=1}\mathbf{g}_i(\tau^{i,m}_0|\mathbf{x}^i_0)$, we then obtain
\begin{equation}
\begin{split}
    &\mathbb{E}[\|\mathbf{u}_0-\nabla\mathbf{J}(\mathbf{x}_0)+\nabla\mathbf{J}(\mathbf{x}_0)\|^2]=\sum^N_{i=1}\mathbb{E}[\|\mathbf{u}^i_0-\nabla J_i(\mathbf{x}^i_0)\|^2]+\|\nabla\mathbf{J}(\mathbf{x}_0)\|^2\\&=\sum_{i=1}^N\mathbb{E}[\|\frac{1}{|\mathcal{B}|}\sum^{|\mathcal{B}|}_{m=1}(\mathbf{g}_i(\tau^{i,m}_{0}|\mathbf{x}_0^i)-\nabla J_i(\mathbf{x}_0^i))\|^2]+\|\nabla\mathbf{J}(\mathbf{x}_0)\|^2\\&=\frac{1}{|\mathcal{B}|^2}\sum^N_{i=1}\sum^{|\mathcal{B}|}_{m=1}\mathbb{E}[\|\mathbf{g}_i(\tau^{i,m}_{0}|\mathbf{x}_0^i)-\nabla J_i(\mathbf{x}_0^i)\|^2]+\|\nabla\mathbf{J}(\mathbf{x}_0)\|^2\\&\leq\frac{N\bar{\sigma}^2}{|\mathcal{B}|}+\|\nabla\mathbf{J}(\mathbf{x}_0)\|^2,
\end{split}
\end{equation}
which completes the first part of the proof.

We now proceed to prove the second part. As
\begin{equation}
    \begin{split}
        \mathbf{v}_{k+1}-\Lambda\mathbf{v}_{k+1}&=\underline{\mathbf{W}}\mathbf{v}_k+\mathbf{u}_k-\mathbf{u}_{k-1}-\Lambda(\underline{\mathbf{W}}\mathbf{v}_k+\mathbf{u}_k-\mathbf{u}_{k-1})\\&=(\underline{\mathbf{W}}-\Lambda)\mathbf{v}_k+(\mathbf{I}-\Lambda)(\mathbf{u}_k-\mathbf{u}_{k-1}),
    \end{split}
\end{equation}
conducting the squared norm results in the following relationship
\begin{equation}\label{eq.59}
    \begin{split}
        \|\mathbf{v}_{k+1}-\Lambda\mathbf{v}_{k+1}\|^2&=\|(\underline{\mathbf{W}}-\Lambda)\mathbf{v}_k+(\mathbf{I}-\Lambda)(\mathbf{u}_k-\mathbf{u}_{k-1})\|^2\\&=\|(\underline{\mathbf{W}}-\Lambda)\mathbf{v}_k\|^2+2\langle(\underline{\mathbf{W}}-\Lambda)\mathbf{v}_k,(\mathbf{I}-\Lambda)(\mathbf{u}_k-\mathbf{u}_{k-1})\rangle\\&+\|(\mathbf{I}-\Lambda)(\mathbf{u}_k-\mathbf{u}_{k-1})\|^2\\&\leq\lambda^2\|\mathbf{v}_k-\Lambda\mathbf{v}_k\|^2+2\langle(\underline{\mathbf{W}}-\Lambda)\mathbf{v}_k,(\mathbf{I}-\Lambda)(\mathbf{u}_k-\mathbf{u}_{k-1})\rangle+\|\mathbf{u}_k-\mathbf{u}_{k-1}\|^2.
    \end{split}
\end{equation}
We next investigate the second term on the right hand side of the last inequality. Based on the update for the policy gradient surrogate, we have \[\mathbf{u}^i_k=\mathbf{g}_i(\tau^i_k|\mathbf{x}_k^i)+(1-\beta)\mathbf{u}^i_{k-1}-(1-\beta)\upsilon_i(\tau^i_k|\mathbf{x}^i_{k-1},\mathbf{x}^i_k)\mathbf{g}_i(\mathbf{x}^i_{k-1}).\] Thus,
\begin{equation}\label{eq.60}
    \begin{split}
        \mathbf{u}^i_k-\mathbf{u}^i_{k-1}&=\mathbf{g}_i(\tau^i_k|\mathbf{x}_k^i)-\beta\mathbf{u}^i_{k-1}-(1-\beta)\upsilon_i(\tau^i_k|\mathbf{x}^i_{k-1},\mathbf{x}^i_k)\mathbf{g}_i(\mathbf{x}^i_{k-1})\\&=(1-\beta)(\mathbf{g}_i(\tau^i_k|\mathbf{x}_k^i)-\upsilon_i(\tau^i_k|\mathbf{x}^i_{k-1},\mathbf{x}^i_k)\mathbf{g}_i(\tau^i_k|\mathbf{x}_{k-1}^i))+\beta(\mathbf{g}_i(\tau^i_k|\mathbf{x}_k^i)-\mathbf{g}_i(\tau^i_k|\mathbf{x}_{k-1}^i))\\&+\beta(\mathbf{g}_i(\tau^i_k|\mathbf{x}_{k-1}^i)-\nabla J(\mathbf{x}^i_{k-1})+\nabla J(\mathbf{x}^i_{k-1})-\mathbf{u}^i_{k-1}).
    \end{split}
\end{equation}
Hence, taking the expectation on both sides of the last equality leads to
\begin{equation}
    \begin{split}
        \mathbb{E}[\mathbf{u}_k-\mathbf{u}_{k-1}]=(1-\beta)(\nabla\mathbf{J}(\mathbf{x}_k)-\nabla\mathbf{J}(\mathbf{x}_{k-1}))+\beta\mathbb{E}[\mathbf{g}(\tau_k|\mathbf{x}_k)-\mathbf{g}(\tau_k|\mathbf{x}_{k-1})]-\beta(\mathbf{u}_{k-1}-\nabla \mathbf{J}(\mathbf{x}_{k-1})).
    \end{split}
\end{equation}
We now investigate the term of $2\langle(\underline{\mathbf{W}}-\Lambda)\mathbf{v}_k,(\mathbf{I}-\Lambda)(\mathbf{u}_k-\mathbf{u}_{k-1})\rangle$ and denote it as $B_k$. Consequently, the following relationship is attained
\begin{equation}\label{eq.62}
    \begin{split}
        \mathbb{E}[B_k]&=2\langle(\underline{\mathbf{W}}-\Lambda)\mathbf{v}_k,(\mathbf{I}-\Lambda)\mathbb{E}[(\mathbf{u}_k-\mathbf{u}_{k-1})]\rangle\\&=2\langle(\underline{\mathbf{W}}-\Lambda)\mathbf{v}_k,(\mathbf{I}-\Lambda)[(1-\beta)(\nabla\mathbf{J}(\mathbf{x}_k)-\nabla\mathbf{J}(\mathbf{x}_{k-1}))+\beta\mathbb{E}[\mathbf{g}(\tau_k|\mathbf{x}_k)-\mathbf{g}(\tau_k|\mathbf{x}_{k-1})]\\&-\beta(\mathbf{u}_{k-1}-\nabla \mathbf{J}(\mathbf{x}_{k-1}))]\rangle\\&\leq 2\lambda\|\mathbf{v}_k-\Lambda\mathbf{v}_k\|\|(1-\beta)(\nabla\mathbf{J}(\mathbf{x}_k)-\nabla\mathbf{J}(\mathbf{x}_{k-1}))+\beta\mathbb{E}[\mathbf{g}(\tau_k|\mathbf{x}_k)-\mathbf{g}(\tau_k|\mathbf{x}_{k-1})]\\&-\beta(\mathbf{u}_{k-1}-\nabla \mathbf{J}(\mathbf{x}_{k-1}))\|\\&\leq \frac{1-\lambda^2}{2}\|\mathbf{v}_k-\Lambda\mathbf{v}_k\|^2+\frac{6\lambda^3(1-\beta)^2}{1-\lambda^2}\|\nabla\mathbf{J}(\mathbf{x}_k)-\nabla\mathbf{J}(\mathbf{x}_{k-1})\|^2\\&+\frac{6\lambda^3\beta^2}{1-\lambda^2}\mathbb{E}[\|\mathbf{g}(\tau_k|\mathbf{x}_k)-\mathbf{g}(\tau_k|\mathbf{x}_{k-1})\|^2]+\frac{6\lambda^3\beta^2}{1-\lambda^2}\|\mathbf{u}_{k-1}-\nabla \mathbf{J}(\mathbf{x}_{k-1})\|^2\\&\leq\frac{1-\lambda^2}{2}\|\mathbf{v}_k-\Lambda\mathbf{v}_k\|^2+(\beta^2+(1-\beta)^2)\frac{6\lambda^3L^2}{1-\lambda^2}\mathbb{E}[\|\mathbf{x}_k-\mathbf{x}_{k-1}\|^2]+\frac{6\lambda^3\beta^2}{1-\lambda^2}\|\mathbf{u}_{k-1}-\nabla \mathbf{J}(\mathbf{x}_{k-1})\|^2.
    \end{split}
\end{equation}
The first inequality is due to Cauthy-Schwartz inequality, while the second inequality follows from basic inequalities $2ab\leq ea^2+b^2/e$, with $e=\frac{1-\lambda^2}{2\lambda}$ for all $a,b\in\mathbb{R}$ and $\|\mathbf{a}+\mathbf{b}+\mathbf{c}\|^2\leq3\|\mathbf{a}\|^2+3\|\mathbf{b}\|^2+3\|\mathbf{c}\|^2, \forall \mathbf{a}, \mathbf{b}, \mathbf{c}\in\mathbb{R}^d$. The last inequality is due to the smoothness property.
Using Eq.~\ref{eq.60} yields the upper bound of the second moment of $\mathbf{u}^i_k-\mathbf{u}^i_{k-1}$,
\begin{equation}
    \begin{split}
        &\mathbb{E}[\|\mathbf{u}^i_k-\mathbf{u}^i_{k-1}\|^2]\leq 4(1-\beta)^2\mathbb{E}[\|\mathbf{g}_i(\tau^i_k|\mathbf{x}_k^i)-\upsilon_i(\tau^i_k|\mathbf{x}^i_{k-1},\mathbf{x}^i_k)\mathbf{g}_i(\tau^i_k|\mathbf{x}_{k-1}^i)\|^2]\\&+4\beta^2\mathbb{E}[\|\mathbf{g}_i(\tau^i_k|\mathbf{x}_k^i)-\mathbf{g}_i(\tau^i_k|\mathbf{x}_{k-1}^i)\|^2]+4\beta^2\mathbb{E}[\|\mathbf{g}_i(\tau^i_k|\mathbf{x}_{k-1}^i)-\nabla J(\mathbf{x}^i_{k-1})\|^2]\\&+4\beta^2\mathbb{E}[\|\nabla J(\mathbf{x}^i_{k-1})-\mathbf{u}^i_{k-1}\|^2]\\&= 4(1-\beta)^2\mathbb{E}[\|\mathbf{g}_i(\tau^i_k|\mathbf{x}_k^i)-\mathbf{g}_i(\tau^i_k|\mathbf{x}_{k-1}^i)+\mathbf{g}_i(\tau^i_k|\mathbf{x}_{k-1}^i)-\upsilon_i(\tau^i_k|\mathbf{x}^i_{k-1},\mathbf{x}^i_k)\mathbf{g}_i(\tau^i_k|\mathbf{x}_{k-1}^i)\|^2]\\&+4\beta^2\mathbb{E}[\|\mathbf{g}_i(\tau^i_k|\mathbf{x}_k^i)-\mathbf{g}_i(\tau^i_k|\mathbf{x}_{k-1}^i)\|^2]+4\beta^2\mathbb{E}[\|\mathbf{g}_i(\tau^i_k|\mathbf{x}_{k-1}^i)-\nabla J(\mathbf{x}^i_{k-1})\|^2]\\&+4\beta^2\mathbb{E}[\|\nabla J(\mathbf{x}^i_{k-1})-\mathbf{u}^i_{k-1}\|^2]\\&\leq4(1-\beta)^2\mathbb{E}[2\|\mathbf{g}_i(\tau^i_k|\mathbf{x}_k^i)-\mathbf{g}_i(\tau^i_k|\mathbf{x}_{k-1}^i)\|^2+2\|\mathbf{g}_i(\tau^i_k|\mathbf{x}_{k-1}^i)-\upsilon_i(\tau^i_k|\mathbf{x}^i_{k-1},\mathbf{x}^i_k)\mathbf{g}_i(\tau^i_k|\mathbf{x}_{k-1}^i)\|^2]\\&+4\beta^2\mathbb{E}[\|\mathbf{g}_i(\tau^i_k|\mathbf{x}_k^i)-\mathbf{g}_i(\tau^i_k|\mathbf{x}_{k-1}^i)\|^2]+4\beta^2\mathbb{E}[\|\mathbf{g}_i(\tau^i_k|\mathbf{x}_{k-1}^i)-\nabla J(\mathbf{x}^i_{k-1})\|^2]\\&+4\beta^2\mathbb{E}[\|\nabla J(\mathbf{x}^i_{k-1})-\mathbf{u}^i_{k-1}\|^2]\\&=8(1-\beta)^2\mathbb{E}[\|\mathbf{g}_i(\tau^i_k|\mathbf{x}_k^i)-\mathbf{g}_i(\tau^i_k|\mathbf{x}_{k-1}^i)\|^2]\\&+8(1-\beta)^2\mathbb{E}[\|\mathbf{g}_i(\tau^i_k|\mathbf{x}_{k-1}^i)-\upsilon_i(\tau^i_k|\mathbf{x}^i_{k-1},\mathbf{x}^i_k)\mathbf{g}_i(\tau^i_k|\mathbf{x}_{k-1}^i)\|^2]\\&+4\beta^2\mathbb{E}[\|\mathbf{g}_i(\tau^i_k|\mathbf{x}_k^i)-\mathbf{g}_i(\tau^i_k|\mathbf{x}_{k-1}^i)\|^2]+4\beta^2\mathbb{E}[\|\mathbf{g}_i(\tau^i_k|\mathbf{x}_{k-1}^i)-\nabla J(\mathbf{x}^i_{k-1})\|^2]\\&+4\beta^2\mathbb{E}[\|\nabla J(\mathbf{x}^i_{k-1})-\mathbf{u}^i_{k-1}\|^2]\\&\leq 8(1-\beta)^2L^2\mathbb{E}[\|\mathbf{x}^i_k-\mathbf{x}^i_{k-1}\|^2]+8(1-\beta)^2G^2C_\upsilon\mathbb{E}[\|\mathbf{x}^i_k-\mathbf{x}^i_{k-1}\|^2]\\&+4\beta^2L^2\mathbb{E}[\|\mathbf{x}^i_k-\mathbf{x}^i_{k-1}\|^2]+4\beta^2\sigma^2_i+4\beta^2\mathbb{E}[\|\nabla J(\mathbf{x}^i_{k-1})-\mathbf{u}^i_{k-1}\|^2]\\&=[8(1-\beta)^2L^2+8(1-\beta)^2G^2C_\upsilon+4\beta^2L^2]\mathbb{E}[\|\mathbf{x}^i_k-\mathbf{x}^i_{k-1}\|]\\&+4\beta^2\sigma^2_i+4\beta^2\mathbb{E}[\|\nabla J(\mathbf{x}^i_{k-1})-\mathbf{u}^i_{k-1}\|^2].
    \end{split}
\end{equation}
The last inequality is based on the smoothness property, Lemma~\ref{lemma_1} and Lemma~\ref{lemma_2}. We then have that
\begin{equation}\label{eq.64}
    \begin{split}
        \mathbb{E}[\|\mathbf{u}_k-\mathbf{u}_{k-1}\|^2]&\leq[8(1-\beta)^2L^2+8(1-\beta)^2G^2C_\upsilon+4\beta^2L^2]\mathbb{E}[\|\mathbf{x}_k-\mathbf{x}_{k-1}\|^2]+\\&4\beta^2N\bar{\sigma}^2+4\beta^2\mathbb{E}[\|\nabla \mathbf{J}(\mathbf{x}_{k-1}-\mathbf{u}_{k-1})\|^2].
    \end{split}
\end{equation}
Taking the expectation on both sides of Eq.~\ref{eq.59} and substituting Eq.~\ref{eq.62} and Eq.~\ref{eq.64} into it, the following relationship can be obtained
\begin{equation}
    \begin{split}
        &\mathbb{E}[\|\mathbf{v}_{k+1}-\Lambda\mathbf{v}_{k+1}\|^2]\leq \lambda^2\mathbb{E}[\|\mathbf{v}_{k}-\Lambda\mathbf{v}_{k}\|^2]+\frac{1-\lambda^2}{2}\mathbb{E}[\|\mathbf{v}_{k}-\Lambda\mathbf{v}_{k}\|^2]\\&+(\beta^2+(1-\beta)^2)\frac{6\lambda^3L^2}{1-\lambda^2}\mathbb{E}[\|\mathbf{x}_k-\mathbf{x}_{k-1}\|^2]\\&+\frac{6\lambda^3\beta^2}{1-\lambda^2}\mathbb{E}[\|\mathbf{u}_{k-1}-\nabla \mathbf{J}(\mathbf{x}_{k-1})\|^2]+4\beta^2N\bar{\sigma}^2\\&+[8(1-\beta)^2L^2+8(1-\beta)^2G^2C_\upsilon+4\beta^2L^2]\mathbb{E}[\|\mathbf{x}_k-\mathbf{x}_{k-1}\|^2]\\&+4\beta^2\mathbb{E}[\|\mathbf{u}_{k-1}-\nabla \mathbf{J}(\mathbf{x}_{k-1})\|^2]\\&=\frac{1+\lambda^2}{2}\lambda^2\mathbb{E}[\|\mathbf{v}_{k}-\Lambda\mathbf{v}_{k}\|^2]\\&+\bigg[(\beta^2+(1-\beta)^2)\frac{6\lambda^3L^2}{1-\lambda^2}+8(1-\beta)^2L^2+8(1-\beta)^2G^2C_\upsilon+4\beta^2L^2\bigg]\mathbb{E}[\|\mathbf{x}_k-\mathbf{x}_{k-1}\|^2]\\&+\bigg(4\beta^2+\frac{6\lambda^3\beta^2}{1-\lambda^2}\bigg)\mathbb{E}[\|\mathbf{u}_{k-1}-\nabla \mathbf{J}(\mathbf{x}_{k-1})\|^2]+4\beta^2N\bar{\sigma}^2\\&\leq \frac{1+\lambda^2}{2}\lambda^2\mathbb{E}[\|\mathbf{v}_{k}-\Lambda\mathbf{v}_{k}\|^2]+\bigg(\frac{12\lambda^3L^2}{1-\lambda^2}+8(L^2+G^2C_\upsilon)+4L^2\bigg)\mathbb{E}[\|\mathbf{x}_k-\mathbf{x}_{k-1}\|^2]\\&+4N\beta^2\bar{\sigma}^2+\frac{10\beta^2}{1-\lambda^2}\mathbb{E}[\|\mathbf{u}_{k-1}-\nabla \mathbf{J}(\mathbf{x}_{k-1})\|^2]\\&\leq \frac{1+\lambda^2}{2}\lambda^2\mathbb{E}[\|\mathbf{v}_{k}-\Lambda\mathbf{v}_{k}\|^2]+\frac{16L^2+8(L^2+G^2C_\upsilon)}{1-\lambda^2}\mathbb{E}[\|\mathbf{x}_k-\mathbf{x}_{k-1}\|^2]+4\beta^2N\bar{\sigma}^2\\&+\frac{10\beta^2}{1-\lambda^2}\mathbb{E}[\|\mathbf{u}_{k-1}-\nabla \mathbf{J}(\mathbf{x}_{k-1})\|^2].
    \end{split}
\end{equation}
Since
\[\|\mathbf{x}_{k}-\mathbf{x}_{k-1}\|^2\leq 3\|\mathbf{x}_k-\Lambda\mathbf{x}_k\|^2 + 3N\eta^2\|\bar{\mathbf{u}}_{k-1}\|^2+3\|\mathbf{x}_{k-1}-\Lambda\mathbf{x}_{k-1}\|^2,\] combining Fact~\ref{fact_1}, the following is obtained
\begin{equation}
    \|\mathbf{x}_{k}-\mathbf{x}_{k-1}\|^2\leq 6\lambda^2\eta^4\|\mathbf{v}_k-\Lambda \mathbf{v}_k\|^2 + 3N\eta^2\|\bar{\mathbf{u}}_{k-1}\|^2+9\|\mathbf{x}_{k-1}-\Lambda\mathbf{x}_{k-1}\|^2.
\end{equation}
Combining the last two inequalities, one can attain
\begin{equation}
    \begin{split}
        &\mathbb{E}[\|\mathbf{v}_{k+1}-\Lambda\mathbf{v}_{k+1}\|^2]\leq \bigg(\frac{1+\lambda^2}{2}+\frac{96L^2+48(L^2+G^2C_\upsilon)}{1-\lambda^2}\lambda^2\eta^2\bigg)\mathbb{E}[\|\mathbf{v}_{k}-\Lambda\mathbf{v}_{k}\|^2]\\&+\frac{48L^2+24(L^2+G^2C_\upsilon)}{1-\lambda^2}N\eta^2\mathbb{E}[\|\bar{\mathbf{u}}_{k-1}\|^2]+\frac{144L^2+72(L^2+G^2C_\upsilon)}{1-\lambda^2}\mathbb{E}[\|\mathbf{x}_{k-1}-\Lambda\mathbf{x}_{k-1}\|^2]\\&+4\beta^2N\bar{\sigma}^2+\frac{10\beta^2}{1-\lambda^2}\mathbb{E}[\|\mathbf{u}_{k-1}-\nabla \mathbf{J}(\mathbf{x}_{k-1})\|^2].
    \end{split}
\end{equation}
Combining the fact that $0<\eta\leq\frac{1-\lambda^2}{8\lambda\sqrt{9L^2+3G^2C_\upsilon}}$ yields the desirable result.
\end{proof}
According to Lemma~\ref{lemma_8} and Lemma~\ref{lemma_9}, the remaining step to obtain the explicitly accurate error bound is to inaugurate the correlation between $\sum_{k=0}^K\mathbb{E}[\|\mathbf{x}_k-\Lambda\mathbf{x}_k\|^2]$ and $\sum_{k=0}^K\mathbb{E}[\|\bar{\mathbf{u}}_{k}\|^2]$. Thus, the following lemma is constructed for this purpose.
\begin{lemma}\label{lemma_10}
Let $\mathbf{x}_k$ be generated by \texttt{MDPGT} initialized with a mini-batch of trajectories $\mathcal{B}$. If $0<\eta\leq\frac{(1-\lambda^2)^2}{\lambda\sqrt{14592L^2+9984G^2C_\upsilon}}$ and $\beta\in(0,1)$, then for $K\geq 2$, we have
\begin{equation}
    \begin{split}
        &\sum_{k=0}^K\mathbb{E}[\|\mathbf{x}_k-\Lambda\mathbf{x}_k\|^2]\leq\Bigg[\frac{1436L^2+4608(L^2+G^2C_\upsilon)}{(1-\lambda^2)^4}\Bigg]\cdot\\&\lambda^2N\eta^4\sum_{k=0}^{K-2}\mathbb{E}[\|\bar{\mathbf{u}}_k\|^2]+\frac{32\lambda^2N\bar{\sigma}^2\eta^2}{(1-\lambda^2)^3|\mathcal{B}|}\bigg(1+\frac{10\beta}{1-\lambda^2}\bigg)\\&+\frac{32\lambda^2\eta^2}{(1-\lambda^2)^3}\|\nabla \mathbf{J}(\mathbf{x}_0)\|^2+\frac{128\lambda^2N\beta^2K\bar{\sigma}^2\eta^2}{(1-\lambda^2)^3}\bigg(1+\frac{5\beta}{1-\lambda^2}\bigg).
    \end{split}
\end{equation}
\end{lemma}
\begin{proof}
Applying Eq.~\ref{eq.23} to the fact that \[\|\mathbf{x}_{k+1}-\Lambda\mathbf{x}_{k+1}\|^2\leq\frac{1+\lambda^2}{2}\|\mathbf{x}_{k}-\Lambda\mathbf{x}_{k}\|^2+\frac{2\eta^2\lambda^2}{1-\lambda^2}\|\mathbf{v}_{k+1}-\Lambda\mathbf{v}_{k+1}\|^2\] leads to \[\sum_{k=0}^K\|\mathbf{x}_k-\Lambda\mathbf{x}_k\|^2\leq \frac{4\lambda^2\eta^2}{(1-\lambda^2)^2}\sum^K_{k=1}\|\mathbf{v}_k-\Lambda\mathbf{v}_k\|^2.\] Similarly, applying Eq.~\ref{eq.24} to Eq.~\ref{eq.32} results in
\begin{equation}
    \begin{split}
       &\sum_{k=1}^K\mathbb{E}[\|\mathbf{v}_{k}-\Lambda\mathbf{v}_{k}\|^2]\leq\frac{4}{1-\lambda^2}\mathbb{E}[\|\mathbf{v}_{1}-\Lambda\mathbf{v}_{1}\|^2]+\frac{192L^2+96(L^2+G^2C_\upsilon)}{1-\lambda^2}N\eta^2\sum_{k=0}^{K-2}\mathbb{E}[\|\bar{\mathbf{u}}_{k-1}\|^2]\\&+\frac{576L^2+288(L^2+G^2C_\upsilon)}{(1-\lambda^2)^2}\sum_{k=0}^{K-2}\mathbb{E}[\|\mathbf{x}_k-\Lambda\mathbf{x}_k\|^2]+\frac{40\beta^2}{(1-\lambda^2)^2}\sum_{k=0}^{K-2}\mathbb{E}[\|\mathbf{u}_k-\nabla \mathbf{J}(\mathbf{x}_k)\|^2]\\&+\frac{16\beta^2N\bar{\sigma}^2K}{1-\lambda^2}\\&\leq\frac{4N\bar{\sigma}^2}{(1-\lambda^2)|\mathcal{B}|}+\frac{4\|\nabla \mathbf{J}(\mathbf{x}_0)\|^2}{1-\lambda^2}+\frac{192L^2+96(L^2+G^2C_\upsilon)}{1-\lambda^2}N\eta^2\sum_{k=0}^{K-2}\mathbb{E}[\|\bar{\mathbf{u}}_{k-1}\|^2]\\&+\frac{576L^2+288(L^2+G^2C_\upsilon)}{(1-\lambda^2)^2}\sum_{k=0}^{K-2}\mathbb{E}[\|\mathbf{x}_k-\Lambda\mathbf{x}_k\|^2]+\frac{40\beta^2}{(1-\lambda^2)^2}\sum_{k=0}^{K-2}\mathbb{E}[\|\mathbf{u}_k-\nabla \mathbf{J}(\mathbf{x}_k)\|^2]\\&+\frac{16\beta^2N\bar{\sigma}^2K}{1-\lambda^2}.
    \end{split}
\end{equation}
The last inequality follows from Lemma~\ref{lemma_9}. As
\begin{equation}
    \begin{split}
        \frac{40\beta^2}{(1-\lambda^2)^2}&\sum_{k=0}^{K-2}\mathbb{E}[\|\mathbf{u}_k-\nabla \mathbf{J}(\mathbf{x}_k)\|^2]\leq \frac{40N\bar{\sigma}^2\beta}{(1-\lambda^2)^2|\mathcal{B}|}+\frac{480\beta N(L^2+G^2C_\upsilon)\eta^2}{(1-\lambda^2)^2}\\&\sum_{k=0}^{K-1}\mathbb{E}[\|\bar{\mathbf{u}}_k\|^2]+\frac{960\beta(L^2+G^2C_\upsilon)}{(1-\lambda^2)^2}\sum_{k=0}^K\mathbb{E}[\|\mathbf{x}_k-\Lambda(\mathbf{x}_k)\|^2]\\&+\frac{80N\beta^3\bar{\sigma}^2K}{(1-\lambda^2)^2},
    \end{split}
\end{equation}
we then have
\begin{equation}
    \begin{split}
        \sum_{k=1}^K\mathbb{E}[\|\mathbf{v}_{k}-\Lambda\mathbf{v}_{k}\|^2]&\leq\frac{192L^2+576(L^2+G^2C_\upsilon)}{(1-\lambda^2)^2}N\eta^2\sum_{k=0}^{K-2}\mathbb{E}[\|\bar{\mathbf{u}}_k\|^2]\\&+\frac{576L^2+1248(L^2+G^2C_\upsilon)}{(1-\lambda^2)^2}\sum^{k=0}_{K-1}\mathbb{E}[\|\mathbf{x}_k-\Lambda\mathbf{x}_k\|^2]\\&\frac{4N\bar{\sigma}^2}{(1-\lambda^2)|\mathcal{B}|}\bigg(1+\frac{10\beta}{1-\lambda^2}\bigg)+\frac{16\beta^2N\bar{\sigma}^2K}{1-\lambda^2}\bigg(1+\frac{5\beta}{1-\lambda^2}\bigg)+\frac{4\|\nabla \mathbf{J}(\mathbf{x}_0)\|^2}{1-\lambda^2}.
    \end{split}
\end{equation}
Thus, we have
\begin{equation}
\begin{split}
    &\sum_{k=0}^{K}\mathbb{E}[\|\mathbf{x}_k-\Lambda\mathbf{x}_k\|^2]\leq \frac{768L^2+2304(L^2+G^2C_\upsilon)\lambda^2N\eta^4}{(1-\lambda^2)^4}\sum^{K-2}_{k=0}\mathbb{E}[\|\bar{\mathbf{u}}_k\|^2]\\&+\frac{2304L^2+4992(L^2+G^2C_\upsilon)}{(1-\lambda^2)^4}\lambda^2\eta^2\sum^{K-1}_{k=0}\mathbb{E}[\|\mathbf{x}_k-\Lambda\mathbf{x}_k\|^2]+\frac{16N\bar{\sigma}^2\lambda^2\eta^2}{(1-\lambda^2)^3|\mathcal{B}|}\bigg(1+\frac{10\beta}{1-\lambda^2}\bigg)\\&+\frac{64\beta^2N\bar{\sigma}^2K\eta^2\lambda^2}{(1-\lambda^2)^3}\bigg(1+\frac{5\beta}{1-\lambda^2}\bigg)+\frac{16\lambda^2\eta^2\|\nabla \mathbf{J}(\mathbf{x}_0)\|^2}{(1-\lambda^2)^3}.
\end{split}
\end{equation}
With simple mathematical manipulations, based on the conditions $0<\eta\leq\frac{(1-\lambda^2)^2}{\lambda\sqrt{14592L^2+9984G^2C_\upsilon}}$ and $\beta\in(0,1)$, it is obtained that \[1-\frac{2304L^2+4992(L^2+G^2C_\upsilon)}{(1-\lambda^2)^4}\lambda^2\eta^2\leq\frac{1}{2}.\] Hence the proof is completed by adopting this inequality.
\end{proof}
Hence, with Lemma~\ref{lemma_10}, it suffices to show Theorem~\ref{theorem_1}, whose proof is presented next. While for Theorem~\ref{theorem_2}, the same proof ideas can be applied and the corresponding results are obtained by setting $|\mathcal{B}|=1$. Thus, we are not going to repeat statements for the auxiliary lemmas, instead using the conclusions from the lemmas and adjusting slightly the constants in the error bounds.
\subsection*{Proof of Theorem~\ref{theorem_1}}
\begin{proof}
Recall the conclusion of Lemma~\ref{lemma_5}, we have\[\sum_{k=0}^{K}\|\nabla J(\bar{\mathbf{x}}_k)\|^2\leq\frac{2(J^*-J(\bar{\mathbf{x}}_0))}{\eta}-\frac{1}{2}\sum_{k=0}^K\|\bar{\mathbf{u}}_k\|^2+2\sum_{k=0}^K\|\bar{\mathbf{u}}_k-\overline{\nabla\mathbf{J}}(\mathbf{x}_k)\|^2+\frac{2L^2}{N}\sum_{k=0}^K\|\mathbf{x}_k-\Lambda\mathbf{x}_k\|^2.\] Substituting Eq.~\ref{ave_gradient_var} into the last inequality yields 
\begin{equation}
    \begin{split}
    &\sum_{k=0}^{K}\|\nabla J(\bar{\mathbf{x}}_k)\|^2\leq \frac{2(J^*-J(\bar{\mathbf{x}}_0))}{\eta}-\frac{1}{2}\sum_{k=0}^K\mathbb{E}[\|\bar{\mathbf{u}}_k\|^2]+\frac{2L^2}{N}\sum^{K}_{k=0}\mathbb{E}[\|\mathbf{x}_k-\Lambda\mathbf{x}_k\|^2]\\&+\frac{2\bar{\sigma}^2}{N\beta|\mathcal{B}|}+\frac{24(L^2+G^2C_\upsilon)\eta^2}{N\beta}\sum_{k=0}^{K-1}\mathbb{E}[\|\bar{\mathbf{u}}_k\|^2]+\frac{48(L^2+G^2C_\upsilon)}{N^2\beta}\sum^K_{k=0}\mathbb{E}[\|\mathbf{x}_k-\Lambda\mathbf{x}_k\|^2]+\frac{4\beta\bar{\sigma}^2K}{N}\\&=\frac{2(J^*-J(\bar{\mathbf{x}}_0))}{\eta}-\frac{1}{4}\sum_{k=0}^K\mathbb{E}[\|\bar{\mathbf{u}}_k\|^2]+\frac{2}{N}\bigg(L^2+\frac{24(L^2+G^2C_\upsilon)}{N\beta}\bigg)\sum^{K}_{k=0}\mathbb{E}[\|\mathbf{x}_k-\Lambda\mathbf{x}_k\|^2]\\&+\frac{2\bar{\sigma}^2}{N\beta|\mathcal{B}|}+\frac{4\beta\bar{\sigma}^2K}{N}-\bigg(\frac{1}{4}-\frac{24(L^2+G^2C_\upsilon)\eta^2}{N\beta}\bigg)\sum_{k=0}^{K}\mathbb{E}[\|\bar{\mathbf{u}}_k\|^2].
    \end{split}
\end{equation}
To get rid of the last term on the right hand side of the above inequality, the following relationship can be obtained
\[\frac{1}{4}-\frac{24(L^2+G^2C_\upsilon)\eta^2}{N\beta}\geq 0\Rightarrow \frac{96(L^2+G^2C_\upsilon)\eta^2}{N}\leq \beta< 1\Rightarrow 0<\eta<\frac{1}{6\sqrt{6(L^2+G^2C_\upsilon)}}.\] It is easily to verified the above relationship as we have that $\beta=\frac{96(L^2+G^2C_\upsilon)\eta^2}{N}$.
Thus,
\begin{equation}
    \begin{split}
        &\sum_{k=0}^{K}\|\nabla J(\bar{\mathbf{x}}_k)\|^2\leq\frac{2(J^*-J(\bar{\mathbf{x}}_0))}{\eta}-\frac{1}{4}\sum_{k=0}^K\mathbb{E}[\|\bar{\mathbf{u}}_k\|^2]+\frac{2\bar{\sigma}^2}{N\beta|\mathcal{B}|}+\frac{4\beta\bar{\sigma}^2K}{N}\\&+\frac{2}{N}\bigg(L^2+\frac{24(L^2+G^2C_\upsilon)}{N\beta}\bigg)\sum^{K}_{k=0}\mathbb{E}[\|\mathbf{x}_k-\Lambda\mathbf{x}_k\|^2].
    \end{split}
\end{equation}
As
\begin{equation}
    \begin{split}
        &\frac{1}{N}\sum_{i=1}^N\sum_{k=0}^K\mathbb{E}[\|\nabla J(\mathbf{x}^i_k)\|^2]\leq\frac{2}{N}\sum^N_{i=1}\sum^K_{k=0}\mathbb{E}[\|\nabla J(\mathbf{x}^i_k)-\nabla J(\bar{\mathbf{x}}_k)\|^2+\|\nabla J(\bar{\mathbf{x}}_k)\|^2]\\&\leq \frac{2L^2}{N}\sum^K_{k=0}\mathbb{E}[\|\mathbf{x}_k-\Lambda\mathbf{x}_k\|^2]+2\sum^K_{k=0}\mathbb{E}[\|\nabla J(\bar{\mathbf{x}}_k)\|^2],
    \end{split}
\end{equation}
we have
\begin{equation}
    \begin{split}
        &\frac{1}{N}\sum_{i=1}^N\sum_{k=0}^K\mathbb{E}[\|\nabla J(\mathbf{x}^i_k)\|^2]\leq\frac{2L^2}{N}\sum^K_{k=0}\mathbb{E}[\|\mathbf{x}_k-\Lambda\mathbf{x}_k\|^2]+\frac{4(J^*-J(\bar{\mathbf{x}}_0))}{\eta}\\&-\frac{1}{2}\sum_{k=0}^K\mathbb{E}[\|\bar{\mathbf{u}}_k\|^2]+\frac{4\bar{\sigma}^2}{N\beta|\mathcal{B}|}+\frac{8\beta\bar{\sigma}^2K}{N}+\frac{4}{N}\bigg(L^2+\frac{24(L^2+G^2C_\upsilon)}{N\beta}\bigg)\sum^{K}_{k=0}\mathbb{E}[\|\mathbf{x}_k-\Lambda\mathbf{x}_k\|^2]\\&=\frac{6}{N}\bigg(L^2+\frac{16(L^2+G^2C_\upsilon)}{N\beta}\bigg)\sum^{K}_{k=0}\mathbb{E}[\|\mathbf{x}_k-\Lambda\mathbf{x}_k\|^2]+\frac{4(J^*-J(\bar{\mathbf{x}}_0))}{\eta}-\frac{1}{2}\sum_{k=0}^K\mathbb{E}[\|\bar{\mathbf{u}}_k\|^2]\\&+\frac{4\bar{\sigma}^2}{N\beta|\mathcal{B}|}+\frac{8\beta\bar{\sigma}^2K}{N}.
    \end{split}
\end{equation}
As
$\frac{6}{N}\bigg(L^2+\frac{16(L^2+G^2C_\upsilon)}{N\beta}\bigg)=\frac{6}{N}\bigg(L^2+\frac{1}{6\eta^2L^2}\bigg)$ and $\eta^2L^2\leq\eta^2(L^2+G^2C_\upsilon)\leq\frac{1}{96}$ based on $\eta<\frac{1}{6\sqrt{6(L^2+G^2C_\upsilon)}}$, we have\[\frac{6}{N}\bigg(L^2+\frac{16(L^2+G^2C_\upsilon)}{N\beta}\bigg)<\frac{17}{16N\eta^2}.\]Thus, the following relationship can be attained
\begin{equation}\label{eq.76}
    \begin{split}
        \frac{1}{N}\sum_{i=1}^N\sum_{k=0}^K\mathbb{E}[\|\nabla J(\mathbf{x}^i_k)\|^2]&\leq\frac{17}{16\eta^2N}\mathbb{E}[\|\mathbf{x}_k-\Lambda\mathbf{x}_k\|^2]+\frac{4(J^*-J(\bar{\mathbf{x}}_0))}{\eta}-\frac{1}{2}\sum_{k=0}^K\mathbb{E}[\|\bar{\mathbf{u}}_k\|^2]\\&+\frac{4\bar{\sigma}^2}{N\beta|\mathcal{B}|}+\frac{8\beta\bar{\sigma}^2K}{N}.
    \end{split}
\end{equation}
We investigate the combined term $-\frac{1}{2}\sum_{k=0}^K\mathbb{E}[\|\bar{\mathbf{u}}_k\|^2]+\frac{17}{16\eta^2N}\mathbb{E}[\|\mathbf{x}_k-\Lambda\mathbf{x}_k\|^2]$. Since
\begin{equation}\label{eq.77}
    \begin{split}
        &-\frac{1}{2}\sum_{k=0}^K\mathbb{E}[\|\bar{\mathbf{u}}_k\|^2]+\frac{17}{16\eta^2N}\mathbb{E}[\|\mathbf{x}_k-\Lambda\mathbf{x}_k\|^2]\leq -\frac{1}{2}\sum_{k=0}^K\mathbb{E}[\|\bar{\mathbf{u}}_k\|^2]\\&+\frac{17}{16\eta^2N}\bigg[\bigg(\frac{1436L^2+4608(L^2+G^2C_\upsilon)}{(1-\lambda^2)^4}\bigg)\lambda^2N\eta^4\sum_{k=0}^{K-2}\mathbb{E}[\|\bar{\mathbf{u}}_k\|^2]+\frac{32\lambda^2N\bar{\sigma}^2\eta^2}{(1-\lambda^2)^3|\mathcal{B}|}\\&\bigg(1+\frac{10\beta}{1-\lambda^2}\bigg)+\frac{32\lambda^2\eta^2}{(1-\lambda^2)^3}\|\nabla \mathbf{J}(\mathbf{x}_0)\|^2+\frac{128\lambda^2N\beta^2K\bar{\sigma}^2\eta^2}{(1-\lambda^2)^3}\bigg(1+\frac{5\beta}{1-\lambda^2}\bigg)\bigg]\\&=-\frac{1}{2}\sum_{k=0}^K\mathbb{E}[\|\bar{\mathbf{u}}_k\|^2]+\frac{1526L^2+4896(L^2+G^2C_\upsilon)}{(1-\lambda^2)^4}\lambda^2\eta^2\sum_{k=0}^{K-2}\mathbb{E}[\|\bar{\mathbf{u}}_k\|^2]\\&+\frac{34\lambda^2}{N(1-\lambda^2)^3}\|\nabla \mathbf{J}(\mathbf{x}_0)\|^2+\frac{34\lambda^2\bar{\sigma}^2}{(1-\lambda^2)^3|\mathcal{B}|}\bigg(1+\frac{10\beta}{1-\lambda^2}\bigg)+\frac{136\lambda^2\beta^2K\bar{\sigma}^2}{(1-\lambda^2)^3}\bigg(1+\frac{5\beta}{1-\lambda^2}\bigg)\\&=-\frac{1}{2}\bigg(1-\frac{3052L^2+9792(L^2+G^2C_\upsilon)}{(1-\lambda^2)^4}\lambda^2\eta^2\bigg)\sum_{k=0}^{K-2}\mathbb{E}[\|\bar{\mathbf{u}}_k\|^2]\\&+\frac{34\lambda^2}{N(1-\lambda^2)^3}\|\nabla \mathbf{J}(\mathbf{x}_0)\|^2+\frac{34\lambda^2\bar{\sigma}^2}{(1-\lambda^2)^3|\mathcal{B}|}\bigg(1+\frac{10\beta}{1-\lambda^2}\bigg)+\frac{136\lambda^2\beta^2K\bar{\sigma}^2}{(1-\lambda^2)^3}\bigg(1+\frac{5\beta}{1-\lambda^2}\bigg).
    \end{split}
\end{equation}
With the condition that $\eta\leq\frac{(1-\lambda^2)^2}{\lambda\sqrt{12844L^2+9792G^2C_\upsilon}}$, it is immediately that \[1-\frac{3052L^2+9792(L^2+G^2C_\upsilon)}{(1-\lambda^2)^4}\lambda^2\eta^2\geq 0\] such that the first term on the right hand side of Eq.~\ref{eq.77} can be removed. Moreover, since $\eta\leq\frac{\sqrt{N(1-\lambda^2)}\lambda}{31\sqrt{L^2+G^2C_\upsilon}}$, which leads to $\beta\leq\frac{(1-\lambda^2)\lambda^2}{10}$ such that
\[1+\frac{10\beta}{1-\lambda^2}< 2,\; 1+\frac{5\beta}{1-\lambda^2}<\frac{3}{2}.\] By now we can conclude that
\begin{equation}\label{eq.78}
    \begin{split}
        -\frac{1}{2}\sum_{k=0}^K\mathbb{E}[\|\bar{\mathbf{u}}_k\|^2]+\frac{17}{16\eta^2N}\mathbb{E}[\|\mathbf{x}_k-\Lambda\mathbf{x}_k\|^2]&\leq\frac{34\lambda^2}{N(1-\lambda^2)^3}\|\nabla \mathbf{J}(\mathbf{x}_0)\|^2+\frac{68\lambda^2\bar{\sigma}^2}{(1-\lambda^2)^3|\mathcal{B}|}\\&+\frac{204\lambda^2\beta^2K\bar{\sigma}^2}{(1-\lambda^2)^3}.
    \end{split}
\end{equation}
Substituting Eq.~\ref{eq.78} into Eq.~\ref{eq.76} and dividing both sides by $\frac{1}{K+1}$, and with $\mathbb{E}[\|\nabla J(\tilde{\mathbf{x}}_K)\|^2]=\frac{1}{N(K+1)}\sum_{i=1}^N\sum_{k=0}^K\mathbb{E}[\|\nabla J(\mathbf{x}^i_k)\|^2]$ completes the proof.
\end{proof}
\subsection*{Proof of Corollary~\ref{coro_1}}
\begin{proof}
Substituting $\eta, \beta, \textnormal{and} |\mathcal{B}|$ into the Eq.~\ref{eq.14} and conducting some simple mathematical manipulations can easily attain the desirable result.
\end{proof}
\subsection*{Analysis for \texttt{MDPGT} with Single Trajectory Initialization}
\begin{theorem}\label{theorem_2}
Let Assumptions~\ref{assum_1},\ref{assum_3} and~\ref{assum_4} hold. Let the momentum coefficient $\beta=\frac{96L^2+96G^2C_\upsilon}{N}\eta^2$. If \texttt{MDPGT} is initialized by a single trajectory and the step size satisfies the following condition
\begin{equation}\begin{split}0<\eta\leq\textnormal{min}&\bigg\{\frac{(1-\lambda^2)^2}{\lambda\sqrt{12844L^2+9792G^2C_\upsilon}},\frac{\sqrt{N(1-\lambda^2)}\lambda}{31\sqrt{L^2+G^2C_\upsilon}},\\&\frac{1}{6\sqrt{6(L^2+G^2C_\upsilon)}}\bigg\},\end{split}\end{equation} then the output $\tilde{\mathbf{x}}_K$ satisfies: for all $K\geq 2$:
\begin{equation}\label{eq.17}
\begin{split}
    &\mathbb{E}[\|\nabla J(\tilde{\mathbf{x}}_K)\|^2]\leq \frac{4(J^*-J(\bar{\mathbf{x}}_0))}{\eta K}+\frac{4\bar{\sigma}^2}{N\beta K}+\frac{8\beta\bar{\sigma}^2}{N}\\&+\frac{34\lambda^2}{KN(1-\lambda^2)^3}\|\nabla \mathbf{J}(\bar{\mathbf{x}}_0)\|^2+\frac{68\lambda^2\bar{\sigma}^2}{(1-\lambda^2)^3K}\\&+\frac{204\lambda^2\beta^2\bar{\sigma}^2}{(1-\lambda^2)^3},
\end{split}
\end{equation}
where $J^*$ is the upper bound of $J(\mathbf{x})$ and $\|\nabla \mathbf{J}(\bar{\mathbf{x}}_0)\|^2\triangleq\sum^N_{i=1}\|\nabla J_i(\bar{\mathbf{x}}_0)\|^2$.
\end{theorem}
\begin{proof}
As the initialization is single trajectory, $|\mathcal{B}|=1$. Following the similar proof techniques as already shown in Theorem~\ref{theorem_1} yields the desirable result. The proof process is not repeated in this context.
\end{proof}
Similarly, in view of Theorem~\ref{theorem_2}, the asymptotic behavior is the same as in Theorem~\ref{theorem_1}, having the same steady-state error. Due to an infinite time horizon, regardless of the initialization strategy, diverse agents are able to learn effectively in a collaborative manner. However, in terms of the non-asymptotic property, the single trajectory initialization strategy makes a difference in the sampling complexity, which is reflected by the following result.
\begin{corollary}\label{coro_2}
Let $\eta=\frac{N^{3/4}}{8LK^{1/4}}, \beta=\frac{DN^{1/2}}{64L^2K^{1/2}}$, in Theorem~\ref{theorem_2}. We have
\begin{equation}\label{eq_coro2}
\begin{split}
    &\mathbb{E}[\|\nabla J(\tilde{\mathbf{x}}_K)\|^2]\leq\frac{32L(J^*-J(\bar{\mathbf{x}}_0))}{(NK)^{3/4}}+\frac{2048L^4\bar{\sigma}^2+D^2\bar{\sigma}^2N}{8L^2DN^{3/2}K^{1/2}}\\&+\frac{\lambda^2}{(1-\lambda^2)^3K}\bigg(\frac{34\|\nabla \mathbf{J}(\bar{\mathbf{x}}_0)\|^2}{N}+68\bar{\sigma}^2+\frac{51\bar{\sigma}^2D^2N}{1024L^4}\bigg),
\end{split}
\end{equation}

for all \begin{equation}\begin{split}&K\geq\textnormal{max}\bigg\{\frac{N^3D^2}{4096L^4},\frac{923521N(L^2+G^2C_\upsilon)^2}{4096L^4(1-\lambda^2)^2\lambda^4},\\& \frac{(12844L^2+9792G^2C_\upsilon)^{2}\lambda^4N^3}{4096L^4(1-\lambda^2)^8}\bigg\},\end{split}\end{equation} where
$D=96L^2+96G^2C_\upsilon$.
\end{corollary}
\begin{proof}
Substituting $\eta, \beta$ into the Eq.~\ref{eq.17} and conducting some simple mathematical manipulations can easily attain the desirable result.
\end{proof}
Corollary~\ref{coro_2} implies that with only a single trajectory initialization for \texttt{MDPGT}, when $K$ is sufficiently large, the mean-squared convergence rate is \[\mathbb{E}[\|\nabla J(\tilde{\mathbf{x}}_K)\|^2]\leq\mathcal{O}\bigg(\frac{1}{(NK)^{1/2}}\bigg),\]which is slower than that obtained in Corollary~\ref{coro_1}. With similar mathematical manipulation, the eventual sampling complexity is $\mathcal{O}(N^{-1}\epsilon^{-4})$. Though variance reduction techniques has not reduced the order of $\epsilon^{-1}$, compared to the SOTA approaches, the linear speed up still enables the complexity to be $N$ times smaller than that in~\cite{xu2019sample,huang2020momentum}. Additionally, different from traditional decentralized learning problems, MARL has more significant variances in the optimization procedure due to the non-oblivious characteristic. Using just a single trajectory for each agent to initialize is a quite poor scheme, but the adopted variance reduction techniques can successfully maintain the SOTA sampling complexity in a decentralized setting.
\subsection*{Implication for Gaussian Policy}
In this section, we study the sample complexity when the policy function $\pi^i(a^i|s)$ of each agent is explicitly a Gaussian distribution. For a bounded action space $\mathcal{A}^i\subset\mathbb{R}$, a Gaussian policy parameterized by $\mathbf{x}_i$ is defined as
\begin{equation}\label{gaussion_dis}
    \pi^i(a^i|s)=\frac{1}{\sqrt{2\pi}}\textnormal{exp}\bigg(-\frac{((\mathbf{x}^i)^\top\phi_i(s)-a^i)^2}{2\xi^2}\bigg),
\end{equation}
where $\xi^2$ is a constant standard deviation parameter and $\phi_i(s):\mathcal{S}\to\mathbb{R}^{d_i}$ is mapping from the state space to the feature space. Note that the standard deviation parameter can be varying in terms of different agents, while for simplicity in this context, we assume that it is the same for each agent. Next, we verify the assumptions on the Gaussian policy. We first impose a mild assumption that the action space and feature space are bounded, i.e., there exist constants $C_a>0$ and $C_f>0$ such that for all $i\in\mathcal{V}$, $|a^i|\leq C_a, \forall a^i\in\mathcal{A}^i$ and $\|\phi_i(s)\|\leq C_f, \forall s\in\mathcal{S}$. Then we can show that Assumption~\ref{assum_1} hold. Due to the limit of space, we defer the derivation to the supplementary materials. While Assumption~\ref{assum_4} does not always hold for all Gaussian distributions, but based on a result from~\cite{cortes2010learning}, it has been shown that for two Gaussian distributions, $\pi^i_{\mathbf{x}_1}(a^i|s)=\mathcal{N}(\mu_1,\xi_1^2)$ and $\pi^i_{\mathbf{x}_2}(a^i|s)=\mathcal{N}(\mu_2,\xi_2^2)$, if $\xi_2>\frac{\sqrt{2}}{2}\xi_1$, then the variance of the importance sampling weight $\upsilon_i(\tau^i|\mathbf{x}_1,\mathbf{x}_2)$ is bounded. Due to a constant standard deviation parameter defined for the Gaussian distributions, it is easily verified that for any time step $k\geq 0$, $\mathbb{V}(\upsilon_i(\tau^i_k|\mathbf{x}^i_{k-1},\mathbf{x}_{k}^i))$ should be bounded by some constant $\mathcal{M}>0$. 
We recall Corollary~\ref{coro_1} that can apply to any general policy models. Therefore, based on the above discussion, it also applies to the Gaussian policy function scenario. We present the result in the following corollary to explicitly state the relationship between $\gamma$ and $\epsilon$ in the sampling complexity. We first present an established result that bounds the variance of policy gradient estimator $\mathbf{g}_i(\tau^i|\mathbf{x}^i)$ in Eq.~\ref{gradient_estimator} when the policy function $\pi^i$ follows a Gaussian distribution.
\begin{lemma}(Lemma 5.5 in \cite{pirotta2013adaptive})\label{lemma_11}
Given a Gaussian policy $\pi^i_{\mathbf{x}^i}=\mathcal{N}((\mathbf{x}^i)^\top\phi_i(s), \xi^2)$, for all $i\in\mathcal{V}$, if $|r_i(s,a^i)|\leq R$ and $\|\phi^i(s)\|\leq C_f$ for all $s\in\mathcal{S}, a^i\in\mathcal{A}^i$ and $R>0, C_f>0$ are constants, then the variance of policy gradient estimator $\mathbb{V}(\mathbf{g}_i(\tau^i|\mathbf{x}^i))$ can be bounded as
\begin{equation}
    \mathbb{V}(\mathbf{g}_i(\tau^i|\mathbf{x}^i))\leq\frac{R^2C_f^2}{(1-\gamma)^2\xi^2}\bigg(\frac{1-\gamma^{2H}}{1-\gamma^2}-H\gamma^{2H}-2\gamma^H\frac{1-\gamma^H}{1-\gamma}\bigg).
\end{equation}
\end{lemma}
In this context, we have defined a constant standard deviation parameter for the policy of each agent to simplify the analysis. However, one can still define separately this parameter for each agent, i.e., $\xi_i$. Then in the upper bound of last inequality, $\xi^2=\textnormal{min}\{\xi^2_1, \xi^2_2, ..., \xi^2_N\}$. We show the proof for Corollary~\ref{coro_3} in the following. 
\begin{proof}
According to Corollary~\ref{coro_1}, when $K$ is sufficiently large, $T_1$ dominates the convergence such that we investigate it. According to Lemma~\ref{lemma_11}, we can immediately obtain that \[\bar{\sigma}^2\leq\frac{R^2C_f^2}{(1-\gamma)^2\xi^2}\bigg(\frac{1-\gamma^{2H}}{1-\gamma^2}-H\gamma^{2H}-2\gamma^H\frac{1-\gamma^H}{1-\gamma}\bigg)=\mathcal{O}\bigg(\frac{1}{(1-\gamma)^3}\bigg).\]
As $L=\frac{C_hR}{(1-\gamma)^2}$ and $G=\frac{C_gR}{(1-\gamma)^2}$, we have $D=96L^2+96G^2C_\upsilon=\mathcal{O}\bigg(\frac{1}{(1-\gamma)^4}\bigg)$. Thus,
\[\frac{256L^3D(J^*-J(\bar{\mathbf{x}}_0))+2048L^4\bar{\sigma}^2+D^2\bar{\sigma}^2}{8L^2D(NK)^{2/3}}\leq\mathcal{O}\bigg(\frac{1}{(1-\gamma)^3(NK)^{2/3}}\bigg).\] We can easily obtain that the sampling complexity is $\mathcal{O}\bigg(\frac{1}{(1-\gamma)^{4.5}N\epsilon^{3}}\bigg)$ when the policy is parameterized by a Gaussian distribution.
\end{proof}
\begin{remark}
When the policy function is a Gaussian distribution for each agent, Corollary~\ref{coro_3} shows that the sampling complexity is inherently in the order of $\mathcal{O}(\epsilon^{-3})$, which matches the generalized result. Moreover, the linear speed up is still retained. Additionally, we also show that the complexity does not rely on the horizon. The dependence on $(1-\gamma)^{-4.5}$ stems from the variance of the policy gradient estimator, which has been known in the centralized counterpart~\cite{xu2019sample}. When the initialization is with a single trajectory, it can be inferred from Corollary~\ref{coro_2} that the sampling complexity is $\mathcal{O}((1-\gamma)^{-6}N^{-1}\epsilon^{-4})$.
\end{remark}
\subsection*{Algorithmic Framework for \texttt{MDPG}}
We next present a new decentralized algorithm adapted from the centralized \texttt{MBPG}~\cite{huang2020momentum}, which serves as a new baseline for empirical comparison. Different from \texttt{MBPG}, the updates are simplified with constant step size and $\beta$ value. This simplification is to enable a fair comparison among all algorithms. 
\begin{algorithm}[H]
\SetAlgoLined
\KwResult{$\tilde{\mathbf{x}}_K$ chosen uniformly random from $\{\mathbf{x}^i_k,i\in\mathcal{V}\}^K_{k=1}$}
 \textbf{Input:} $\mathbf{x}^i_1=\bar{\mathbf{x}}_1\in\mathbb{R}^d,\eta\in\mathbb{R}^+,\beta\in(0,1),\mathbf{W}\in\mathbb{R}^{N\times N},K, \mathcal{B}\in\mathbb{Z}^+,k=1$\;
 
 


 \While{$k<K$}{
  \For{each agent}{
    if $k=1$: compute the local policy gradient surrogate by sampling a trajectory $\tau^i_1$ from $p_i(\tau^i|\mathbf{x}^i_1):\mathbf{u}^i_1=\mathbf{g}_i(\tau^i_1|\mathbf{x}^i_1)$, or \textcolor{blue}{by sampling a \textit{mini-batch} of trajectories $\{\tau_1^{i,m}\}_{m=1}^{|\mathcal{B}|}$ from $p_i(\tau^i|\mathbf{x}^i_1):\mathbf{u}^i_1=\frac{1}{|\mathcal{B}|}\sum_{m=1}^{|\mathcal{B}|}\mathbf{g}_i(\tau^{i,m}_1|\mathbf{x}^i_1)$}\;
    
    if $k>1$: sample a trajectory $\tau^i_k$ from $p_i(\tau^i|\mathbf{x}^i_k)$ and compute the local policy gradient surrogate using Eq.~\ref{surrogate}\;
    
    
    Update the local estimate of the policy network parameters $\mathbf{x}^i_{k+1}=\sum_{j\in Nb(i)}\omega_{ij}(\mathbf{x}^j_k+\eta\mathbf{u}^j_{k})$\;
    }
$k=k+1$\;
 }
 \caption{Momentum-based Decentralized Policy Gradient (\texttt{MDPG})}
 \label{mdpg}
\end{algorithm}

\subsection*{Computing Resource Details}

All numerical experiments presented in the main article and supplementary section were performed on a computing cluster with 2 x 16 Core Intel Xeon CPU E5-2620 v4 and a memory of 256 GB. All codes were implemented using PyTorch version 1.8.1.

\subsection*{Additional Experimental Results}
In the following section, we include additional results for comparing \texttt{MDPGT}, \texttt{MDPG} and \texttt{DPG} in a gridworld environment with two and three agents. Additionally, we also provide results for five, ten, twenty and thirty agents in a simplified lineworld environment. 

\subsubsection*{Gridworld Environment}

\begin{figure}[!h]
    \centering
    \includegraphics[width=0.5\columnwidth]{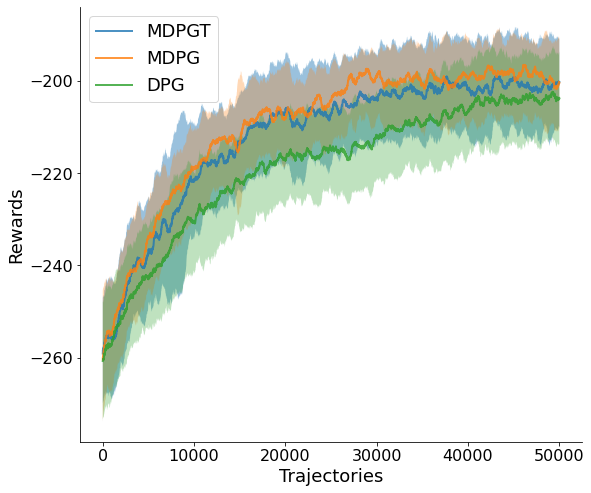}
    \caption{Experimental results comparing \texttt{MDPGT}, \texttt{MDPG} and \texttt{DPG} in the gridworld environment with 2 agents using fully-connected topology.}
    \label{fig:2dgridworld}
\end{figure}

\begin{figure}[!h]
    \centering
    \includegraphics[width=0.5\columnwidth]{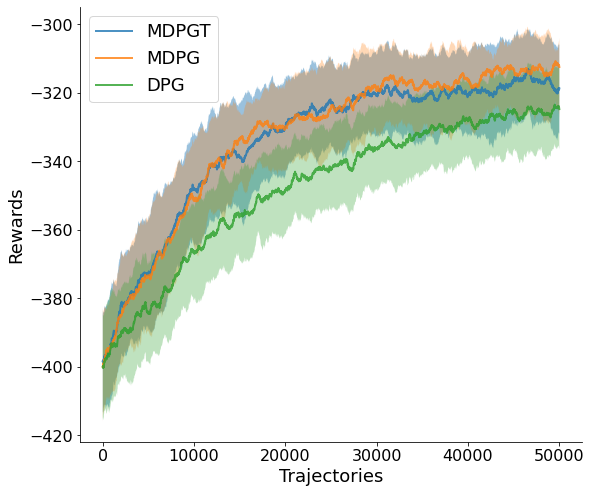}
    \caption{Experimental results comparing \texttt{MDPGT}, \texttt{MDPG} and \texttt{DPG} in the gridworld environment with 3 agents using fully-connected topology.}
    \label{fig:3dgridworld}
\end{figure}

\begin{figure}[!h]
    \centering
    \includegraphics[width=0.5\columnwidth]{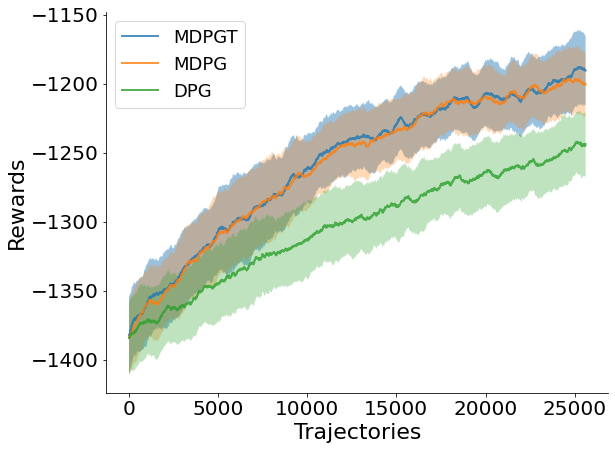}
    \caption{Experimental results comparing \texttt{MDPGT}, \texttt{MDPG} and \texttt{DPG} in the gridworld environment with 10 agents using fully-connected topology.}
    \label{fig:10dgridworld}
\end{figure}

Figures~\ref{fig:2dgridworld}, \ref{fig:3dgridworld}, and~\ref{fig:10dgridworld} illustrate the average reward obtained by the three algorithms, \texttt{MDPGT}, \texttt{MDPG} and \texttt{DPG} in with two, three and ten agents respectively using a fully-connected network topology with $\beta$ = 0.5. It should be noted that in Figure~\ref{fig:10dgridworld} the number of trajectories is only 25,000 such that \texttt{MDPGT} and \texttt{MDPG} perform similarly. However, when the number is 50,000, it suffices to show that the result should resemble the plot in Figure~\ref{fig:five_agents}, with even a larger gap between \texttt{MDPGT} and \texttt{MDPG}. In all scenarios, we observed that both \texttt{MDPGT} and \texttt{MDPG} still outperforms \texttt{DPG}. Interestingly, with fewer agents, the performance of \texttt{MDPGT} and \texttt{MDPG} are similar, in contrast to the five agent scenario presented in the main article. Fewer agents may reduce the impact of linear speed up on the error bound, which is w.r.t $\mathcal{O}(N^{-1})$. By observing carefully, when the number of agents becomes larger, \texttt{MDPGT} is more advantageous in a complex graph than \texttt{MDPG} as the extra tracking step contributes to correcting the gradient bias caused by diverse agents. This interesting finding will be validated in the following simplified lineworld environment where more agents are incorporated. 

\subsubsection*{Lineworld Environment}

As an initial proof of concept and to demonstrate our algorithm's ability to scale to a larger number of agents, we created a simplification of the gridworld environment called the Lineworld environment. In this environment, all agents are randomly initialized with a 1-D coordinate. The goal of all the agents are to cooperatively arrive at the coordinate \textbf{0}. In this setting, the agents are not allowed to collide and have a smaller action set of either moving up, down or stay stationary. The reward function we used for this environment is also defined as the Euclidean distance of individual agent to its respective goal. All agent's policy is represented by a 3-layer neural network with 64 hidden units with $tanh$ activation functions. The agents were trained for 10,000 episodes with a horizon of 500 steps and discount factor of 0.99. A learning rate of 3E-4 was used in all our experiments and we average the results conducted over 5 random seeds. 

Figures~\ref{fig:5_10d_lineworld} and~\ref{fig:20_30d_lineworld} shows the rewards of our proposed methods and the baseline $\texttt{DPG}$ in the Lineworld environment with five, ten, twenty and thirty agents respectively. In summary, the findings from the performance in the Lineworld environment strengthens the conclusion attained from the Gridworld setting that \texttt{MDPGT} outperforms \texttt{MDPG} in a more sophisticated graphs involving more agents. With five agents, we observed that \texttt{MDPGT} performs better than \texttt{DPG} but worse than \texttt{MDPG}. However, as the number of agent increases, we see that the performance of \texttt{MDPGT} improves to be on-par or slightly better than \texttt{MDPG} and significantly better than \texttt{DPG}. Thus, a practical guideline from an application point-of-view is that one can choose \texttt{MDPG} for simple graphs and expect \texttt{MDPGT} to perform much better in complex graphs.

\begin{figure}
    \centering
    \includegraphics[width=8cm]{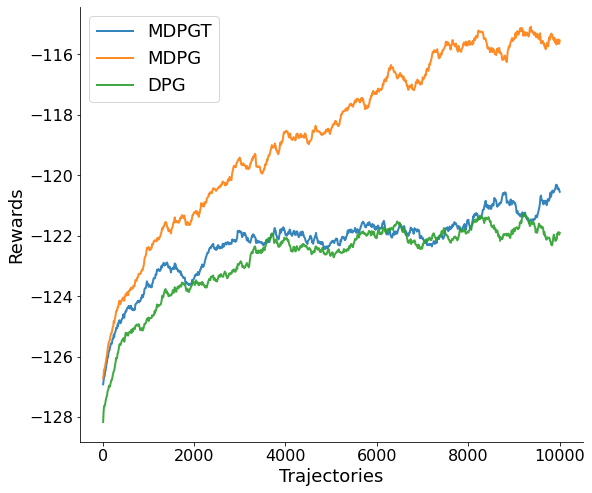}
    \includegraphics[width=8cm]{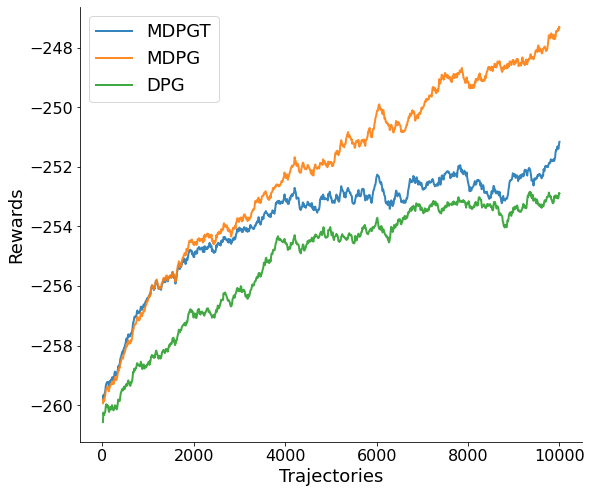}
    \caption{Experimental results comparing \texttt{MDPGT}, \texttt{MDPG} and \texttt{DPG} in the lineworld environment with 5 agents (left) and 10 agents (right) using fully-connected topology.}
    \label{fig:5_10d_lineworld}
\end{figure}

\begin{figure}
    \centering
    \includegraphics[width=8cm]{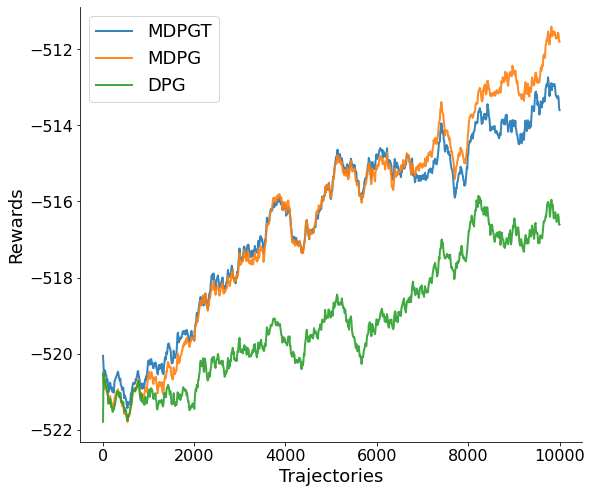}
    \includegraphics[width=8cm]{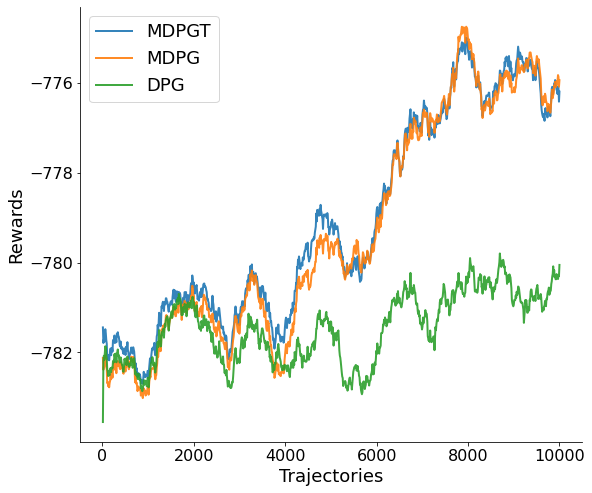}
    \caption{Experimental results comparing \texttt{MDPGT}, \texttt{MDPG} and \texttt{DPG} in the lineworld environment with 20 agents (left) and 30 agents (right) using fully-connected topology.}
    \label{fig:20_30d_lineworld}
\end{figure}





\begin{figure}
    \centering
    \includegraphics[width=0.8\columnwidth]{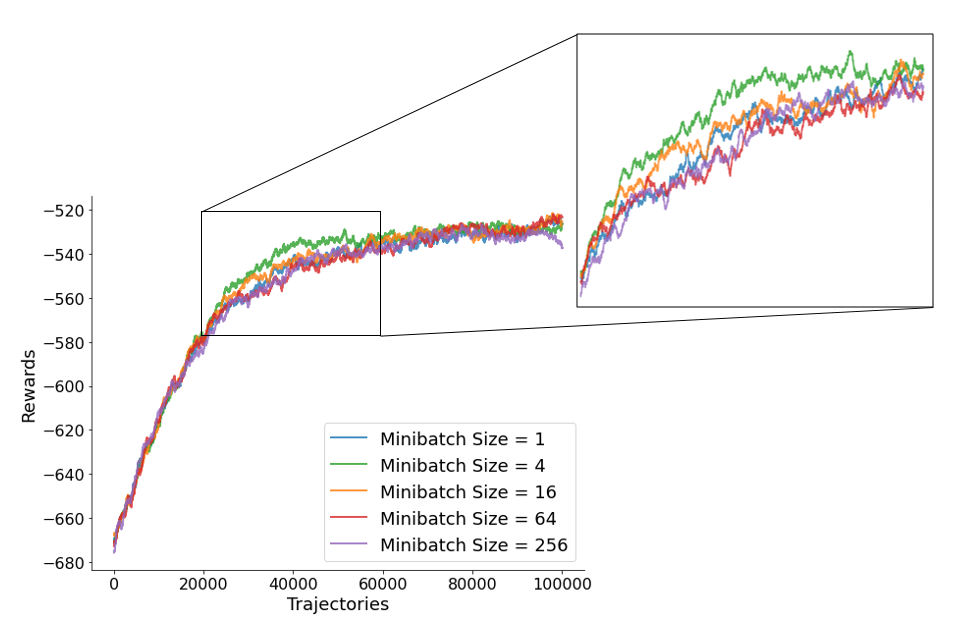}
    \caption{Experimental results comparing \texttt{MDPGT} in the gridworld environment with 5 agents and $\beta=0.5$ using fully-connected topology with different mini-batch initializations.}
    \label{fig:MI}
\end{figure}

\subsubsection*{Minibatch Initialization}
To validate the theoretical findings obtained in this work for the difference between two different versions of \texttt{MDPGT}, we implement several experiments with 5 agents using fully-connected topology with different mini-batch initializations. In this context, mini-batch size equal to 1 corresponds to \texttt{MDPGT} with only single trajectory initialization. From Figure~\ref{fig:MI}, one can observe that the reward curves due to different mini-batch initializations are depicted in three phases. Before approximately 20,000 trajectories, regardless of whichever initialization, \texttt{MDPGT} yields similar performance. After 20,000 trajectories, the difference becomes appealing, which is reflected by the outperforming of \texttt{MPDGT} with mini-batch size equal to 4. This supports the claim of Corollary~\ref{coro_1} that $K$ has to satisfy a certain condition to start showing the non-asymptotic behavior. Additionally, when the number of trajectories is larger than 60,000, another interesting phenomenon is that all reward curves start converging together again, which can be explained by using Theorem~\ref{theorem_1} as \texttt{MDPGT} enables the asymptotic results once $K$ is sufficiently larger than a certain number. We have pointed out in the analysis that they all converge to the same steady-state error defined in Eq.~\ref{sss}. Intuitively, practitioners can set a sufficiently large number of trajectories to get the same optimal solution regardless of what initialization is taken. While different \textit{approximate} solutions can also be obtained when different specific mini-batch sizes for initialization are defined. One may argue how to set the mini-batch size for a specific problem. For example, we can observe from Figure~\ref{fig:MI} that to get better approximate solutions, the mini-batch size could be selected in between 4 and 16. This is also suggested from Corollary~\ref{coro_1} that the mini-batch size is w.r.t both the number of agents and the number of trajectories. Thus, either mini-batch size equal to 64 or 256 results in slightly worse result than only one trajectory initialization. 

\end{document}